\documentclass[journal]{IEEEtran}

\usepackage[utf8]{inputenc}
\usepackage[T1]{fontenc}
\usepackage{microtype}
\usepackage{booktabs}
\usepackage{array}
\usepackage{xcolor}
\usepackage{amsmath}
\usepackage{amssymb}

\usepackage{amsthm}
\usepackage{graphicx}
\usepackage{tabularx}
\usepackage{float}
\usepackage{algorithm}
\usepackage{algpseudocode}
\usepackage[hidelinks]{hyperref}
\usepackage{xurl}
\usepackage{needspace}

\newtheorem{proposition}{Proposition}
\newtheorem{remark}{Remark}[section]

\newcommand{\truth}{Y}
\newcommand{\signal}{Z}
\newcommand{\Description}[1]{}
\newcommand{\shabind}[2]{\begingroup\footnotesize\texttt{#1}\newline\texttt{#2}\endgroup}

\title{Resample or Reroute? Recoverable Stopping Debt\\Without Identified Action Selection}

\author{Teng-Ruei~Chen%
\thanks{T.-R. Chen is with the Institute of Bioinformatics and Systems Biology, National Yang Ming Chiao Tung University, Hsinchu, Taiwan, and also with Krixvon, Taipei, Taiwan (e-mail: ymchen.bi04g@g2.nctu.edu.tw; ORCID: 0000-0001-8995-334X).}%
\thanks{Preprint, September 2026. Version 3 of arXiv:2607.08665 is a substantive methodological reconstruction of versions 1--2 of that record. It replaces the earlier budget-allocation, learned-router, retrospective cost-oracle, and broad Pareto interpretations with the three-gate identification framework and the narrower evidence-licensed conclusions reported here; versions 1--2 remain available in the arXiv version history. Every reported number is reproduced from frozen, SHA-256-bound result artifacts listed in the Supplementary Material.}}

\begin{document}
\markboth{Preprint, September 2026 --- arXiv:2607.08665v3}{Preprint, September 2026 --- arXiv:2607.08665v3}
\maketitle

\begin{abstract}
After a weak verifier accepts a large-language-model response, a second call
may resample or reroute.  Because correctness is hidden, action selection is
an identification problem.  We order three gates: recoverable stopping debt,
two-sided FIT action support, and held-out value from an outcome-blind
selector.  In a pinned 152-query MBPP+ experiment, a Qwen2.5-14B Base-only
false-positive stop leaves $+2.592$ percentage points of Qwen2.5-7B recovery
(query-cluster 95\% interval $[+1.618,+3.664]$).  Separately, after 7B
Base-test rejection, fixed escalation to 14B exceeds leave-one-out 7B
resampling by $+2.882$ points $[+0.931,+5.201]$; this is fixed-action ranking,
not conditional selection.  An all-episode audit produces a $+2.697$-point
realized-maximum gap, but for two actions this statistic equals
$\tfrac12\mathbb E|\Delta|-\tfrac12|\mathbb E\Delta|$ and contains no
observable-history term.  It lies inside an exact-fold exchangeable reference
(mean $+3.158$; 95\% interval $[+2.434,+3.947]$).  The audit unconditionally
acts on 1,520 episodes: 1,240 observable stops and 280 verifier rejections; 198
stops are evaluator-only false positives.  Neither tested outcome-blind
controller improves on fixed rerouting.  A separate LiveCodeBench ladder has
all-zero FIT action advantages despite exclusive TEST rescues.  A
preregistered BigCodeBench support gate then finds only 23/19 and 22/19
signed episodes/queries against minima of 25/20, so L1--L4, DEV, and TEST stay
unopened.  Stopping debt exists, but current evidence does not identify when
to resample rather than reroute.
\end{abstract}


\section{Introduction}\label{sec:intro}

Routing and test-time sampling expose two controls over one serving budget.  A
router changes \emph{which model} answers; best-of-$N$ and self-consistency
change \emph{how many responses} are drawn~\cite{routellm,selfconsistency,
scalingtesttime}.  Recent cascades and multi-model methods combine these
controls before or after observing a response~\cite{frugalgpt,automix2024,
cascaderouting2025,bestroute2025,modelswitch2026}.  Our question begins later:
after a fallible verifier has already accepted a candidate, when is one more
call scientifically justified, and when can data identify whether that call
should resample or reroute?

The hidden-truth boundary makes three statements easy to confuse.  First, a
verifier false positive may create \emph{recoverable verifier-induced stopping
debt}: the system stops incorrectly although a declared second action can
recover.  Second, even when both actions sometimes help, the fitting bank may
contain too few histories on which their utilities differ.  Third, even with
such support, legal stopped-history features may not identify the better
action out of sample.  A pointwise maximum of realized action outcomes
addresses none of these distinctions: it can be positive when the actions are
exchangeable and have identical conditional means.

We therefore replace the broad question ``does dynamic routing work?'' with
three ordered gates:
\begin{enumerate}
  \item \textbf{Gate 1---Recovery:} does the observable stop leave hidden
        correctness that at least one declared action can recover?
  \item \textbf{Gate 2---Identification support:} conditional on stopping,
        does FIT contain query-distributed evidence of both
        reroute-over-resample and resample-over-reroute utility?
  \item \textbf{Gate 3---Observable policy value:} only after Gates 1 and 2,
        can an outcome-blind selector beat the frozen best fixed action on
        held-out quality or efficiency?
\end{enumerate}
Failure at a gate is a terminal scientific branch.  Gate 1 failure licenses no
second-call claim; Gate 2 failure licenses no feature or generalization claim;
Gate 3 failure licenses the fixed action rather than a dynamic controller.
The gates form one inferential chain only when they use the same declared
stopping event, action set, utility, and target population.  Evidence from
different populations cannot be composed merely by assigning it a gate number.

\smallskip\noindent\textbf{Scientific revision path.}
The final ordering was learned rather than assumed.  The fresh MBPP+
experiment first established the narrow Gate-1 mechanism.  An all-episode
matched-information audit then appeared to show a realized-maximum advantage without
controller value.  Two later corrections changed that reading.  First, the
realized-maximum statistic was recognized as a signal-free function of the two
realized action outcomes.  Its exact-fold exchangeable distribution is retained
as a calibration of marginal action-label symmetry, not as a test of
conditional heterogeneity.  Second, the audit acts on every episode rather than
on the observable Gate-1 stopped population, so it is not a Gate-3 test for the
same chain.  The separated LiveCodeBench ladder supplied a third
correction: all FIT action-advantage labels were zero, so its null controller
could not diagnose feature insufficiency.  The preregistered BigCodeBench
replication converted that lesson into a prospective Gate-2 opening guard and
stopped before L1--L4, DEV, or TEST.  The chronology explains why the paper
contains several experiments; the final logic retracts conclusions that the
available evidence cannot identify.

The headline evidence is therefore assigned rather than pooled.  MBPP+ tests
recovery.  LiveCodeBench diagnoses support nontransport.  BigCodeBench tests
the prospective support gate.  The MBPP+ matched-information analysis is an
all-episode descriptive stress test on an opened, query-cross-fitted bank, not
a stopped-history Gate-3 test or independent deployment evaluation.
Post-outcome Llama and Nemotron arms test
only Gate-1 model-family robustness.  Historical GSM8K, MATH-500,
GPQA-Diamond, and HumanEval+ replay~\cite{gsm8k,mathdataset,lightman2024verify,
gpqa,humaneval,evalplus}
is retained as audit/provenance context,
not as headline replication.  Version 2 of this arXiv record
(arXiv:2607.08665v2), the earlier replay-based RoR preprint, is likewise
disclosed as provenance rather than silently imported
evidence~\cite{rorpreprint2026}.  This version is a substantive methodological
rebuild of it: former broad policy and Pareto claims are not carried into the
three-gate estimands.

This paper makes four contributions:
\begin{enumerate}
  \item \textbf{A gate-conditioned identification framework.}  We formalize
        the information boundary, recoverable stopping-debt estimand,
        two-sided FIT support condition, and held-out policy gate, including
        the requirement that a composed chain share one stopping event and
        target population.
  \item \textbf{A sequence of bounded empirical findings.}  Fresh MBPP+
        evidence passes Gate 1; LiveCodeBench reveals why a zero-target FIT
        bank makes feature conclusions unidentified; and preregistered
        BigCodeBench candidate-011 fails Gate 2 before any held-out opening.
  \item \textbf{A calibrated negative selector result.}  We show generally
        that a two-action realized-maximum gap equals
        $\tfrac12\mathbb E|\Delta|-\tfrac12|\mathbb E\Delta|$ and therefore
        contains no observable-history signal.  In the opened MBPP+ audit,
        fixed rerouting is marginally better than fixed resampling, but neither
        tested outcome-blind controller improves on the cross-fitted fixed
        action.  The experiment identifies no successful conditional selector;
        it establishes neither action homogeneity nor an observability
        bottleneck.
  \item \textbf{A fail-closed design rule.}  We show that action abundance is
        not action-ranking support and that a realized maximum is not
        observable policy value.  The resulting procedure defaults to a
        fixed action unless every upstream gate and a held-out quality or
        efficiency gate pass.
\end{enumerate}

All claims remain finite-bank and protocol conditional.  The manuscript first
positions the problem, then defines the three gates and their mathematics,
then reports evidence in gate order, and finally states only the policy class
licensed by the terminal branch.  Exact execution contracts, artifact rosters,
historical replay grids, and extended derivations are delegated to the
Supplementary Material.

\section{Related Work}\label{sec:related}

We classify prior work by \emph{when} the allocation decision is made and
\emph{what information} is legal at that decision.  This taxonomy separates
pre-response routing, response-conditioned cascades, adaptive sampling, and
verifier-based code selection.  Table~\ref{tab:related-taxonomy} summarizes
the resulting gap: prior systems can combine model choice and repeated
sampling, but, to our knowledge within the reviewed literature, they do not
jointly diagnose recovery opportunity, finite-sample action support, and
stopped-history observability under an explicit hidden-truth boundary.

\textbf{Pre-response and response-conditioned routing.}
RouteLLM learns a pre-response model choice from preference data, and routing
benchmarks expose per-query model outcomes~\cite{routellm,routerbench,
routereval,llmrouterbench}.  Other systems already condition later actions on
realized responses.  FrugalGPT scores the query--answer pair before accepting
or escalating; AutoMix treats noisy self-verification and escalation as a
partially observed decision problem; cascade routing selects between stopping
and invoking another unused model after each response; and Router-R1 learns
multi-round invocation and aggregation~\cite{frugalgpt,automix2024,
cascaderouting2025,routerr1}.  Thus response-conditioned routing itself is not
our novelty.  A contemporaneous unreviewed preprint also optimizes stop/model
actions directly as an MDP~\cite{rlcascaderouter2026}.

Recent work broadens what a legal routing signal can contain.  Confidence
tokens learn an explicit self-assessment for routing or rejection;
Lookahead predicts latent representations of potential responses; and
pre-generation activation probes can encode policy-specific model success
~\cite{confidencetokens2025,lookaheadrouting2025,activationfailures2026}.
These results do not conflict with our negative controllers: they show that
other representations may be informative, while our claim is limited to the
frozen structural, token-probability, verifier-trace, agreement, and signed-
hash regimes.  They also motivate separating a failed implemented learner
from the Bayes envelope of a richer signal.

\textbf{Adaptive sampling and multi-model test-time compute.}
Self-consistency and test-time scaling draw several candidates and require a
selection rule~\cite{selfconsistency,scalingtesttime}.  Adaptive-Consistency
uses agreement to vary the sample count by query; BEST-Route chooses a model
and a sample count before generation; and ModelSwitch performs multi-model
repeated sampling with consistency-triggered switching~\cite{
adaptiveconsistency2023,bestroute2025,modelswitch2026}.  Agreement-based and
semantic-agreement cascades likewise use ensemble consensus for deferral
~\cite{agreementcascade2025,semanticcascade2025}.  Distributional
self-certainty supplies a reward-free best-of-$N$ selector, while strategic
bandit allocation learns which queries merit additional compute
~\cite{selfcertainty2025,strategictesttime2026}.  These studies rule out a
broad claim that combining resampling and rerouting is new.  The present work
instead asks what remains recoverable under a specific weak-verifier failure
event.

\textbf{Code verification and false-positive limits.}
CodeT generates tests and selects candidates by execution agreement, while
LEVER learns a verifier over programs and execution results~\cite{codet2023,
lever2023}.  Model Cascading for Code combines multiple candidate programs,
self-generated tests, and thresholded escalation across models
~\cite{codecascade2024}.  EvalPlus demonstrates why ordinary benchmark tests
are an incomplete correctness signal~\cite{evalplus}.  Agreement itself can
also miss repeated self-consistent errors, including errors that persist as
model scale grows~\cite{selfconsistenterrors2025}.  Most directly, Stroebl
et al. show that verifier false positives impose an accuracy ceiling that more
single-model resampling cannot remove~\cite{stroebl2026limits}.  We take that
limit as prior work.  Our narrower contribution is to separate visible Base
tests from hidden Plus correctness and estimate cross-model recovery mass
conditional on an irreversible Base-test false-positive stop.  This is an
evaluator-only mechanism estimand, not a claim that the hidden event is
available to a deployed router.

\textbf{Measurement companion.}
The companion routing-oracle study distinguishes a pre-response single-commit
ceiling from post-response union events and shows that a union inferred from
marginals depends on a declared coupling~\cite{routinggapreal}.  Those
measurement results motivate caution here but are not re-proved or counted as
RoR contributions.  In particular, this paper does not interpret an oracle
gap as recoverable specialist advantage and does not use a product union as a
deployment ceiling.

\textbf{Earlier RoR preprint.}
Version 2 of this arXiv record (arXiv:2607.08665v2, revised 10 July 2026),
titled ``Resample or Reroute? Budget-Aware Test-Time Model Selection for Large
Language Models,'' presented the historical eleven-model replay as a policy
comparison~\cite{rorpreprint2026}.  The present version does not treat
publication of that replay as validation.  It reclassifies the
archive as an audit/development record, removes the unsupported learned-router,
retrospective cost-oracle, and universal Pareto interpretations, and changes
the target from a broad budget allocator to verifier-gated recovery with an
explicit evaluator/controller information boundary.  The fresh MBPP+ and
LiveCodeBench experiments, support-aware audit, and theoretical decomposition
are new to this rebuild.

\textbf{Budgeted allocation.}
The action-selection rule is related to marginal-value heuristics and
best-arm allocation~\cite{audibert2010,kaufmann2016}.  Unlike pure best-arm
identification, the objective is to return one answer before a heterogeneous
cost cap.  The greedy rule is a transparent heuristic, not a proof of
horizon-optimal control.

\textbf{Opportunity versus learnability.}
Routing papers commonly report either an oracle ceiling or a learned policy,
but the two answer different questions.  A realized pointwise maximum can be
positive even under exchangeable actions, so it does not by itself establish
heterogeneous conditional action value.  Conversely, one failed learner does
not prove that every richer action space is homogeneous.  Causal LLM routing
has shown how policies can be learned from observational feedback, in contrast
to the full-feedback matrices assumed by many routing methods~\cite{causalrouting2025}.
Our matched-information
audit places the realized maximum, cross-fitted fixed action, cost-matched
random mixture, and outcome-blind learned policy on the same fixed episodes and
query folds.  We use it descriptively and calibrate its maximum against an
exchangeable-actions null; because it does not condition on the Gate-1 stop
event, it cannot diagnose which link of that stopped-history chain fails.

\begin{table*}[t]
\centering
\caption{Taxonomy of the closest allocation and verification families.  The
distinguishing object is the information available when the second action is
selected, not whether a system happens to invoke more than one model.}
\label{tab:related-taxonomy}
\footnotesize
\setlength{\tabcolsep}{3.5pt}
\begin{tabularx}{\textwidth}{@{}>{\raggedright\arraybackslash}p{0.14\textwidth}>{\raggedright\arraybackslash}p{0.12\textwidth}>{\raggedright\arraybackslash}p{0.16\textwidth}>{\raggedright\arraybackslash}p{0.14\textwidth}>{\raggedright\arraybackslash}X@{}}
\toprule
Family & Decision point & Legal feedback & Main action & Gap addressed here \\
\midrule
Query routing~\cite{routellm,routerbench,routereval}
& before response & prompt or preference features & choose one model
& cannot condition on a realized weak-verifier stop \\
Response-conditioned cascades~\cite{frugalgpt,automix2024,cascaderouting2025}
& after a response & answer score or self-verification & accept or escalate
& does not separate realized action availability from learnable action choice \\
Adaptive sampling and multi-model compute~\cite{adaptiveconsistency2023,bestroute2025,modelswitch2026}
& before or during sampling & agreement, consistency, or predicted query difficulty & sample again or switch
& does not audit finite-sample support before interpreting a learned selector \\
Code generation and verification~\cite{codet2023,lever2023,codecascade2024,evalplus}
& after candidate execution & generated tests, public tests, or learned score
& select, reject, or escalate a candidate
& typically treats verification as selection evidence rather than measuring hidden recovery debt \\
This study
& after a declared verifier stop for the prospective framework & outcome-blind stopped history only
& resample or reroute, then stop
& separates debt, support, and claim licensing; the all-episode audit is descriptive \\
\bottomrule
\end{tabularx}
\end{table*}

\section{Three-Gate Problem Formulation and Information Boundary}
\label{sec:problem}

This section defines the sequential serving contract and the information
available to each participant.  The policy may use only query features,
response-derived primitives, verifier outputs, costs, and state learned from a
declared fitting split.  Hidden correctness is available only to the evaluator
that measures recovery opportunity and held-out utility.  This boundary is the
common foundation for the mechanism estimand, audit algorithm, and every
baseline comparison.

For query $i$, model $m\in\{1,\ldots,M\}$, and archived draw
$d\in\{0,\ldots,D-1\}$, let $X_{imd}$ denote the response text and
$\truth_{imd}\in\{0,1\}$ its scorer-specific correctness label.  A
response-only primitive $E_{imd}=e(X_{imd})$ may be an extracted answer or a
visible test outcome.  After the $t$th paid call, a verifier adapter maps the
current primitive and prior observable history to a policy-visible signal and
acceptance event,
\begin{equation}
  (\signal_{it},R_{it})=\phi(E_{iA_{it}J_{it}},\mathcal H_{i,t-1}).
\end{equation}
This history dependence is essential for agreement: acceptance is not a static
property of one archived cell.  The adapter may instead reveal the scorer label
in the explicitly separated perfect-label diagnostic, or implement partial
tests or a learned signal.  This is an access-control statement, not an
assertion that scorer and verifier channels are statistically independent.

A policy selects a sequence of model calls $A_{i1},A_{i2},\ldots$ with costs
$c_{A_{it}}>0$ subject to the pathwise cap
\begin{equation}\label{eq:budget}
  \sum_t c_{A_{it}}\le B.
\end{equation}
It may stop when the verifier accepts a candidate or when no further call is
affordable.  A declared selector $g$ then returns one of the drawn candidates.
The evaluation outcome is $\truth_{i,g(\mathcal H_i)}$, where $\mathcal H_i$ is the observable
history.  Evaluator labels may be used for explicitly declared offline
supervised fitting on TRAIN, but online TEST actions use only the resulting
frozen state and the observable verifier history.

Let $Q_i$ denote policy-visible query features; it excludes the benchmark gold
answer, hidden reference tests, and evaluator-only metadata.  Let
$\theta=f(\mathcal D_{\mathrm{TRAIN}})$ be the policy state frozen before TEST,
and let $\omega_i$ collect policy randomization.  The replay engine uses $J_{is}$ to
locate a paid response, but this archive slot is evaluator bookkeeping and is
not visible to the policy or selector.  Write
$X^{\mathrm{obs}}_{is}=X_{iA_{is}J_{is}}$ and
$E^{\mathrm{obs}}_{is}=E_{iA_{is}J_{is}}$.  Define
\begin{align}
  \mathcal H_{i,0}&=\sigma(Q_i,\omega_i,\theta),\\
  \mathcal H_{i,t}&=\sigma\!\left(\mathcal H_{i,0},
  (A_{is},X^{\mathrm{obs}}_{is},E^{\mathrm{obs}}_{is},
  \signal_{is},R_{is},c_{A_{is}})_{s=1}^{t}\right).
\end{align}
The next action $A_{i,t+1}$, stopping decision, and final selector must be
$\mathcal H_{i,t}$-measurable.  Held-out labels $\truth_{\mathrm{TEST}}$ may
enter neither $\theta$ nor the filtration of any deployable online policy.
TRAIN labels may enter only the declared offline fit $f$.  In the perfect-label
diagnostic, $\signal_{it}=\truth_{iA_{it}J_{it}}$ is revealed only after
$A_{it}$ is chosen and its call cost is charged.

\section{Methods: Gate-Ordered Estimation and Claim Licensing}
\label{sec:method}

For stopped episode $j$ from query $i$, let $Y^0_{ij}$ be hidden correctness
of the accepted response and let $Y^{\mathrm{res}}_{ij}$ and
$Y^{\mathrm{rer}}_{ij}$ be hidden correctness after exactly one declared
resample or reroute.  Correct stops are preserved, so the two action utilities
are
\begin{equation}\label{eq:gate-action-utility}
 U_{ij}(a)=\max\{Y^0_{ij},Y^a_{ij}\},
 \qquad a\in\{\mathrm{res},\mathrm{rer}\},
\end{equation}
and the action-ranking target is
\begin{equation}\label{eq:gate-action-difference}
 \Delta_{ij}=U_{ij}(\mathrm{rer})-U_{ij}(\mathrm{res})\in\{-1,0,1\}.
\end{equation}
Gate 1 concerns whether any declared action recovers value over $Y^0$; Gate 2
concerns whether both signs of $\Delta$ have sufficient episode-, query-, and
fold-level support in FIT; Gate 3 concerns the held-out value of a selector
measurable with respect to $\mathcal H_{i,t}$.  In abstract form the licensed
dynamic-policy branch is
\begin{equation}\label{eq:three-gate-conjunction}
 G_{\mathrm{dynamic}}=G_{\mathrm{recovery}}\,
 G_{\mathrm{support}}\,
 G_{\mathrm{policy}},
\end{equation}
where each factor is binary and failure of any factor is absorbing.  The
benchmark-specific estimators and frozen thresholds below instantiate these
three factors; Eq.~\eqref{eq:three-gate-conjunction} is a claim-licensing rule,
not a model-selection score.

\subsection{Verifier-gated recovery estimand}\label{sec:mechanism-estimand}

The fresh confirmatory experiment isolates a narrow recovery mechanism and
does not optimize the retired historical allocator.  For TEST query $i$ and draw
$d\in\{0,\ldots,D-1\}$, let $Z^{14}_{id}$ be 14B Base-test passage and let
$Y^{m}_{id}$ be hidden Base-and-Plus correctness for model $m$.  Because full
correctness implies Base passage, $Y=1\Rightarrow Z=1$.  Define the irreversible
false-positive-stop indicator
\begin{equation}
  F_{id}=\mathbf{1}\{Z^{14}_{id}=1,\,Y^{14}_{id}=0\}.
\end{equation}
The prespecified primary action-space recovery estimand is
\begin{equation}\label{eq:vg-primary}
  \widehat\Delta_{\mathrm{VG}}
  =\frac{1}{N}\sum_{i=1}^{N}
   \left(\frac{1}{D}\sum_{d=0}^{D-1}F_{id}\right)
   \left(\frac{1}{D}\sum_{e=0}^{D-1}Y^{7}_{ie}\right).
\end{equation}
It measures how much fully correct 7B mass remains available when a 14B
candidate would have stopped on a Base-only false positive.  It is an
evaluator-only finite-bank mechanism quantity, not a runtime feature and not
the value of a selected deployment policy.  We call this quantity
\emph{recoverable verifier-induced stopping debt}: correctness mass forfeited
by an observable stopping rule but recoverable by one declared legal action.
It is narrower than the verifier's total false-positive rate, because debt is
counted only when the alternative bank actually contains a full-correctness
success under the prespecified pairing.  The fixed-pairing sensitivity
replaces the product of draw means with
\begin{equation}\label{eq:vg-offset}
  \widehat\Delta_{\mathrm{off}}
  =\frac{1}{ND}\sum_{i=1}^{N}\sum_{d=0}^{D-1}
    F_{id}Y^{7}_{i,(d+1)\bmod D}.
\end{equation}
The prespecified secondary comparison concerns the complementary observable
event on which the 7B Base verifier rejects and execution therefore continues.
Using leave-one-out resampling to avoid reusing the rejected draw, define
\begin{equation}\label{eq:vg-secondary}
\begin{split}
 \widehat\Delta_{\mathrm{sec}}
 &=\frac{1}{N}\sum_{i=1}^{N}\frac{1}{D}\sum_{d=0}^{D-1}
 \mathbf 1\{Z^{7}_{id}=0\}\\
 &\qquad\times\left[
   \frac{1}{D}\sum_{e=0}^{D-1}Y^{14}_{ie}
   -\frac{1}{D-1}\sum_{e\ne d}Y^{7}_{ie}
 \right].
\end{split}
\end{equation}
This is an event-weighted marginal comparison between fixed rerouting to 14B
and fixed leave-one-out resampling of 7B after a Base-test rejection.  It was
prespecified as an underpowered secondary estimand.  It is not a stopping-debt
estimand, does not select an action from history, and does not enter the Gate-1
decision.
Uncertainty resamples query clusters, never individual draws.  The sole
primary decision rule is that the deterministic 10,000-replicate percentile
bootstrap lower 95\% endpoint for \eqref{eq:vg-primary} is strictly above
zero.  No scalar exchange rate combines calls, generated tokens, latency, or
model size in this decision.

\section{Operational Tests and Descriptive Stress Audit}
\label{sec:audit}

The verifier-gated estimator answers a bounded MBPP+ recovery question.  Any
Gate-2 or Gate-3 continuation of that same empirical chain would have to retain
the observable MBPP+ stop $Z^{14}=1$, its action set, utility, and target
population.  Neither LiveCodeBench nor BigCodeBench continues that chain: each
defines a different public-test stopping event and evaluates its gates only
within its own stopped population.  The earlier MBPP+ matched-information audit
is different again because it assigns a second action to every start draw.  We
therefore retain it as a descriptive stress audit, not as an operational
decomposition of any completed stopped-history chain.

For query $i$, start-model draw $d$, and legal second action $a$, define
$d^+=(d+1)\bmod D$ and the evaluator-only one-step utility
\begin{equation}\label{eq:audit-utility}
  U_{ida}=\max\{Y^{s}_{id},Y^{a}_{i d^+}\},
\end{equation}
where $s$ is the fixed start model.  Selecting $a=s$ is a resample; selecting
$a\ne s$ is a reroute.  The circular offset consumes a distinct, fixed bank
position and never searches future outcomes.  Let $a^{\mathrm{fix}}_{-i}$ be
the single action selected using only training queries in the relevant outer
fold, and let $a^*_{id}=\arg\max_a U_{ida}$ be an evaluator-only oracle.  The
two diagnostic gaps are
\begin{align}
  \widehat\Delta_{\mathrm{rm}}
   &= \mathbb E[U_{id a^*_{id}}-U_{id a^{\mathrm{fix}}_{-i}}],
   \label{eq:opportunity-gap}\\
 \Delta_{\mathrm{obs}}
   &= \mathbb E[U_{id\widehat\pi_{-i}(X_{id})}
                -U_{id a^{\mathrm{fix}}_{-i}}],
   \label{eq:observable-gap}
\end{align}
where $X_{id}$ is computed from the first completion without correctness,
reference solutions, or hidden tests.  We call
$\widehat\Delta_{\mathrm{rm}}$ the \emph{realized-maximum gap}.  It is an
evaluator-only descriptive statistic, not a controller feature and not, by
itself, evidence that the two actions have different conditional expected
utility.

To calibrate the maximum, independently exchange the two action labels within
each episode with probability $1/2$.  For each exchange vector $\xi$, preserve
the recorded query folds, reselect $a^{\mathrm{fix}}_{-i}$ on the corresponding
training queries, and recompute Eq.~\eqref{eq:opportunity-gap}.  If
$G(U;\mathcal F)$ denotes that complete cross-fitted operation, the
post-audit reference distribution is
\begin{equation}\label{eq:exchangeable-null}
 \mathcal N_{\mathrm{exch}}
 =\mathcal L_{\xi}\!\left(
 G(\operatorname{swap}_{\xi}U;\mathcal F)\mid U,\mathcal F
 \right).
\end{equation}
This randomization calibrates the fold-specific fixed-action selection under
within-episode action-label exchangeability.  It is not a test of conditional
heterogeneity because neither $\widehat\Delta_{\mathrm{rm}}$ nor the
randomization uses an observable stopped-history signal.  In this bank all five
folds select rerouting.  Relative to the analytic exchangeable mean, the
observed shortfall is exactly $7/1{,}520=0.460526$ points and reflects the
sample marginal advantage of rerouting; the Monte Carlo estimate is $0.460482$
points.  The observed statistic lies at the 18.283rd percentile of the
exchangeable reference distribution.

The observable MBPP+ stopping event is Qwen14 Base-test acceptance,
$Z^{14}=1$, which occurs on 1,240 of 1,520 draws.  The subset $F_{id}=1$
contains 198 draws but additionally uses hidden correctness and is
evaluator-only.  The matched audit conditions on neither event: it acts on all
1,520 episodes, including the 280 Base-test rejections on which no stop
occurred.

\subsection{Operational decomposition}\label{sec:audit-layers}

In the prospective designs, Gate 2 precedes feature fitting and verifies that
FIT contains both action-difference signs at the episode, query, and fold
levels.  Gate 3 would then evaluate an outcome-blind policy against the frozen
fixed action and a frequency-matched random mixture before applying quality or
efficiency thresholds.  Robustness and evidence-integrity checks can narrow a
claim but cannot reverse a failed gate.  BigCodeBench ends at Gate 2.  The
older matched audit executed neither a stopped-history Gate 1 nor a prospective
FIT support gate, so its frozen machine verdict is retained only as provenance
and is not adopted as the scientific conclusion.

Figure~\ref{fig:audit-flow} shows the information-separation constraint that
applies inside every learned-policy comparison.  Raw
generation and evaluator outcomes are loaded through separate functions.
The feature builder completes before the outcome tensor is joined; policy
selection receives predictions and costs but no held-out label.  Outcomes
meet selected actions only inside the evaluator.

\begin{figure*}[t]
\centering
\includegraphics[width=\textwidth]{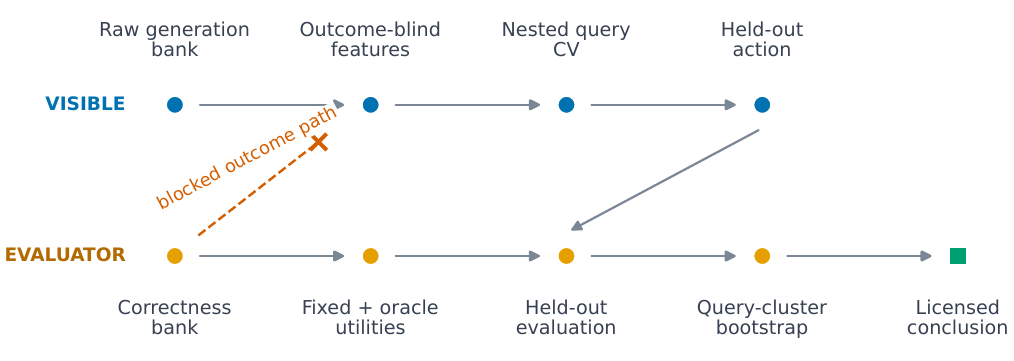}
\caption{Matched-information audit flow.  The upper lane is the only
controller-visible path.  The lower evaluator lane supplies training targets
on training folds and held-out utility after an action has been selected, but
has no path into held-out feature materialization or action selection.}
\Description{A two-lane flow chart separates a frozen raw generation bank and
outcome-blind feature builder from an evaluator-only correctness bank.  Nested
query cross-validation produces a held-out resample or reroute action.  Only
the evaluator combines that action with outcomes, bootstrap uncertainty, and
decision gates.  A dashed arrow from correctness to the feature builder is
marked forbidden.}
\label{fig:audit-flow}
\end{figure*}

The controller is intentionally modest.  For each action it fits a ridge
regression to the vector $(U_{ida})_a$ on training episodes.  An inner
query-level cross-validation chooses the ridge penalty; a disjoint outer fold
evaluates the selected action.  All $D$ draw positions from a query remain in
the same fold.  Ties are deterministic: lower nominal action cost, then frozen
model index.  We also compare with a random action mixture having the same
held-out marginal action frequencies, so a gain cannot be attributed merely
to choosing a cheaper or more common action more often.

\begin{algorithm*}[t]
\caption{Descriptive all-episode matched-information audit}
\label{alg:matched-audit}
\begin{algorithmic}[1]
\State \textbf{Input:} frozen raw bank $R$, evaluator bank $Y$, actions $A$,
costs $c$, query IDs, ridge grid $\Lambda$
\State materialize legal features $X\leftarrow f(R)$ without an evaluator argument
\State validate and join $R$ and $Y$ by query, draw, file, seed, byte length,
and completion digest
\State compute $U_{ida}$ by \eqref{eq:audit-utility}, fixed-action values,
and realized-maximum values on all episodes
\For{each deterministic outer query fold $o$}
  \State choose $\lambda_o\in\Lambda$ by inner query-fold validation on $o^c$
  \State fit action-utility regressions on all episodes from $o^c$
  \State select held-out actions using predictions and costs only
  \State evaluate selected, fixed, mixture, and oracle utilities on fold $o$
\EndFor
\State aggregate by query and compute deterministic 10,000-replicate
paired query-cluster intervals
\State \textbf{Post-audit calibration:} compute the exact-fold exchangeable
reference in Eq.~\eqref{eq:exchangeable-null} with
\texttt{paper/build\_exchangeable\_action\_null.py}
\State verify robustness and evidence invariants
\State \Return descriptive gaps, exchangeable calibration, and claim boundary
\end{algorithmic}
\end{algorithm*}

Lines 1--11 of Algorithm~\ref{alg:matched-audit} record the frozen F5
execution: feature materialization, exact-fold fitting, held-out scoring, and
query-cluster intervals.  The exchangeable reference in line 12 was added
after the adversarial audit and was computed separately by
\texttt{paper/build\_exchangeable\_action\_null.py}; it was not part of the
frozen F5 run.  Neither the F5 execution nor the post-audit calibration can
authorize a stopped-history controller.  Only a separate preregistered chain
with recovery, two-sided FIT support, and an untouched held-out quality or
efficiency pass can authorize
\texttt{GO\_BUILD\_DYNAMIC\_CONTROLLER}.  The Supplementary Material gives the
exact folds, calibration, thresholds, and fail-closed tests.

\section{Theory: Realized Maxima, Support, and Observable Value}\label{sec:theory}

This section formalizes the logical order used by the three gates.  The first
result separates a frozen training-selected action from the TEST-best fixed
action, the value attainable from a stopped-history signal, and the
evaluator-only oracle.  The second states the finite-sample support condition
exposed by the sealed ladder.  Local properties of the retired historical
allocator are retained only in the Supplementary Material.

\subsection{Four-value stopped-history decomposition}

Let $S$ be an observable stopping event, let $\mathcal A$ be a finite legal
second-action set, and let $U_a\in[0,1]$ be evaluator utility after action
$a\in\mathcal A$.  Let $a_{\mathrm{FIT}}$ be the action frozen using FIT and
let $Z$ be any outcome-blind stopped-history signal.  On the evaluation
distribution define
\begin{align}
 V_{\mathrm{frozen}}
   &=\mathbb E[U_{a_{\mathrm{FIT}}}\mid S],\\
 V_{\mathrm{fix}}
   &=\max_{a\in\mathcal A}\mathbb E[U_a\mid S],\\
 V_Z
   &=\mathbb E\!\left[\max_{a\in\mathcal A}
       \mathbb E[U_a\mid Z,S]\mathrel{\Big|}S\right],\\
 V_{\mathrm{or}}
   &=\mathbb E[\max_{a\in\mathcal A}U_a\mid S].
 \label{eq:four-values}
\end{align}

\begin{proposition}[Support-aware value decomposition]
\label{prop:four-value-decomposition}
If $a_{\mathrm{FIT}}\in\mathcal A$, then
\begin{equation}
 V_{\mathrm{frozen}}\le V_{\mathrm{fix}}\le V_Z\le V_{\mathrm{or}},
 \label{eq:four-value-order}
\end{equation}
and
\begin{align}
 V_{\mathrm{or}}-V_{\mathrm{frozen}}
 &=\underbrace{V_{\mathrm{fix}}-V_{\mathrm{frozen}}}_{
      \text{fixed-action transport loss}}\notag\\
 &\quad+\underbrace{V_Z-V_{\mathrm{fix}}}_{\text{observable value}}\notag\\
 &\quad+\underbrace{V_{\mathrm{or}}-V_Z}_{\text{residual information debt}}.
 \label{eq:four-value-decomposition}
\end{align}
\end{proposition}

\begin{proof}
The first inequality follows because $V_{\mathrm{fix}}$ maximizes over a set
containing $a_{\mathrm{FIT}}$.  Conditioning a fixed action on $Z$ and then
maximizing cannot reduce value, which gives
$V_{\mathrm{fix}}\le V_Z$.  Conditional expected utility cannot exceed the
pointwise evaluator maximum, so $V_Z\le V_{\mathrm{or}}$.  Adding and
subtracting $V_{\mathrm{fix}}$ and $V_Z$ gives
\eqref{eq:four-value-decomposition}.
\end{proof}

\begin{proposition}[Exchangeable realized-maximum baseline]
\label{prop:exchangeable-max}
For two integrable actions with an exchangeable joint law conditional on $S$,
$\mathbb E[U_0\mid S]=\mathbb E[U_1\mid S]$ and
\begin{equation}\label{eq:exchangeable-max-bias}
 \mathbb E[\max\{U_0,U_1\}\mid S]
 -\max_a\mathbb E[U_a\mid S]
 =\frac{1}{2}\mathbb E[|U_1-U_0|\mid S].
\end{equation}
The right-hand side is positive whenever realized action outcomes disagree
with positive probability, even though the conditional expected action values
are identical.  For independent Bernoulli$(p)$ actions it equals $p(1-p)$.
\end{proposition}

\begin{proof}
Use $\max\{x,y\}=(x+y+|x-y|)/2$ and exchangeability to replace
$(\mathbb E[U_0\mid S]+\mathbb E[U_1\mid S])/2$ by their common conditional
mean.  Independence in the Bernoulli case gives
$\mathbb E|U_1-U_0|=2p(1-p)$.
\end{proof}

\begin{remark}[General two-action identity]
Exchangeability is not required for the realized-maximum identity.  Let
$\Delta=U_1-U_0$.  For any two integrable actions and any event $S$ with
positive probability,
\begin{equation}\label{eq:general-realized-max}
\begin{split}
 &\mathbb E[\max\{U_0,U_1\}\mid S]
 -\max_{a\in\{0,1\}}\mathbb E[U_a\mid S]\\
 &\qquad=\frac12\mathbb E[|\Delta|\mid S]
 -\frac12\left|\mathbb E[\Delta\mid S]\right|.
\end{split}
\end{equation}
Thus the exchangeable value $\tfrac12\mathbb E[|\Delta|\mid S]$ is the
largest expected gap compatible with a fixed disagreement rate.  A shortfall
from it records a marginal action difference, not observable conditional
heterogeneity.  The statistic contains no stopped-history signal.
\end{remark}

Proposition~\ref{prop:exchangeable-max} distinguishes a clairvoyant
realization selector from action heterogeneity in conditional expectation.
It also warns against treating an empirical maximum as a neutral plug-in for
$V_{\mathrm{or}}$: sample maxima require an explicit reference distribution.
The exact-fold exchangeable reference in Eq.~\eqref{eq:exchangeable-null}
provides that calibration for the opened MBPP+ audit.

When second-action utility preserves a correct stop,
$U_a=\max\{Y^0,Y^a\}$, define
$V_{\mathrm{stop}}=\mathbb E[Y^0\mid S]$.  Then
$V_{\mathrm{stop}}\le V_{\mathrm{frozen}}$, and the same identity can be
anchored at stopping debt:
\begin{align}
 V_{\mathrm{or}}-V_{\mathrm{stop}}
 &=(V_{\mathrm{frozen}}-V_{\mathrm{stop}})
 +(V_{\mathrm{fix}}-V_{\mathrm{frozen}})\notag\\
 &\quad+(V_Z-V_{\mathrm{fix}})+(V_{\mathrm{or}}-V_Z).
 \label{eq:stop-anchored-decomposition}
\end{align}
This separates recovery supplied by the frozen action from action-ranking
transport, observable selection, and the remaining evaluator-only ceiling.

The decomposition is a population identity for one declared stopping event.
An empirical TEST-best fixed action is a selected maximum and can be upward
biased; inserting it into Eq.~\eqref{eq:four-value-decomposition} does not by
itself estimate the corresponding population terms.  We therefore report the
LiveCodeBench fixed-action and realized-maximum values descriptively in
Section~\ref{sec:ladder-results}, without allocating an empirical
``information debt.''  The all-episode MBPP+ audit uses $S=\Omega$ and
therefore cannot instantiate the Gate-1 decomposition.  Its target mixes
stopped and non-stopped episodes; the exact population accounting appears in
Section~\ref{sec:rq2-results}.

\subsection{Finite-sample action support}

For FIT stopped-history episode $j$, let
$\Delta_j=U_j(a_{\mathrm{alt}})-U_j(a_{\mathrm{FIT}})$ be the alternative-action
advantage used by the frozen ladder.

\begin{proposition}[Zero-support degeneracy]
\label{prop:zero-support-main}
For the ridge-plus-isotonic learner used in the ladder, if $\Delta_j=0$ for every
FIT episode, every fitted and calibrated prediction is zero.  If deviation
requires $\widehat\Delta(Z)>\tau$ with $\tau\ge0$, the policy makes no DEV or TEST
deviation regardless of the feature tier.  Equality with
$V_{\mathrm{frozen}}$ therefore does not imply
$V_Z=V_{\mathrm{fix}}$.
\end{proposition}

\begin{proof}
The ridge response vector is zero, so every positive-penalty normal equation
has zero coefficients and intercept.  Out-of-fold predictions are zero, and
isotonic calibration against the same all-zero target remains zero.  The
strict inequality $0>\tau$ is false for every nonnegative threshold.  This
argument constrains the fitted algorithm, not the held-out joint law of
$(\Delta,Z)$.
\end{proof}

\section{Experimental Setup}\label{sec:setup}

The experiments are organized by gate rather than by artifact-production
chronology.  Table~\ref{tab:evidence-map} assigns one inferential role to each
evidence family.  MBPP+ decides the recovery mechanism; LiveCodeBench exposes
the support problem retrospectively; and BigCodeBench tests the resulting
opening guard prospectively.  The opened-population audit, model-family
extension, and historical archive remain diagnostic, robustness, or
subtractive evidence.

\begin{table*}[t]
\centering
\caption{Evidence map.  Each dataset serves one declared gate role; sharing a
response bank or TEST population does not create an independent replication.
The rows are separate empirical applications of the framework; gates are not
composed across datasets or rows.}
\label{tab:evidence-map}
\scriptsize
\begin{tabularx}{\textwidth}{@{}>{\raggedright\arraybackslash}p{0.16\textwidth}>{\raggedright\arraybackslash}p{0.11\textwidth}>{\raggedright\arraybackslash}p{0.18\textwidth}>{\raggedright\arraybackslash}p{0.13\textwidth}>{\raggedright\arraybackslash}X@{}}
\toprule
Evidence & Gate role & Design & Inferential role & Licensed conclusion \\
\midrule
Fresh MBPP+ Qwen pair & Gate 1 & 152 hash-split TEST queries; pinned generation and rescoring & prespecified primary & whether Base-only stopping debt is recoverable \\
MBPP+ matched-information audit & descriptive audit & nested query cross-fitting on all 1,520 episodes of the opened TEST bank & post-audit diagnostic & realized-maximum calibration, marginal fixed-action comparison, and no tested outcome-blind selector identified when to deviate from the fixed action \\
LiveCodeBench v5+v6 & Gate 2 diagnosis & 136 FIT, 68 DEV, one 137-query sealed TEST opening & separated frozen follow-up & support-shift diagnosis for the declared ladder \\
BigCodeBench v0.1.4 & Gate 2 prospective & 483 FIT, 211 unopened DEV, 405 unopened TEST; near-duplicate-group split & preregistered gate replication & frozen design misses two-sided count minima; no feature fit or outcome opening \\
Llama/\allowbreak Nemotron extension & robustness & same MBPP+ queries and Qwen14 stop bank; new alternative banks & post-outcome sensitivity & model-family persistence of the mechanism only \\
Historical 11-model archive & audit & fixed GSM8K, MATH-500, GPQA, and HumanEval+ replay & exploratory & why broad policy and static-routing claims are rejected \\
\bottomrule
\end{tabularx}
\end{table*}

\subsection{Fresh MBPP+ verifier-gated TEST}

The headline experiment uses the MBPP source corpus~\cite{austin2021mbpp} and
MBPP+ v0.2.0 as distributed by EvalPlus 0.3.1~\cite{evalplus}.
The dataset contains 378 queries and is pinned by SHA-256; a hash-defined outer
split assigns ranks 226--377 to TEST, yielding $N=152$ queries.  The earlier
one-query pipeline pilot, MBPP/100, is excluded from every scientific split.
We generate $D=10$ stochastic candidates per TEST query from
Qwen2.5-7B-Instruct-AWQ and Qwen2.5-14B-Instruct-AWQ
~\cite{qwen25technical,lin2024awq} under pinned repository
revisions, model-snapshot manifests, tokenizer and prompt bytes, sampling
parameters, and a two-GPU tensor-parallel runtime.  The scored bank therefore
contains $152\times10\times2=3{,}040$ candidates.  A separately recorded
greedy candidate per query is generation provenance and is not part of the
estimands.

EvalPlus Base-test passage is the observable verifier $Z$; hidden correctness
$Y$ requires both Base and Plus passage.  Scoring runs in a fixed EvalPlus
container with no network, a read-only root filesystem, dropped capabilities,
and bounded CPU, memory, processes, and temporary storage.

Before the first sealed evaluation, the generated TEST banks, scoring inputs,
scorer, runner, loader, and the then-current evaluator candidate were frozen.
The first publisher execution completed both sealed scoring replays and
computed the aggregate outcome using evaluator SHA
\texttt{4ea5053c6c33f034\allowbreak 83f7cbcc8872c8b6\allowbreak b851fd2946dd77c6\allowbreak d8bc3345c2d2005c};
however, strict JCS validation rejected publication because
\texttt{\_fraction\_record} represented its \texttt{decimal} field as a Python
float.  Thus no canonical result artifact was published, although the sealed
outcome had already been computed.  After this opening, the evaluator was
changed to SHA
\texttt{e1a41ceb2427726a\allowbreak dffd3ff7e62c2b27\allowbreak d5aae8718ba4527e\allowbreak 905575c9d346bca2}.
The change was limited to representing that field as a round-trip decimal
string, and the \texttt{\_query\_estimands} output was byte-identical.  The
original authority record is retained under \texttt{invalidated/}; the fixed
publisher was rerun and the canonical artifacts were frozen.  Accordingly,
the generation and scoring inputs were frozen before outcome opening, but the
final evaluator source closure was not.

The sealed TEST opening was not authorized by an independent external
reviewer.  Checkpoint F2-056 requested an external Claude review and named an
expected review artifact, but that review was never delivered.  After the raw
TEST banks were frozen, the authors issued local checkpoint F2-057 and used it
to authorize the fixed scoring and opening procedure.  Checkpoint F2-057 was
amended after the first invalidated sealed computation to record the corrected
evaluator digest; the version retained on disk therefore postdates that
computation.  We therefore classify the opening as self-authorized rather than
independently reviewed.  Its scope
was limited to the two preregistered Qwen TEST banks, CPU-only scoring, the
fixed zero-argument publisher, and one sealed opening; it did not authorize a
32B model, additional benchmark outcomes, pre-opening policy fitting, or
static prompt routing.

The final authority re-executed both raw banks in fresh containers and
constructed evaluator input only from those new receipts; stored child
evidence served as provenance but not as final truth.  All result files are
canonical JSON, content-addressed, new-only, and frozen after publication.
Recorded generation-token counts include the terminal stop token; no
token-based cost enters either estimand.

\subsection{Matched-information audit protocol}

The audit reuses the frozen MBPP+ raw and evaluator banks but asks a different
one-step question.  Every episode begins with one Qwen2.5-14B completion at
draw position $d$.  The second action either reroutes to Qwen2.5-7B or
resamples Qwen2.5-14B at the fixed position $(d+1)\bmod10$.  The primary costs
are the frozen nominal parameter-scale values 70 and 140 tenths; hence the
best fixed action and all cost tie-breaks are defined before fitting.  No
additional generation or correctness execution is required.

This construction assigns one second action to every archived episode,
irrespective of whether the observable verifier stopped.  Its population
therefore includes both observable stops and verifier rejections and differs
from the Gate-1 stopped population.  It is retained to describe the opened
banks and tested controllers, not to complete a causal stopped-history
sequence.

We evaluate three nested information regimes.  \textsc{Null} supplies only an
intercept.  \textsc{Structure} adds input and generated-token counts, byte and
line counts, finish reason, fenced-code extraction, delimiter balance, Python
parse status, expected-function coverage, and 13 AST node counts.
\textsc{Structure+Hash} adds a normalized 256-dimensional signed hash of code
token unigrams and bigrams.  Raw text, token IDs, reference solutions, hidden
tests, and correctness are not passed directly to the estimator.  The visible
projection is physically materialized by a function that accepts only the raw
generation bank; a separate evaluator loader accepts only score receipts.

Five deterministic outer query folds provide held-out estimates.  Within each
outer training partition, four inner query folds choose
$\lambda\in\{0.1,1,10,100\}$.  All ten positions of a query remain together.
Paired percentile intervals use 10,000 deterministic query-cluster bootstrap
replicates.  The complete analysis was run with 24 workers and repeated with
one worker; after normalizing only the recorded worker count, the canonical
reports are byte-identical.

The post-audit exchangeable reference independently swaps the two action
labels within every episode, preserves these exact outer folds, reselects each
fold's fixed action on its training queries, and recomputes the
realized-maximum gap.  We use $10^6$ deterministic permutations with seed
20260902.  This analysis was added after the adversarial audit and is reported
as calibration, not as a preregistered test.

The four-model sensitivity added post-outcome Llama and Nemotron banks as
second actions.  Because it was constructed after outcome opening and used an
equal synthetic cost solely for tie-breaking, it is confined to the
Supplementary Material as exploratory provenance and is not used to support
action heterogeneity, observability, or efficiency.

\subsection{Prespecified LiveCodeBench observability ladder}
\label{sec:ladder-setup}

We then ran one bounded follow-up on a new dataset, without changing the two-
Qwen action set.  LiveCodeBench \texttt{code\_generation\_lite}
v5+v6~\cite{jain2025livecodebench} is
pinned to one repository revision and one 342-task manifest: one pipeline
pilot, 136 FIT queries, 68 DEV queries, and 137 sealed TEST queries.  For each
scientific query, three blocks with distinct domain-separated deterministic
seeds begin with at most three
sequential Qwen2.5-14B attempts.  The first attempt passing all public tests
defines an observable stopped history; blocks with no public-test pass are not
controller episodes.  Counterfactual actions from disjoint seed namespaces
either resample Qwen2.5-14B or reroute to Qwen2.5-7B.  Full
public-plus-private correctness remains evaluator only.

For stopped block $b$ of eligible query $i$, let $Y^0_{ib}$ be the evaluator-only
correctness of the public-test-passing stop and $Y^a_{ib}$ the correctness of
second action $a\in\mathcal A=\{\textsc{Resample},\textsc{Reroute}\}$.  The
implemented episode utility and split value are
\begin{align}
 U_{ib}(a)&=\max\{Y^0_{ib},Y^a_{ib}\},
 \label{eq:ladder-utility}\\
 \widehat V_r(\pi)&=\frac{1}{|\mathcal I_r^+|}
  \sum_{i\in\mathcal I_r^+}\frac{1}{|\mathcal B_i|}
  \sum_{b\in\mathcal B_i}U_{ib}(\pi(Z_{ib})),
 \label{eq:ladder-value}
\end{align}
where $\mathcal I_r^+$ contains only split-$r$ queries with at least one
natural stop and $\mathcal B_i$ contains their stopped blocks.  Every policy
must take exactly one second action; stop-only is a descriptive reference, not
a member of $\mathcal A$.  Consequently the TEST estimand covers 70 eligible
queries, not end-to-end accuracy over all 137 TEST queries.

The cumulative information tiers are L1 structural and cost fields; L2 adds
token-confidence summaries; L3 adds public-verifier traces and a FIT-calibrated
verifier score; and L4 adds outcome-blind agreement over the observed pre-stop
history.  Every tier uses the same ridge-plus-isotonic selective family.  FIT
chooses the fixed action and estimates action advantage; DEV chooses a ridge
penalty and a nonnegative deviation threshold under a 105\% token-and-latency
gate.  TEST is evaluated once, with three blocks kept in their query cluster
and eligible queries weighted equally.  A separate visible-only materializer
was compared with the opened evaluator projection after removing evaluator
objects; the stopped-history rows matched exactly.  The controller, single-use
TEST-opening claim, isolated scorer chain, feature artifacts, and final result
are all SHA-256 bound in the Supplementary Material.

Ridge fitting, action-isotonic calibration, and verifier calibration give each
FIT stopped episode equal weight.  Fixed-action selection and all DEV/TEST
utility, token, latency, and inferential summaries instead first average
stopped episodes within query and then weight eligible queries equally as in
\eqref{eq:ladder-value}.  This difference is immaterial to the observed
all-zero FIT advantage target but is part of the estimand contract.

\begin{algorithm*}[t]
\caption{Support-aware observability ladder}
\label{alg:observability-ladder}
\begin{algorithmic}[1]
\State \textbf{Input:} FIT and DEV stopped histories; frozen tiers
$Z_1\subset\cdots\subset Z_4$; sealed TEST; actions $\mathcal A$
\State choose $a_{\mathrm{FIT}}$ by FIT equal-query value; break an exact tie
in favor of \textsc{Resample}
\For{$k\in\{1,2,3,4\}$}
  \State set $\Delta=U(a_{\mathrm{alt}})-U(a_{\mathrm{FIT}})$ on FIT stopped episodes
  \For{$\lambda\in\{10^{-4},10^{-3},\ldots,10^4\}$}
    \State obtain five query-fold out-of-fold ridge predictions; fit an
    isotonic calibrator; refit ridge on all FIT
    \For{$\tau\in\{0,.01,.02,\ldots,.50\}$}
      \State on DEV, deviate iff $\widehat\Delta_k>\tau$; retain candidates with
      total-token and second-action-latency ratios both $\le1.05$
    \EndFor
  \EndFor
  \State freeze the candidate with highest DEV equal-query utility, then
  fewest deviations, largest $\tau$, and largest $\lambda$
\EndFor
\State open TEST once; evaluate every frozen tier and a deterministic
deviation-count-matched random control
\State use 20,000 PCG64 query-cluster replicates (seed 20260830404; random
comparison seed $+1$) for one-sided studentized max-$T$ simultaneous bounds
\State a tier passes only if both gain lower bounds are strictly positive and
its cost gate passes; return the least informative passing tier, or the frozen
no-tested-signal verdict
\end{algorithmic}
\end{algorithm*}

\subsection{Preregistered BigCodeBench support-gate replication}
\label{sec:bcb-setup}

We subsequently froze a stricter replication on BigCodeBench v0.1.4
\cite{bigcodebench}.  The source parquet contains 1,140 tasks; pinned scorer,
qualification, and near-duplicate-component contracts retain 1,100 tasks and
assign one pipeline pilot, 483 FIT, 211 DEV, and 405 TEST tasks.  Each FIT task
has three history episodes of at most three Qwen2.5-14B attempts.  The first
attempt passing the frozen public-test subset is the stop event.  Only those
stopped episodes receive one disjoint Qwen2.5-14B resample and one
Qwen2.5-7B reroute; 351 FIT queries yield 895 stopped episodes, while 132 FIT
queries yield no action draw.  Models, revisions, prompts, seeds, and the
two-action set are fixed; no 32B model, additional family, or prompt router is
introduced.

The 483 FIT tasks are not a model-generated convenience sample.  After
qualification, near-duplicate prompt components are kept indivisible and a
source-only deterministic hash assigns components to roles before any action
outcome is scored.  Model generation then induces a second selection layer:
only tasks and episodes that reach the frozen public-test stop enter the
stopped FIT bank.  Finally, the Qwen14 resample, Qwen7 reroute, episode budget,
and official evaluator induce the action-difference labels.  If
$\vartheta$ denotes this complete frozen design, the Gate-2 sample is drawn
from the conditional design distribution
\begin{equation}\label{eq:fit-induced-distribution}
 \mathcal P_{\vartheta}
 =\mathcal L\!\left(I,\mathcal H,\Delta\mid
   \mathrm{qualified}=1,\ \mathrm{role}=\mathrm{FIT},\ S=1;
   \vartheta\right),
\end{equation}
not from BigCodeBench in isolation.  Here $I$ is task identity, $\mathcal H$ is the
stopped history, $S$ is the public-stop event, and $\Delta$ is
Eq.~\eqref{eq:gate-action-difference}.  This distinction controls the later
attribution: a failed gate diagnoses this frozen design distribution, not one
component as a cause.

For each stopped episode the evaluator records exactly one official full-suite
realization per arm and defines $U(a)=\max\{Y^0,Y^a\}$ as in
Eq.~\eqref{eq:ladder-utility}.  Public replay and the equality between
separately executed public/full/evaluator-only partitions are diagnostic, not
label-selection criteria.  This amendment is necessary because the official
tests need not share common evaluator randomness across separate processes;
retrying until partition equality holds would condition the accepted labels.
Code, digest, task, and container provenance remain fail-closed.

Before fitting L1--L4, the frozen gate requires at least 25 positive
($U(\textsc{Reroute})>U(\textsc{Resample})$) and 25 negative advantage
episodes, at least 20 distinct FIT queries of each sign, at least two queries
of each sign in every one of five folds, and one-sided exact 97.5\% lower
prevalence bounds above 1\%.  It also requires at least 200 stopped queries,
400 stopped episodes, and both stop-correctness classes to have at least 25
episodes and 20 queries.  Any failure returns
\texttt{STOP\_\allowbreak INSUFFICIENT\_\allowbreak TWO\_\allowbreak SIDED\_\allowbreak FIT\_\allowbreak SUPPORT};
L1--L4 are then not fitted and DEV/TEST outcomes remain unopened.

Formally, let $\mathcal E$ be the stopped FIT episodes, $\mathcal Q$ their
distinct queries, $N_s$ the number of episodes with advantage sign
$s\in\{-1,+1\}$, $Q_s$ the number of queries containing at least one such
episode, and $Q_{sf}$ the corresponding count in frozen fold $f$.  Let
$C_y^E$ and $C_y^Q$ be the episode and query counts for stop-correctness
class $y\in\{0,1\}$, and let
$L_s=\mathrm{Beta}^{-1}_{0.025}(Q_s,|\mathcal Q|-Q_s+1)$ be the one-sided
Clopper--Pearson lower bound.  The authorization bit is the conjunction
\begin{equation}
\begin{split}
G_{\mathrm{FIT}}={}&
\mathbf 1\{|\mathcal E|\ge400,\ |\mathcal Q|\ge200\}\\
&\land\!\bigwedge_{s\in\{-1,+1\}}
\bigl(\mathbf 1\{N_s\ge25,\ Q_s\ge20\}\\
&\qquad\qquad\land\ \mathbf 1\{\min_f Q_{sf}\ge2,\ L_s>0.01\}\bigr)\\
&\land\!\bigwedge_{y\in\{0,1\}}
\mathbf 1\{C_y^E\ge25,\ C_y^Q\ge20\}.
\end{split}
\label{eq:bcb-fit-gate}
\end{equation}
Thus a large total sample cannot compensate for a missing sign, fold, or stop
class.  Algorithm~\ref{alg:bcb-opening-guard} makes the corresponding
fail-closed transition explicit.

\begin{algorithm*}[t]
\caption{BigCodeBench FIT support gate and guarded opening}
\label{alg:bcb-opening-guard}
\begin{algorithmic}[1]
\Require frozen FIT banks, source/config digests, five folds, and gate minima
\State validate canonical bytes, coordinates, draw digests, container audits,
and zero DEV/TEST reads; otherwise stop with a provenance error
\For{each stopped FIT episode $(i,b)\in\mathcal E$}
  \State score history, resample, and reroute exactly once on the official
  full suite; compute $\Delta_{ib}$
  \State record public replay and partition agreement as diagnostics only;
  never retry or select a label from their values
\EndFor
\State compute $N_s,Q_s,Q_{sf},L_s,C_y^E,C_y^Q$ and $G_{\mathrm{FIT}}$ in
Eq.~\eqref{eq:bcb-fit-gate}; freeze the canonical gate receipt
\If{$G_{\mathrm{FIT}}=0$}
  \State lock L1--L4, DEV, and TEST
  \State \Return \texttt{STOP\_INSUFFICIENT\_TWO\_SIDED\_FIT\_SUPPORT}
\EndIf
\State fit L1--L4 on FIT and authorize DEV only; freeze the DEV-selected
controller and threshold
\State open TEST exactly once only if the separate opening guard validates
every frozen prerequisite
\end{algorithmic}
\end{algorithm*}

\subsection{Post-outcome model-family robustness extension}

After the prespecified Qwen14-to-Qwen7 outcome was opened, we froze a separate
two-arm robustness protocol.  It reuses the same 152 TEST queries and the same
frozen Qwen2.5-14B first-stage score bank, but independently generates and
rescores ten draws per query from
Llama-3.1-8B-Instruct and NVIDIA-Nemotron-Nano-9B-v2
~\cite{llama31herd,nemotronnano2}.  For alternative model
$m$, let $F_i$ be the number of Qwen14 draws among ten that pass Base but fail
Plus and let $S_{im}$ be the number of alternative draws that pass both.  The
extension estimator is
\begin{equation}
  \widehat\Delta_m
  = \frac{1}{N}\sum_{i=1}^{N}\frac{F_i}{10}\frac{S_{im}}{10}.
  \label{eq:model-family-extension}
\end{equation}
It is the same independent-pairing recovery estimand as the primary mechanism,
now evaluated with two non-Qwen alternatives.  We retain the fixed offset-$1$
pairing sensitivity.  Each arm receives a descriptive 95\% query-cluster
bootstrap interval; the familywise decision requires both two-sided 97.5\%
Bonferroni interval lower bounds to exceed zero over 10,000 deterministic
replicates.

The alternative banks use pinned Hub revisions and strict snapshot manifests.
They were generated sequentially on one RTX 4090 (TP=1), fully unloaded between
models, and rescored in the same fixed network-disabled EvalPlus image.  The
Nemotron prompt uses its documented native \texttt{/no\_think} directive and a
float32 Mamba state-space cache.  The original Qwen14 bank retains its frozen
TP=2 generation contract.  Thus this extension tests model-family robustness
of the finite-bank mechanism, not hardware-configuration invariance or an
independent repetition of the TEST split.

\subsection{Historical replay audit boundary}

The legacy 11-model replay over GSM8K, MATH-500, GPQA-Diamond, and HumanEval+
~\cite{gsm8k,mathdataset,lightman2024verify,gpqa,humaneval,evalplus}
is retained solely as subtractive provenance.  Its benchmark, scorer,
selector, cost, figure, and artifact details appear in the Supplementary
Material under ``Historical replay audit boundary'' and do not determine any
headline gate.

\section{Results}\label{sec:results}

The results are presented in gate order for exposition, not as a chain
composed across datasets.  MBPP+, LiveCodeBench, and BigCodeBench each have a
different stopping event and target population.  MBPP+ measures a bounded
recovery mechanism; LiveCodeBench retrospectively exposes zero FIT
action-difference support within its own population; BigCodeBench
prospectively applies a support gate within another population.  The
matched-information result is reported separately as an all-episode
descriptive audit.  Model-family and historical analyses are subordinate
robustness or audit evidence.

\subsection{Gate 1: Weak-verifier stopping debt is recoverable}
\label{sec:rq1-results}

\begin{table*}[t]
\centering
\caption{Fresh MBPP+ verifier-gated TEST results over 152 query clusters.
Intervals are deterministic 10,000-replicate query-cluster percentile
bootstrap intervals.  Effects are absolute percentage-point changes.}
\label{tab:verifier-gated-test}
\small
\setlength{\tabcolsep}{5pt}
\begin{tabularx}{\textwidth}{@{}>{\raggedright\arraybackslash}Xrr>{\raggedright\arraybackslash}p{0.23\textwidth}@{}}
\toprule
Estimand & Point & 95\% interval & Role \\
\midrule
14B false-positive-stop recovery by a uniformly selected 7B bank draw
  & $+2.592$ & $[+1.618,+3.664]$ & Prespecified primary; lower bound $>0$ \\
Fixed offset-$1$ pairing
  & $+2.500$ & $[+1.513,+3.618]$ & Pairing sensitivity \\
Post-audit same-model fixed-offset resample
  & $+2.039$ & --- & Descriptive control; no preregistered inference \\
After 7B Base-test rejection: fixed reroute to 14B minus leave-one-out fixed
resample of 7B
  & $+2.882$ & $[+0.931,+5.201]$
  & Prespecified underpowered secondary; non-stopped event \\
\bottomrule
\end{tabularx}
\end{table*}

The primary estimate in Table~\ref{tab:verifier-gated-test} is
$197/7{,}600=0.025921\ldots$, or $+2.592$ percentage points.  Its
query-cluster bootstrap interval lies entirely above zero, so the
prespecified decision is ``primary mechanism present: lower 95\% bound above
zero.''  The
fixed offset-$1$ construction gives $+2.500$ points with a similarly positive
interval.  The conclusion therefore does not depend on multiplying the two
within-query draw means in \eqref{eq:vg-primary}.  On the complementary
observable event---a 7B Base-test rejection, occurring on 294 of 1,520
draws---the prespecified secondary comparison in
Eq.~\eqref{eq:vg-secondary} favors fixed rerouting to 14B over leave-one-out
resampling of 7B by $3943/136800=+2.882$ points
$[+0.931,+5.201]$.  The event-weighted components are
$7.151$ points for rerouting and $4.269$ points for resampling.  This is a
prespecified but underpowered marginal fixed-action
ranking on a non-stopped event; it is neither a recovery estimand nor a
conditional selector.

A post-audit same-model control further limits the mechanism claim.  At the
same fixed offset, Qwen14 resampling recovers $31/1{,}520=+2.039$ points and
Qwen7 rerouting recovers $38/1{,}520=+2.500$ points, a descriptive difference
of only $+0.461$ points.  This comparison was not preregistered and has no
confirmatory interval; it shows that Gate 1 establishes recoverable stopping
debt for the 14B-first false-positive-stop primary, not a rerouting-specific
causal mechanism within that event.  The separate secondary $Z^7=0$
comparison has a different event and estimand.

The effect is conditional complementarity, not evidence that one model is
globally stronger.  Across all stochastic TEST draws, full-correctness rates
are 68.36\% for 7B (1,039/1,520) and 68.55\% for 14B (1,042/1,520).  Base-only
false-positive-stop rates are 12.30\% (187/1,520) and 13.03\% (198/1,520),
respectively.  The measured recovery arises because a draw from the
alternative model can be fully correct on queries where a Base-only verifier
would irreversibly accept an incorrect first-model completion.

\begin{figure*}[t]
\centering
\includegraphics[width=\textwidth]{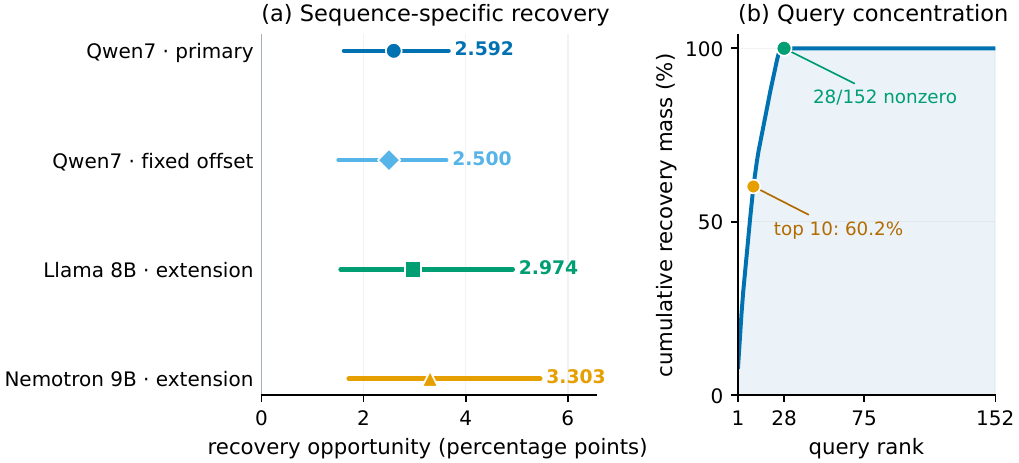}
\caption{Fresh MBPP+ verifier-gated recovery.  Panel (a) shows the
prespecified Qwen7 primary estimate and fixed-offset sensitivity with 95\%
query-cluster intervals.  Llama and Nemotron use familywise-corrected 97.5\%
intervals and are explicitly post-outcome extensions.  Panel (b) orders the
152 primary query-cluster contributions and shows their cumulative share of
the finite-bank recovery mass.}
\Description{A forest plot shows positive recovery estimates for Qwen7,
Llama 8B, and Nemotron 9B, with all intervals above zero.  A cumulative curve
shows that 28 of 152 query clusters contribute nonzero primary recovery mass
and that the top ten account for about sixty percent.}
\label{fig:fresh-recovery}
\end{figure*}

The mechanism is concentrated rather than diffuse.  Only 28 of the 152 TEST
query clusters have a nonzero contribution to the primary finite-bank
estimand; the ten largest account for 60.2\% of its total mass
(Figure~\ref{fig:fresh-recovery}b).  This concentration reinforces the
sequence-specific interpretation and is why uncertainty resamples complete
query clusters instead of treating the 7,600 draw pairs as independent.  It
does not license a selective controller class: the contributing clusters were
identified using hidden correctness and their predictability from legal
runtime information is not tested by Gate 1.

Because generation used a fixed draw-index-to-batch-position map, we also
tested draw-position exchangeability before interpreting draw averages.
Cochran's $Q$ tests for correctness and Friedman tests for generated-token
count were corrected together by Holm's procedure.  None of the four nulls
was rejected familywise: the unadjusted $p$-values were 0.808 and 0.041 for
7B/14B correctness and 0.037 and 0.614 for 7B/14B token count.  This is no
proof of exchangeability---two unadjusted diagnostics are borderline---so the
fixed-offset result remains the more conservative pairing check.

The TEST publication was produced only after both 1,520-row score banks were
re-executed in fresh isolated containers and their receipts were rebound to a
new authority.  This establishes the finite-bank mechanism result specified
in Section~\ref{sec:mechanism-estimand}.  It does not compare an implemented
dynamic controller against a cost-matched deployment baseline, and it does
not support extrapolation to a 32B model or another benchmark.

\subsection{Gate 1 robustness across two model families}

\begin{table*}[t]
\centering
\caption{Post-outcome model-family robustness extension over the same 152
MBPP+ TEST query clusters and frozen Qwen2.5-14B first-stage bank.  Recovery
effects are absolute percentage points.  The 97.5\% intervals implement the
prespecified Bonferroni familywise rule for the two extension arms.}
\label{tab:model-family-extension}
\footnotesize
\setlength{\tabcolsep}{3.5pt}
\begin{tabularx}{\textwidth}{@{}>{\raggedright\arraybackslash}Xrrrr@{}}
\toprule
Alternative model & Full-correct & Recovery & 95\% interval & 97.5\% interval \\
\midrule
Llama-3.1-8B-Instruct
  & 61.18\% & $+2.974$ & $[+1.671,+4.632]$ & $[+1.553,+4.921]$ \\
NVIDIA-Nemotron-Nano-9B-v2
  & 71.71\% & $+3.303$ & $[+1.868,+5.145]$ & $[+1.711,+5.461]$ \\
\bottomrule
\end{tabularx}
\begin{minipage}{0.96\linewidth}\footnotesize\raggedright
The fixed offset-$1$ sensitivity estimates are $+3.092$
($97.5\%$ interval $[+1.579,+5.132]$) for Llama and $+3.487$
($[+1.908,+5.658]$) for Nemotron.  Both familywise lower bounds exceed zero,
yielding \texttt{ROBUSTNESS\_SUPPORTED\_BOTH\_ARMS}.  This extension was
specified after the original Qwen14-to-Qwen7 outcome was known and is not an
independent TEST replication.
\end{minipage}
\end{table*}

Table~\ref{tab:model-family-extension} reports the deliberately bounded
post-outcome extension.  With the first-stage Qwen14 bank held fixed, the
independent-pairing recovery opportunity is $113/3{,}800=+2.974$ percentage
points for Llama and $251/7{,}600=+3.303$ points for Nemotron.  Both
familywise-corrected 97.5\% query-cluster intervals lie entirely above zero,
and both fixed offset-$1$ sensitivities remain positive.  The frozen decision
is \texttt{ROBUSTNESS\_SUPPORTED\_BOTH\_ARMS}.  Because the arms were specified
after the primary outcome, share its TEST queries and Qwen14 stop bank, and use
TP=1 alternative generation, they test only Gate-1 model-family persistence;
they are neither independent confirmation nor a wider selector study.

\subsection{Gate 2 diagnosis: the sealed ladder fails first at FIT action support}
\label{sec:ladder-results}
\label{sec:rq3-results}

The follow-up ladder was trained before a single opening of its distinct
LiveCodeBench TEST split.  Among 411 candidate TEST blocks, 155 naturally
stopped histories occur in 70 of 137 queries.  Of these stops, 106 are fully
correct and 49 are stopping-debt episodes.  In the latter set, neither action
rescues 44 episodes, resampling alone rescues two, rerouting alone rescues
three, and both-action rescue occurs zero times.  Thus TEST contains five
realized exclusive rescues among 49 incorrect stops.  These finite-bank
disagreements are support counts, not evidence of population action
heterogeneity.

\begin{table*}[t]
\centering
\caption{Sealed LiveCodeBench observability-ladder result.  Values first
average stopped episodes within query and then average the 70 eligible
queries.  Reroute is a descriptive fixed-action sensitivity; the frozen
primary baseline remains the FIT-selected resample action.}
\label{tab:ladder-test}
\small
\begin{tabularx}{\textwidth}{@{}>{\raggedright\arraybackslash}Xrrl@{}}
\toprule
Rule & TEST value & Learned deviations & Inferential role \\
\midrule
Frozen fixed resample & 0.595238 & --- & preregistered baseline \\
Fixed reroute & 0.611905 & --- & descriptive sensitivity \\
Evaluator-only two-action realized maximum & 0.621429 & --- & clairvoyant ceiling \\
L1 structural controller & 0.595238 & 0 & frozen learned policy \\
L2 $+$ token confidence & 0.595238 & 0 & frozen learned policy \\
L3 $+$ verifier trace/calibration & 0.595238 & 0 & frozen learned policy \\
L4 $+$ stopped-history agreement & 0.595238 & 0 & frozen learned policy \\
\bottomrule
\end{tabularx}
\end{table*}

\begin{figure*}[t]
\centering
\includegraphics[width=\textwidth]{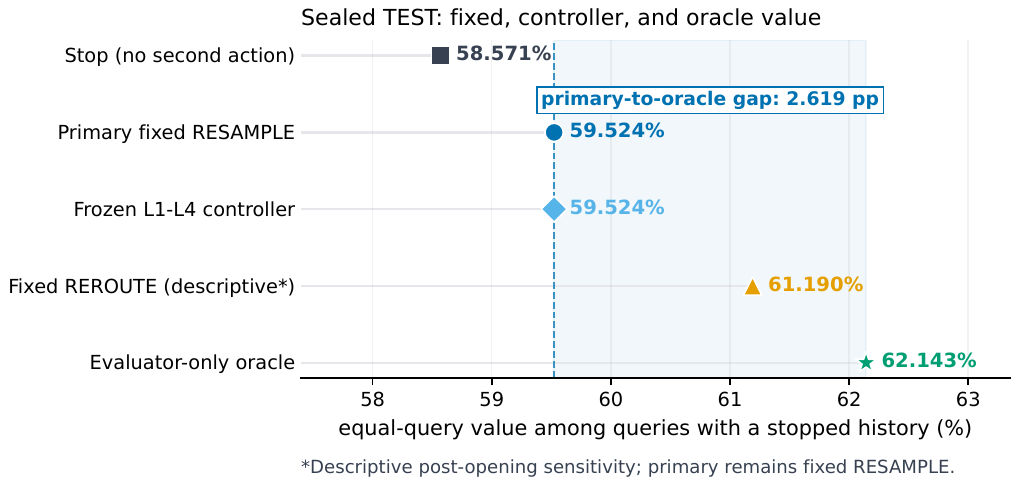}
\caption{Sealed LiveCodeBench TEST value among the 70 queries with at least
one stopped history, using equal query weight and equal stopped-episode weight
within query.  The frozen L1--L4 controllers overlap the preregistered fixed-
resample baseline because every tier makes zero deviations.  Fixed reroute is
a post-opening descriptive sensitivity, and the evaluator-only oracle measures
opportunity rather than achievable controller value.  The shaded 2.619-point
oracle--resample gap contains both fixed-action transport loss and residual
evaluator-only information; it is neither a deployable gain nor a pure information
gap.}
\Description{A horizontal value plot shows stop-only value 58.571 percent,
fixed resample and every frozen controller at 59.524 percent, descriptive
fixed reroute at 61.190 percent, and the evaluator-only oracle at 62.143
percent.}
\label{fig:ladder-value-frontier}
\end{figure*}

Figure~\ref{fig:ladder-value-frontier} separates the frozen controller value
from both the descriptive fixed-reroute sensitivity and the evaluator-only
ceiling.  The decisive fact precedes TEST.  On FIT, 66 of 165 stops are incorrect, but
neither action rescues any of them: the alternative-minus-fixed utility target
is exactly zero for every FIT episode.  The utility-only tie rule therefore
freezes resampling as the fixed action.  A ridge fit to the all-zero target,
followed by isotonic calibration and a strict positive-over-nonnegative-
threshold decision, predicts no deviation for any input.  DEV cannot repair
this absence of FIT support.  Although five of its 41 incorrect stops are
rescued relative to stop-only, the alternative-over-fixed target has counts
$(\Delta>0,\Delta<0,\Delta=0)=(0,4,85)$ on DEV; positive deviation support appears only on
TEST, with counts $(3,2,150)$.  Consequently all four frozen tiers make zero
TEST deviations and attain the frozen resample value 0.595238.  Descriptively,
TEST-best fixed reroute is 1.667 points higher and the realized maximum is
another 0.952 points higher.  Both are selected empirical maxima, so this
arithmetic is not an estimate of the population decomposition in
Eq.~\eqref{eq:four-value-decomposition}.  Because the learner never deviates,
the run cannot estimate observable action value.

\begin{figure*}[t]
\centering
\includegraphics[width=\textwidth]{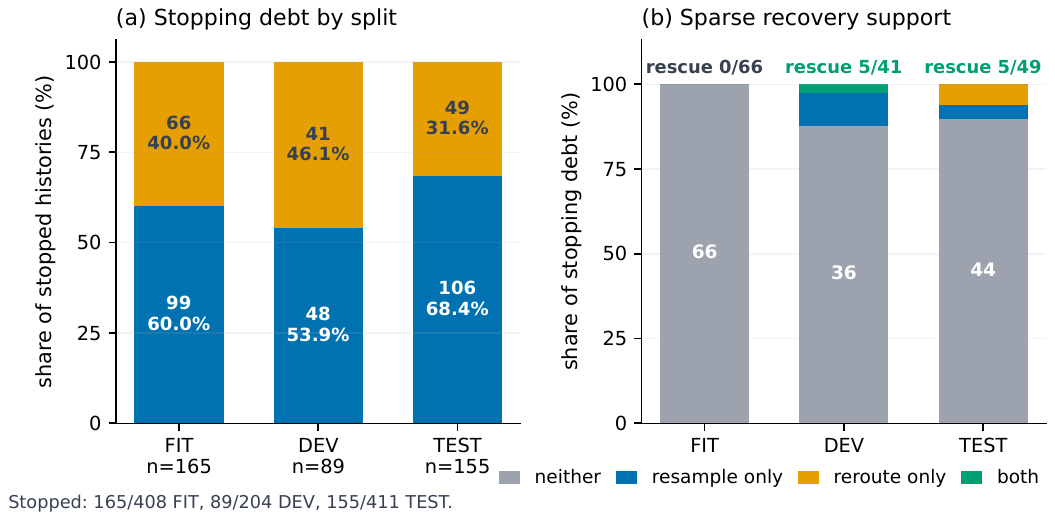}
\caption{Stopped-history composition and action-specific rescue support across
FIT, DEV, and sealed TEST.  Panel (a) reports episode-weighted correct-stop and
stopping-debt shares for descriptive support accounting; these are not the
primary equal-query values.  Panel (b) shows that FIT has no rescue among 66
incorrect stops, whereas DEV and TEST each contain five rescues among 41 and
49 incorrect stops.  These frozen-sample differences document support
nontransport; they do not test the nested feature tiers.}
\Description{Two panels show correct versus incorrect stopped histories and
the resample-only, reroute-only, both-action, or neither-action rescue patterns
across FIT, DEV, and TEST.  FIT has zero rescues; DEV and TEST each have five.}
\label{fig:ladder-support-transport}
\end{figure*}

Figure~\ref{fig:ladder-support-transport} shows that stopping debt is present
in every split while action-specific rescue is absent only on FIT.  This is a
FIT-to-TEST action-support and transport failure, not evidence that
L1--L4 features are generally uninformative.  The algorithm had no nonzero
FIT action-advantage target from which to estimate a feature relationship, so
the sealed result cannot compare the relative sufficiency of the four tiers.
The exact machine verdict
\nolinkurl{OBSERVABILITY_LADDER_NO_TESTED_SIGNAL_BEATS_FIXED} remains the
correct protocol outcome, while the scientific diagnosis is the narrower
finite-sample FIT support and action-ranking transport failure.  After each
OPEN episode row's evaluator field is removed, all 155 scientific
stopped-history rows match their VISIBLE counterparts exactly.  The artifacts
differ in their \texttt{episodes}, \texttt{mode}, and
\texttt{test\_scoring\_authority} fields, but both bind the same materializer
source closure,
\texttt{8a049d9814aa516baee7b1aeebc04454}\allowbreak
\texttt{cb817418f87606b8aaeb8b34e344d3a1}.
Thus row-level feature noninterference passes without claiming byte identity
of the complete artifacts.

The frozen tie also matters.  Always rerouting has a 0.611905 descriptive
TEST value, 1.667 points above the primary resample baseline, but the
comparison is post-opening and descriptive.  It does not establish population
reroute superiority or authorize replacing the preregistered baseline after
outcome opening.  Relative to this descriptive reroute rule, the oracle's
remaining point gap is only 0.952 points.  Both reference choices are
therefore reported; neither changes the zero-deviation controller result.

\subsection{Gate 2 prospective test: BigCodeBench stops before opening}
\label{sec:bcb-results}

The stricter replication reaches its intended terminal negative branch before
any feature fit or held-out opening.  All 895 FIT checkpoints form one
contiguous canonical prefix and bind 351 stopped queries, three draw digests
per episode, one owned scorer-container lifecycle per job, and zero DEV/TEST
flags.  The official full-suite history label finds 221 incorrect stops in
118 queries, i.e., 24.693\% of stopped episodes and 33.618\% of eligible
queries contain stopping debt.  Among those 221 episodes, neither action
rescues 167, resampling alone rescues 22, rerouting alone rescues 23, and both
rescue nine.  Realized action-specific rescue therefore exists, but the sign needed for a
selector is rare.

\begin{table*}[t]
\centering
\caption{Terminal BigCodeBench FIT support audit.  The sign minima were frozen
before full-outcome scoring.  Failure of any hard minimum blocks L1--L4, DEV,
and TEST regardless of the other passing checks.}
\label{tab:bcb-fit-gate}
\small
\begin{tabularx}{\textwidth}{@{}>{\raggedright\arraybackslash}Xrrl@{}}
\toprule
Gate quantity & Observed & Minimum & Result \\
\midrule
Stopped queries / episodes & 351 / 895 & 200 / 400 & pass \\
Reroute-advantage episodes / queries & 23 / 19 & 25 / 20 & fail / fail \\
Resample-advantage episodes / queries & 22 / 19 & 25 / 20 & fail / fail \\
Per-fold queries of each sign & $\ge2$ & 2 & pass \\
One-sided exact lower prevalence, each sign & 3.290\% & $>1$\% & pass \\
Stop-correct / stop-incorrect support & 674 / 221 episodes & 25 each & pass \\
\bottomrule
\end{tabularx}
\end{table*}

The reroute-over-resample sign occurs in 23 episodes from 19 queries; the
reverse sign occurs in 22 episodes from 19 queries.  Thus the run misses the
episode minima by two and three and each query minimum by one.  Both signs are
present in every fold and their exact one-sided lower prevalence bounds are
3.290\%, but those facts do not replace the preregistered absolute minima.
The exact machine verdict is
\texttt{STOP\_INSUFFICIENT\_TWO\_SIDED\_FIT\_SUPPORT}.  Consequently no
L1--L4 controller, DEV selection, or TEST opening exists for this experiment.
This is a support result, not a feature-sufficiency or generalization result.

The direct bottleneck is differential action support, not an absence of tasks
or stopping debt.  Fifty-four incorrect stops are rescued by at least one
action, but nine are rescued by both; only the 22 resample-only and 23
reroute-only episodes contribute a sign to $\Delta$, and those signs occur in only
36 queries.  For $s\in\{-1,+1\}$, write
\begin{equation}\label{eq:expected-fit-support}
 p_s(\vartheta)=\Pr_{\mathcal P_\vartheta}(\Delta=s),
 \qquad
 \mathbb E_\vartheta[N_s\mid |\mathcal E|=n]=n p_s(\vartheta).
\end{equation}
The observed count therefore combines the sign rate induced by the complete
design $\vartheta$ with the finite stopped-episode budget.  One frozen design
cannot identify whether the benchmark distribution, public-stop rule, model
pair, episode budget, or evaluator realization caused the low count.  The
authors own this design choice; the data should not be described as having
``failed'' independently of it.  Component-level attribution requires a new
outcome-blind, preregistered sensitivity or factorial study, not a post-gate
repair of candidate-011.

\begin{figure*}[t]
\centering
\includegraphics[width=\textwidth]{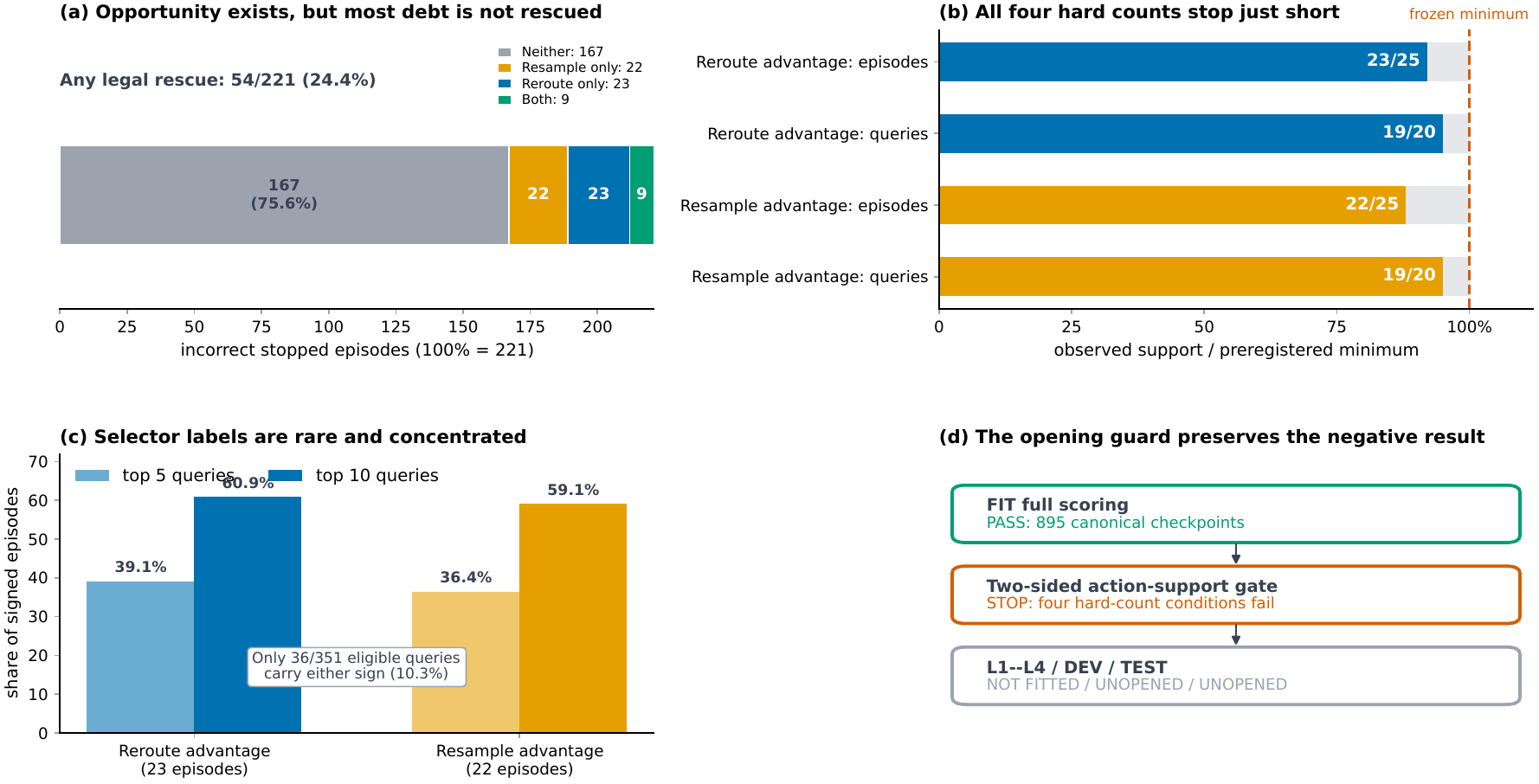}
\caption{BigCodeBench realized recovery, support, distribution, and guarded
decision.  Panel (a) partitions the 221 incorrect stops and shows that 54
(24.4\%) admit at least one legal rescue.  Panel (b) normalizes each signed
episode/query count by its own preregistered minimum, avoiding a mixed-unit
axis.  Panel (c) shows that either sign appears in only 36 of 351 eligible
queries.  Panel (d) gives the fail-closed path: full FIT scoring passes, the
two-sided support gate stops, and L1--L4/DEV/TEST remain unopened.}
\Description{Four panels show the composition of 221 incorrect stops, four
signed counts narrowly below their frozen thresholds, the distribution of
rare action-advantage labels, and a guarded pipeline ending before model
fitting or held-out opening.}
\label{fig:bcb-fit-support}
\end{figure*}

Post-gate descriptive values do not amend this verdict.  With equal query
weight, stop-only, fixed resample, fixed reroute, and the two-action oracle are
0.739791, 0.771130, 0.776353, and 0.799145.  Reroute is 0.522 points above
resample in this FIT bank and uses 169.995 fewer generated tokens per
equal-query action draw (339.710 versus 509.705).  Its separately recorded
batch wall time is also lower (3.899 versus 9.129 seconds per draw), but these
two banks were timed in separate runs, so latency is descriptive rather than a
causal service comparison.  The sample-Pareto ordering of the fixed actions
cannot authorize a static prompt router or bypass the failed two-sided support
gate.

The rare signs occur in only 36 of 351 eligible queries; two queries contain
both signs across episodes.  The ten most sign-bearing queries account for
60.9\% of reroute-favoring and 59.1\% of resample-favoring episodes.  Because
these ten are selected from only 19 sign-bearing queries per direction, the
shares are descriptive and are not used as evidence of exceptional
concentration or a selective policy class.  Evaluator diagnostics
find zero public-stop replay discrepancies but five partition inconsistencies
across four jobs when public/full/evaluator-only suites are executed
separately.  Excluding those four diagnostic jobs leaves 22/19 positive and
21/18 negative episode/query counts, which still fail the frozen minima.
Neither diagnostic exclusion nor episode-versus-query weighting changes the
terminal scientific branch.

\subsection{Descriptive all-episode audit: realized maximum and negative selector result}
\label{sec:rq2-results}

\begin{table*}[t]
\centering
\caption{Descriptive all-episode MBPP+ matched audit.  The original paired
query-cluster intervals describe finite-bank sampling variation; they do not
calibrate the bias induced by taking a pointwise realized maximum.  The
exchangeable reference uses $10^6$ within-episode action-label permutations,
the exact recorded outer folds, and foldwise reselection of the fixed action.}
\label{tab:dynamic-audit}
\scriptsize
\setlength{\tabcolsep}{1.5pt}
\renewcommand{\arraystretch}{1.12}
\begin{tabularx}{0.975\textwidth}{@{}
 p{0.24\textwidth}
>{\centering\arraybackslash}p{0.08\textwidth}
>{\centering\arraybackslash}p{0.19\textwidth}
>{\centering\arraybackslash}p{0.19\textwidth}
 X@{}}
\toprule
Quantity & \shortstack{Estimate\\(pp)} & \shortstack{Bootstrap 95\%\\CI (pp)}
& \shortstack{Exchangeable 95\%\\range (pp)} & Interpretation \\
\midrule
Realized maximum $-$ cross-fitted fixed reroute
  & $+2.697$ & $[+1.711,+3.816]$ & ---
  & Descriptive maximum; not a heterogeneity test \\
Exchangeable-actions reference
  & mean $+3.158$ & --- & $[+2.434,+3.947]$
  & Observed statistic is at the 18.283rd null percentile \\
Structure controller $-$ fixed reroute
  & $0.000$ & $[-0.987,+1.053]$ & ---
  & Neither tested selector improves on the fixed action \\
Structure+hash controller $-$ fixed reroute
  & $-0.197$ & $[-1.184,+0.789]$ & ---
  & Neither tested selector improves on the fixed action \\
\bottomrule
\end{tabularx}
\end{table*}

\begin{figure*}[t]
\centering
\includegraphics[width=\textwidth]{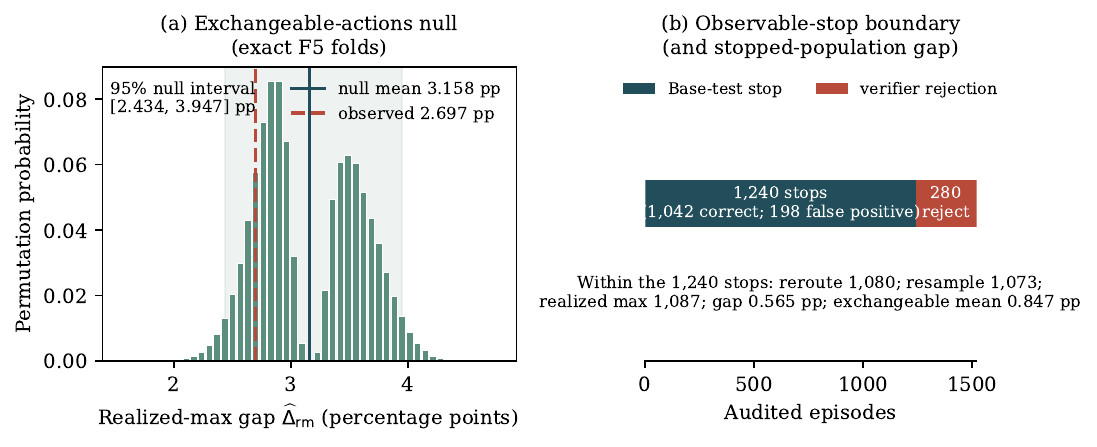}
\caption{Post-audit calibration and population boundary.  Panel (a) shows the
exchangeable-actions reference distribution under the exact recorded F5 outer
folds.  The realized-maximum statistic contains no observable-history signal;
the reference calibrates marginal label symmetry rather than conditional
heterogeneity.  Panel (b) partitions the 1,520 audited episodes into 1,240
observable Base-test stops and 280 verifier rejections on which no stop
occurred.}
\Description{A permutation histogram has mean 3.158 percentage points and a
95 percent interval from 2.434 to 3.947, with the observed 2.697-point value
inside the interval.  The second panel separates observable stops from
verifier rejections and reports the stopped-population diagnostic.}
\label{fig:exchangeable-action-null}
\end{figure*}

The first Qwen14 call is correct on 68.553\% of episodes.  Fixed rerouting
succeeds on 1,167/1,520 episodes and fixed resampling on 1,153/1,520; every
outer fold therefore selects rerouting.  The pointwise realized maximum
succeeds on 1,208/1,520 episodes, giving
$\widehat\Delta_{\mathrm{rm}}=41/1{,}520=+2.697$ points relative to the
cross-fitted fixed action.

This value follows from an identity, not from an observable opportunity
signal.  The two action utilities disagree on 96 episodes: rerouting alone
succeeds on 55 and resampling alone on 41.  Hence
\begin{equation}
 \widehat\Delta_{\mathrm{rm}}
 =\frac{96-|55-41|}{2\times1{,}520}
 =\frac{41}{1{,}520}.
\end{equation}
The corresponding analytic exchangeable mean is
$96/(2\times1{,}520)=3.157895$ points.  One million exact-fold permutations
give mean 3.157851 points, Monte Carlo standard error 0.000448 points, and a
95\% reference interval $[2.434211,3.947368]$.  The observed statistic is at
the 18.283rd percentile.  This placement calibrates marginal action-label
symmetry; it does not test conditional heterogeneity.

The discordance is concentrated outside the stopped population.  Among 1,042
correct stops there are no discordant utilities.  The 198 false-positive stops
contain 21 discordances---14 reroute-only and seven resample-only---whereas the
280 verifier rejections contain 75---41 reroute-only and 34 resample-only.
Thus 75/96 discordances arise where no stop occurred.  Restricted
descriptively to the observable stop $Z^{14}=1$, rerouting succeeds on 1,080
episodes, resampling on 1,073, and their pointwise maximum on 1,087.  The
realized-maximum gap over the better fixed action is therefore
$7/1{,}240=0.564516$ points, compared with an analytic exchangeable mean of
$21/2{,}480=0.846774$ points.

The tested controllers provide the relevant negative policy comparison.
Runtime structure gains $0.000$ points $[-0.987,+1.053]$ over cross-fitted
fixed rerouting, and adding the signed token-hash sketch gains $-0.197$ points
$[-1.184,+0.789]$.  These results establish only that neither tested
outcome-blind selector improves on the fixed action.  They do not prove equal
action values, absence of conditional heterogeneity for every legal signal,
or an observability bottleneck.

Finally, the audit targets an all-episode mixture rather than the observable
Gate-1 stopped population and is trained within an already opened TEST bank.
It is therefore not an instance of Eq.~\eqref{eq:four-values} for the MBPP+
stopping event and supplies no Gate-3 evidence for that chain.  None of these
opened-bank results can bypass the prospective BigCodeBench Gate-2 stop.

\subsection{Post-opening cost sensitivity}
\label{sec:cost-boundary}

Because the experiment did not observe a production price vector, cost
sensitivity is reported only through symbolic break-even conditions.  On the
eligible sealed-TEST queries, fixed rerouting is descriptively better than
fixed resampling on the three recorded coordinates---evaluator utility,
generated-output-token count, and amortized second-action generation
latency---but unmeasured action overhead remains symbolic.  This is not a
production Pareto claim and does not amend the FIT-frozen primary action.
Exact rational arithmetic, measurement limitations, the cost-boundary figure,
and the SHA-bound receipt appear in the Supplementary Material.

\subsection{Historical replay as subtractive audit evidence}

The historical replay explains why broad allocator, static-routing, and
truth-guided-verifier claims were removed; its detailed results and
observable-verifier figure are reported only in the Supplementary Material.

\section{Discussion}\label{sec:discussion}

\textbf{One thesis, three non-composable empirical applications.}
The paper's thesis is that recovery, action-ranking support, and held-out
policy value must be established in that order on one common stopped
population before a dynamic selector is licensed.  The datasets illustrate
different boundaries of this rule but do not jointly pass its gates.  MBPP+
establishes recovery under its Base-test stop; LiveCodeBench retrospectively
shows a zero-support FIT bank under a different public-test stop; and
BigCodeBench prospectively fails its own Gate-2 support threshold under a third
stop definition.  No evidence is multiplied across these populations, and the
all-episode MBPP+ audit cannot supply the missing Gate 3 for any of them.

\textbf{Gate 1 licenses recovery, not policy.}
The primary MBPP+ estimate is recoverable verifier-induced stopping debt under
one fixed sequence, verifier, model pair, and finite bank.  Its concentration
in 28 of 152 queries justifies query-cluster inference but does not select a
policy class, because those clusters are identified with hidden correctness.
The post-outcome Llama and Nemotron arms support mechanism persistence only;
they do not widen the legal action set for controller claims.

\textbf{Gate 2 separates action abundance from identification.}
LiveCodeBench contains stopping debt in every split, but FIT has no exclusive
rescue and therefore no nonzero target for L1--L4.  BigCodeBench contains much
more opportunity: 221 incorrect stops and 54 rescues by at least one action.
Yet nine shared rescues do not rank the actions, leaving only 23 reroute-only
and 22 resample-only episodes across 36 sign-bearing queries.  This is
\emph{realized recovery without prespecified identification support}: rescue
relative to STOP does not guarantee two-sided support for choosing between actions.  Missing
the 25-episode/20-query minima blocks all later claims even though the margin
is narrow.

\textbf{FIT attribution belongs to the design.}
Equation~\eqref{eq:fit-induced-distribution} makes the stopped FIT bank an
output of the qualified benchmark population, deterministic role split,
public-stop selection, Qwen14 history generator, Qwen14/Qwen7 actions,
episode budget, and official evaluator contract.  Candidate-011 observes only
their joint realization.  It cannot causally apportion the shortfall among
BigCodeBench, the stop rule, the action pair, the budget, or evaluator
randomness.  The licensed statement is therefore that \emph{our frozen design
did not supply the preregistered two-sided support}; saying merely that ``the
FIT sample was insufficient'' wrongly makes the data appear to be an
independent actor.  A larger FIT allocation might cross the count threshold,
but establishing that requires a new preregistered generation, not an
outcome-conditioned extension.

\textbf{No prospective population reaches Gate 3.}
On the opened MBPP+ bank, the realized-maximum statistic contains no
observable-history signal and lies inside its exact-fold exchangeable
reference distribution.  Neither declared outcome-blind selector improves on
the fixed action, and the audit does not target a stopped-history population.
BigCodeBench never reaches a held-out controller comparison.  This study
therefore supplies no completed Gate-3 evidence on any prospective
stopped-history population.

\textbf{Why the model pool is closed.}
The practical procedure is: measure recovery; compare fixed actions; validate
two-sided FIT support; and only then fit a selector and demand held-out quality
or efficiency.  Failure selects the simpler preceding policy.  Adding a 32B
model, another family, or static prompt routing after a failed gate changes the
stopping event, action set, reference policy, and estimand; it is a new study,
not a repair.  This is why the model pool is closed.  Exact provenance and
negative branches make the study useful: they let the system return ``use a
fixed action'' or ``policy learning is unidentified'' without manufacturing a
dynamic-routing result.

\section{Limitations}\label{sec:limitations}

The three gates reduce overclaiming but do not create population identification
from finite banks.  Gate 1 conditions on one MBPP+ hash split, pinned model
snapshots, ten draws, and a one-sided Base/Plus verifier.  It measures recovery
availability, not end-to-end accuracy, production latency, or the net value of
always buying a second call.  The Llama/Nemotron arms share the TEST queries
and first-stage bank and were added post-outcome.

The matched-information analysis is query-cross-fitted within an opened TEST
population, not a once-trained controller evaluated on a second untouched
sample.  It targets an all-episode mixture rather than a runtime-observable
stopping event.  Its realized-maximum statistic is signal-free, and the
post-audit exchangeable reference calibrates marginal label symmetry rather
than conditional heterogeneity.  The controller comparisons are limited to
the declared ridge learner and legal feature regimes.  LiveCodeBench uses a
true FIT/DEV/sealed-TEST split,
but its all-zero FIT target makes L1--L4 identical; their equal TEST values do
not rank feature tiers or estimate the Bayes value of their information.

BigCodeBench candidate-011 is intentionally FIT-only after Gate 2 fails.  It
cannot estimate DEV selection, TEST generalization, or feature sufficiency.
The 25-episode/20-query signed minima are our preregistered evidence standard,
not a natural constant, and the one-query shortfalls do not establish absence
of population heterogeneity.  Four of 895 jobs show five diagnostic partition
inconsistencies under separate evaluator realizations; excluding them still
fails the gate, but one official realization per arm retains outcome noise.
Retry-until-consistent scoring would instead create selection bias.

Finally, Eq.~\eqref{eq:fit-induced-distribution} identifies only the joint
frozen design.  Separating task-distribution, stop-selection, model-pair,
episode-budget, and evaluator contributions requires new preregistered
orthogonal experiments.  Generated-token comparisons are exact for the frozen
draws, whereas separately timed action banks are descriptive and omit
prefill, verifier, switching, queueing, residency, dollar, and energy costs.
Historical 11-model replay and its incomplete generation provenance remain
Supplementary audit context rather than evidence for any headline gate.

\section{Conclusion}\label{sec:conclusion}

This study turns ``resample or reroute?'' into an ordered identification
problem and records a negative answer for the available evidence.  A dynamic
selector is scientifically licensed only when one common stopped-history
population exhibits recoverable debt, FIT supplies two-sided action-ranking
support, and an outcome-blind policy beats the frozen fixed action on held-out
quality or efficiency.

Fresh MBPP+ evidence establishes a bounded recovery result: after a
Qwen2.5-14B Base-only false-positive stop, a uniformly selected Qwen2.5-7B draw
contributes $+2.592$ event-weighted percentage points
$[+1.618,+3.664]$.  On the separate non-stopped event of a 7B Base-test
rejection, the prespecified secondary fixed-action comparison favors rerouting
to 14B over leave-one-out resampling of 7B by $+2.882$ points
$[+0.931,+5.201]$.  The latter is a marginal fixed-action ranking, not a
selector.

The all-episode MBPP+ stress audit does not complete the stopped-history chain.
Its realized-maximum statistic contains no observable-history signal, and 75
of its 96 discordances arise on verifier rejections where no stop occurred.
Fixed rerouting is descriptively better than fixed resampling, but the tested
outcome-blind controllers gain only $0.000$ and $-0.197$ points.  The licensed
result is therefore that no tested selector identified when to deviate from
the fixed action---not that the actions are indistinguishable.

LiveCodeBench and BigCodeBench are separate applications rather than later
links in the MBPP+ chain.  BigCodeBench is the prospective population in which
recoverable debt and a failed support gate coexist: its two action signs reach
only 23/19 and 22/19 episodes/queries against minima of 25/20.  Consequently
L1--L4 are not fitted and 211 DEV and 405 TEST outcomes remain unopened.

Across these separate populations, realized recovery does not guarantee
action-ranking support, and a pointwise realized maximum does not establish
heterogeneous expected action value.  The frozen designs show recoverable
stopping debt but do not identify when an outcome-blind policy should resample
rather than reroute.  Until all three gates pass on one common stopped
population, the evidence-supported output remains a fixed action or a stopped
experiment.

\section*{Artifact Statement}
The accompanying analysis package contains the analysis source, schemas, frozen
machine-readable result artifacts, figure builders, and SHA-256 bindings
needed to reproduce the reported paper-facing tables and figures from the
locally archived finite banks.  Three provenance-bearing items remain
remote-only and are not redistributed in this package: the LiveCodeBench
FIT/DEV feature bodies, the BigCodeBench qualification authority, and the
independent one-worker report.  Their recorded receipts and digests are
preserved where available, but their bytes cannot be independently
reconstructed from the local package alone.  BigCodeBench DEV and TEST
outcomes were never opened and are not represented as available results.

\section*{Disclosures}
During the preparation of this work, the author used OpenAI Codex and
Anthropic Claude to assist with code implementation, adversarial checks,
LaTeX drafting, and language editing.  The author independently determined the
research questions, protocols, estimands, evidence boundaries, and final
claims; reviewed and edited all assisted output; and takes full responsibility
for the code, data, figures, analysis, and contents of the article.  All
reported numerical results were produced by frozen analysis and rescoring code
in recorded runtimes or containers over SHA-bound artifacts rather than by
generative text output.

The author is the founder of Krixvon; no Krixvon product, data, funding, or
other resource was used in this research.

This research received no specific grant from funding agencies in the public,
commercial, or not-for-profit sectors.

\bibliographystyle{IEEEtran}
\bibliography{refs}

@misc{routinggapreal,
  author = {Chen, Teng-Ruei},
  title  = {How Much of the Routing Gap Is Real? Decomposing the Router-to-Oracle Gap into Reproducible Specialist Advantage and Single-Draw Label Noise},
  year   = {2026},
  eprint = {2607.03436},
  archivePrefix = {arXiv},
  primaryClass = {cs.LG},
  doi    = {10.48550/arXiv.2607.03436},
  url    = {https://arxiv.org/abs/2607.03436},
  note   = {arXiv:2607.03436, doi:10.48550/arXiv.2607.03436, version 2, revised 7 July 2026}
}

@misc{rorpreprint2026,
  author = {Chen, Teng-Ruei},
  title  = {Resample or Reroute? Budget-Aware Test-Time Model Selection for Large Language Models},
  year   = {2026},
  eprint = {2607.08665v2},
  archivePrefix = {arXiv},
  primaryClass = {cs.LG},
  url    = {https://arxiv.org/abs/2607.08665v2},
  note   = {arXiv:2607.08665v2, version 2, revised 10 July 2026. Earlier replay-based version of the present arXiv record, cited for provenance; its policy and Pareto claims are not imported into the present methodological rebuild}
}

@article{frugalgpt,
  author  = {Chen, Lingjiao and Zaharia, Matei and Zou, James},
  title   = {{FrugalGPT}: How to Use Large Language Models While Reducing Cost and Improving Performance},
  journal = {Transactions on Machine Learning Research},
  year    = {2024},
  url     = {https://openreview.net/forum?id=cSimKw5p6R}
}

@inproceedings{routellm,
  author    = {Ong, Isaac and Almahairi, Amjad and Wu, Vincent and Chiang, Wei-Lin and Wu, Tianhao and Gonzalez, Joseph E. and Kadous, M. Waleed and Stoica, Ion},
  title     = {{RouteLLM}: Learning to Route {LLMs} from Preference Data},
  booktitle = {Proceedings of the International Conference on Learning Representations},
  year      = {2025},
  url       = {https://openreview.net/forum?id=8sSqNntaMr}
}

@inproceedings{modelswitch2026,
  author    = {Chen, Jianhao and Xun, Zishuo and Zhou, Bocheng and Qi, Han and Zhang, Hangfan and Zhang, Qiaosheng and Chen, Yang and Hu, Wei and Qu, Yuzhong and Hu, Shuyue},
  title     = {Do We Truly Need So Many Samples? Multi-{LLM} Repeated Sampling Efficiently Scales Test-Time Compute},
  booktitle = {Proceedings of the AAAI Conference on Artificial Intelligence},
  volume    = {40},
  pages     = {20083--20091},
  year      = {2026},
  doi       = {10.1609/aaai.v40i24.39094},
  url       = {https://ojs.aaai.org/index.php/AAAI/article/view/39094}
}

@inproceedings{bestroute2025,
  author    = {Ding, Dujian and Mallick, Ankur and Zhang, Shaokun and Wang, Chi and Madrigal, Daniel and Hipolito Garcia, Mirian Del Carmen and Xia, Menglin and Lakshmanan, Laks V. S. and Wu, Qingyun and R{\"u}hle, Victor},
  title     = {{BEST}-Route: Adaptive {LLM} Routing with Test-Time Optimal Compute},
  booktitle = {Proceedings of the 42nd International Conference on Machine Learning},
  series    = {Proceedings of Machine Learning Research},
  volume    = {267},
  pages     = {13870--13884},
  year      = {2025},
  url       = {https://proceedings.mlr.press/v267/ding25d.html}
}

@inproceedings{cascaderouting2025,
  author    = {Dekoninck, Jasper and Baader, Maximilian and Vechev, Martin},
  title     = {A Unified Approach to Routing and Cascading for {LLM}s},
  booktitle = {Proceedings of the 42nd International Conference on Machine Learning},
  series    = {Proceedings of Machine Learning Research},
  volume    = {267},
  pages     = {12987--13010},
  year      = {2025},
  url       = {https://proceedings.mlr.press/v267/dekoninck25a.html}
}

@inproceedings{automix2024,
  author    = {Aggarwal, Pranjal and Madaan, Aman and Anand, Ankit and Potharaju, Srividya Pranavi and Mishra, Swaroop and Zhou, Pei and Gupta, Aditya and Rajagopal, Dheeraj and Kappaganthu, Karthik and Yang, Yiming and Upadhyay, Shyam and Faruqui, Manaal and Mausam},
  title     = {{AutoMix}: Automatically Mixing Language Models},
  booktitle = {Advances in Neural Information Processing Systems},
  volume    = {37},
  pages     = {131000--131034},
  year      = {2024},
  doi       = {10.52202/079017-4164},
  url       = {https://proceedings.neurips.cc/paper_files/paper/2024/hash/ecda225cb187b40ea8edc1f46b03ffda-Abstract-Conference.html}
}

@inproceedings{stroebl2026limits,
  author    = {Stroebl, Benedikt and Kapoor, Sayash and Narayanan, Arvind},
  title     = {The Limits of Inference Scaling Through Resampling},
  booktitle = {Proceedings of the International Conference on Learning Representations},
  year      = {2026},
  note      = {arXiv:2411.17501},
  url       = {https://iclr.cc/virtual/2026/poster/10007907}
}

@inproceedings{codecascade2024,
  author    = {Chen, Boyuan and Zhu, Mingzhi and Dolan-Gavitt, Brendan and Shafique, Muhammad and Garg, Siddharth},
  title     = {Model Cascading for Code: A Cascaded Black-Box Multi-Model Framework for Cost-Efficient Code Completion with Self-Testing},
  booktitle = {2025 International Joint Conference on Neural Networks (IJCNN)},
  pages     = {1--9},
  year      = {2025},
  doi       = {10.1109/IJCNN64981.2025.11227916},
  url       = {https://doi.org/10.1109/IJCNN64981.2025.11227916}
}

@inproceedings{adaptiveconsistency2023,
  author    = {Aggarwal, Pranjal and Madaan, Aman and Yang, Yiming and Mausam},
  title     = {Let's Sample Step by Step: Adaptive-Consistency for Efficient Reasoning and Coding with {LLM}s},
  booktitle = {Proceedings of the 2023 Conference on Empirical Methods in Natural Language Processing},
  pages     = {12375--12396},
  year      = {2023},
  doi       = {10.18653/v1/2023.emnlp-main.761},
  url       = {https://aclanthology.org/2023.emnlp-main.761/}
}

@article{agreementcascade2025,
  author  = {Kolawole, Steven and Dennis, Don and Talwalkar, Ameet and Smith, Virginia},
  title   = {Agreement-Based Cascading for Efficient Inference},
  journal = {Transactions on Machine Learning Research},
  year    = {2025},
  note    = {arXiv:2407.02348},
  url     = {https://openreview.net/forum?id=jn9B7LMlzk}
}

@inproceedings{semanticcascade2025,
  author    = {Soiffer, Duncan and Kolawole, Steven and Smith, Virginia},
  title     = {Semantic Agreement Enables Efficient Open-Ended {LLM} Cascades},
  booktitle = {Proceedings of the 2025 Conference on Empirical Methods in Natural Language Processing: Industry Track},
  pages     = {2499--2537},
  year      = {2025},
  doi       = {10.18653/v1/2025.emnlp-industry.171},
  url       = {https://aclanthology.org/2025.emnlp-industry.171/}
}

@inproceedings{routerr1,
  author    = {Zhang, Haozhen and Feng, Tao and You, Jiaxuan},
  title     = {{Router-R1}: Teaching {LLM}s Multi-Round Routing and Aggregation via Reinforcement Learning},
  booktitle = {Advances in Neural Information Processing Systems},
  pages     = {156646--156678},
  year      = {2025},
  doi       = {10.52202/085713-4719},
  note      = {arXiv:2506.09033},
  url       = {https://proceedings.nips.cc/paper_files/paper/2025/hash/ceaa137fce916aba5c65fceb1309088b-Abstract-Conference.html}
}

@article{rlcascaderouter2026,
  author  = {Huang, Shihong and Wang, Shengjie and Ma, Hong and Xu, Zhou},
  title   = {{RLCascadeRouter}: Quality-Estimator-Free Cascade Routing via Reinforcement Learning},
  journal = {arXiv preprint arXiv:2608.15817},
  year    = {2026},
  url     = {https://arxiv.org/abs/2608.15817}
}

@inproceedings{codet2023,
  author    = {Chen, Bei and Zhang, Fengji and Nguyen, Anh and Zan, Daoguang and Lin, Zeqi and Lou, Jian-Guang and Chen, Weizhu},
  title     = {{CodeT}: Code Generation with Generated Tests},
  booktitle = {Proceedings of the International Conference on Learning Representations},
  year      = {2023},
  note      = {arXiv:2207.10397},
  url       = {https://openreview.net/forum?id=ktrw68Cmu9c}
}

@inproceedings{lever2023,
  author    = {Ni, Ansong and Iyer, Srini and Radev, Dragomir and Stoyanov, Veselin and Yih, Wen-Tau and Wang, Sida and Lin, Xi Victoria},
  title     = {{LEVER}: Learning to Verify Language-to-Code Generation with Execution},
  booktitle = {Proceedings of the 40th International Conference on Machine Learning},
  series    = {Proceedings of Machine Learning Research},
  volume    = {202},
  pages     = {26106--26128},
  year      = {2023},
  url       = {https://proceedings.mlr.press/v202/ni23b.html}
}

@inproceedings{selfconsistency,
  author    = {Wang, Xuezhi and Wei, Jason and Schuurmans, Dale and Le, Quoc and Chi, Ed and Narang, Sharan and Chowdhery, Aakanksha and Zhou, Denny},
  title     = {Self-Consistency Improves Chain of Thought Reasoning in Language Models},
  booktitle = {Proceedings of the International Conference on Learning Representations},
  year      = {2023},
  note      = {arXiv:2203.11171},
  url       = {https://openreview.net/forum?id=1PL1NIMMrw}
}

@inproceedings{scalingtesttime,
  author    = {Snell, Charlie and Lee, Jaehoon and Xu, Kelvin and Kumar, Aviral},
  title     = {Scaling {LLM} Test-Time Compute Optimally Can be More Effective than Scaling Parameters for Reasoning},
  booktitle = {Proceedings of the International Conference on Learning Representations},
  year      = {2025},
  note      = {arXiv:2408.03314},
  url       = {https://proceedings.iclr.cc/paper_files/paper/2025/hash/1b623663fd9b874366f3ce019fdfdd44-Abstract-Conference.html}
}

@article{routerbench,
  author  = {Hu, Qitian Jason and Bieker, Jacob and Li, Xiuyu and Jiang, Nan and Keigwin, Benjamin and Ranganath, Gaurav and Keutzer, Kurt and Upadhyay, Shriyash Kaustubh},
  title   = {{RouterBench}: A Benchmark for Multi-{LLM} Routing System},
  journal = {arXiv preprint arXiv:2403.12031},
  year    = {2024},
  note    = {; ICML 2024 Agentic Markets Workshop (non-archival)}
}

@inproceedings{llmrouterbench,
  author    = {Li, Hao and Zhang, Yiqun and Guo, Zhaoyan and Wang, Chenxu and Tang, Shengji and Zhang, Qiaosheng and Chen, Yang and Qi, Biqing and Ye, Peng and Bai, Lei and Wang, Zhen and Hu, Shuyue},
  title     = {{LLMRouterBench}: A Massive Benchmark and Unified Framework for {LLM} Routing},
  booktitle = {Findings of the Association for Computational Linguistics: ACL 2026},
  pages     = {37733--37754},
  year      = {2026},
  doi       = {10.18653/v1/2026.findings-acl.1881},
  note      = {arXiv:2601.07206},
  url       = {https://aclanthology.org/2026.findings-acl.1881/}
}

@inproceedings{routereval,
  author    = {Huang, Zhongzhan and Ling, Guoming and Lin, Yupei and Chen, Yandong and Zhong, Shanshan and Wu, Hefeng and Lin, Liang},
  title     = {{RouterEval}: A Comprehensive Benchmark for Routing {LLM}s to Explore Model-level Scaling Up in {LLM}s},
  booktitle = {Findings of the Association for Computational Linguistics: EMNLP 2025},
  pages     = {3860--3887},
  publisher = {Association for Computational Linguistics},
  year      = {2025},
  doi       = {10.18653/v1/2025.findings-emnlp.208},
  url       = {https://aclanthology.org/2025.findings-emnlp.208/}
}

@inproceedings{confidencetokens2025,
  author    = {Chuang, Yu-Neng and Sarma, Prathusha Kameswara and Gopalan, Parikshit and Boccio, John and Bolouki, Sara and Hu, Xia and Zhou, Helen},
  title     = {Learning to Route {LLM}s with Confidence Tokens},
  booktitle = {Proceedings of the 42nd International Conference on Machine Learning},
  series    = {Proceedings of Machine Learning Research},
  volume    = {267},
  pages     = {10859--10878},
  year      = {2025},
  url       = {https://proceedings.mlr.press/v267/chuang25b.html}
}

@inproceedings{lookaheadrouting2025,
  author    = {Huang, Canbin and Shi, Tianyuan and Zhu, Yuhua and Chen, Ruijun and Quan, Xiaojun},
  title     = {Lookahead Routing for Large Language Models},
  booktitle = {Advances in Neural Information Processing Systems},
  volume    = {38},
  year      = {2025},
  doi       = {10.52202/085713-1976},
  url       = {https://papers.neurips.cc/paper_files/paper/2025/hash/552456ddb6f4b2956b2933ab83f56df0-Abstract-Conference.html}
}

@inproceedings{selfcertainty2025,
  author    = {Kang, Zhewei and Zhao, Xuandong and Song, Dawn},
  title     = {Scalable Best-of-$N$ Selection for Large Language Models via Self-Certainty},
  booktitle = {Advances in Neural Information Processing Systems},
  volume    = {38},
  year      = {2025},
  doi       = {10.52202/085713-0663},
  url       = {https://proceedings.nips.cc/paper_files/paper/2025/hash/1c7eff166a8e345f664f0faa8f4e4d2e-Abstract-Conference.html}
}

@inproceedings{selfconsistenterrors2025,
  author    = {Tan, Hexiang and Sun, Fei and Liu, Sha and Su, Du and Cao, Qi and Chen, Xin and Wang, Jingang and Cai, Xunliang and Wang, Yuanzhuo and Shen, Huawei and Cheng, Xueqi},
  title     = {Too Consistent to Detect: A Study of Self-Consistent Errors in {LLM}s},
  booktitle = {Proceedings of the 2025 Conference on Empirical Methods in Natural Language Processing},
  pages     = {4755--4765},
  year      = {2025},
  doi       = {10.18653/v1/2025.emnlp-main.238},
  url       = {https://aclanthology.org/2025.emnlp-main.238/}
}

@inproceedings{activationfailures2026,
  author    = {Lugoloobi, William and Foster, Thomas and Bankes, William and Russell, Chris},
  title     = {{LLM}s Encode Their Failures: Predicting Success from Pre-Generation Activations},
  booktitle = {Proceedings of the Third Conference on Language Modeling},
  year      = {2026},
  doi       = {10.48550/arXiv.2602.09924},
  url       = {https://arxiv.org/abs/2602.09924},
  note      = {Accepted to COLM 2026; prepublication arXiv record}
}

@inproceedings{causalrouting2025,
  author    = {Tsiourvas, Asterios and Sun, Wei and Perakis, Georgia},
  title     = {Causal {LLM} Routing: End-to-End Regret Minimization from Observational Data},
  booktitle = {Advances in Neural Information Processing Systems},
  volume    = {38},
  year      = {2025},
  doi       = {10.52202/085713-1252},
  url       = {https://proceedings.neurips.cc/paper_files/paper/2025/hash/357774d53e5ee21c5f08ba779e3b5dd9-Abstract-Conference.html}
}

@inproceedings{strategictesttime2026,
  author    = {Zuo, Bowen and Zhu, Yinglun},
  title     = {Strategic Scaling of Test-Time Compute: A Bandit Learning Approach},
  booktitle = {Proceedings of the International Conference on Learning Representations},
  pages     = {134105--134128},
  year      = {2026},
  url       = {https://proceedings.iclr.cc/paper_files/paper/2026/hash/d8d2c5f79c53a2a685dbdf93f78c5695-Abstract-Conference.html}
}

@inproceedings{audibert2010,
  author    = {Audibert, Jean-Yves and Bubeck, S{\'e}bastien and Munos, R{\'e}mi},
  title     = {Best Arm Identification in Multi-Armed Bandits},
  booktitle = {Proc. Conf. on Learning Theory (COLT)},
  pages     = {41--53},
  year      = {2010}
}

@article{kaufmann2016,
  author  = {Kaufmann, Emilie and Capp{\'e}, Olivier and Garivier, Aur{\'e}lien},
  title   = {On the Complexity of Best-Arm Identification in Multi-Armed Bandit Models},
  journal = {Journal of Machine Learning Research},
  volume  = {17},
  number  = {1},
  pages   = {1--42},
  year    = {2016}
}

@inproceedings{lightman2024verify,
  author    = {Lightman, Hunter and Kosaraju, Vineet and Burda, Yura and Edwards, Harrison and Baker, Bowen and Lee, Teddy and Leike, Jan and Schulman, John and Sutskever, Ilya and Cobbe, Karl},
  title     = {Let's Verify Step by Step},
  booktitle = {Proceedings of the International Conference on Learning Representations},
  volume    = {2024},
  pages     = {39578--39601},
  year      = {2024},
  url       = {https://proceedings.iclr.cc/paper_files/paper/2024/hash/aca97732e30bcf1303bc22ac3924fd16-Abstract-Conference.html}
}

@article{gsm8k,
  author  = {Cobbe, Karl and Kosaraju, Vineet and Bavarian, Mohammad and Chen, Mark and Jun, Heewoo and Kaiser, Lukasz and Plappert, Matthias and Tworek, Jerry and Hilton, Jacob and Nakano, Reiichiro and Hesse, Christopher and Schulman, John},
  title   = {Training Verifiers to Solve Math Word Problems},
  journal = {arXiv preprint arXiv:2110.14168},
  year    = {2021},
  url     = {https://arxiv.org/abs/2110.14168}
}

@inproceedings{mathdataset,
  author    = {Hendrycks, Dan and Burns, Collin and Kadavath, Saurav and Arora, Akul and Basart, Steven and Tang, Eric and Song, Dawn and Steinhardt, Jacob},
  title     = {Measuring Mathematical Problem Solving With the {MATH} Dataset},
  booktitle = {Proceedings of the Neural Information Processing Systems Track on Datasets and Benchmarks},
  volume    = {1},
  year      = {2021},
  url       = {https://datasets-benchmarks-proceedings.neurips.cc/paper/2021/hash/be83ab3ecd0db773eb2dc1b0a17836a1-Abstract-round2.html}
}

@inproceedings{gpqa,
  author    = {Rein, David and Hou, Betty Li and Stickland, Asa Cooper and Petty, Jackson and Pang, Richard Yuanzhe and Dirani, Julien and Michael, Julian and Bowman, Samuel R.},
  title     = {{GPQA}: A Graduate-Level Google-Proof {Q}\&{A} Benchmark},
  booktitle = {First Conference on Language Modeling},
  year      = {2024},
  url       = {https://openreview.net/forum?id=Ti67584b98}
}

@article{humaneval,
  author  = {Chen, Mark and Tworek, Jerry and Jun, Heewoo and Yuan, Qiming and Pinto, Henrique Ponde de Oliveira and Kaplan, Jared and Edwards, Harri and Burda, Yuri and Joseph, Nicholas and Brockman, Greg and others},
  title   = {Evaluating Large Language Models Trained on Code},
  journal = {arXiv preprint arXiv:2107.03374},
  year    = {2021},
  url     = {https://arxiv.org/abs/2107.03374}
}

@inproceedings{evalplus,
  author    = {Liu, Jiawei and Xia, Chunqiu Steven and Wang, Yuyao and Zhang, Lingming},
  title     = {Is Your Code Generated by {ChatGPT} Really Correct? Rigorous Evaluation of Large Language Models for Code Generation},
  booktitle = {Advances in Neural Information Processing Systems},
  volume    = {36},
  pages     = {21558--21572},
  year      = {2023},
  doi       = {10.52202/075280-0943},
  url       = {https://proceedings.neurips.cc/paper_files/paper/2023/hash/43e9d647ccd3e4b7b5baab53f0368686-Abstract.html}
}

@article{austin2021mbpp,
  author  = {Austin, Jacob and Odena, Augustus and Nye, Maxwell and Bosma, Maarten and Michalewski, Henryk and Dohan, David and Jiang, Ellen and Cai, Carrie and Terry, Michael and Le, Quoc V. and Sutton, Charles},
  title   = {Program Synthesis with Large Language Models},
  journal = {arXiv preprint arXiv:2108.07732},
  year    = {2021},
  doi     = {10.48550/arXiv.2108.07732},
  url     = {https://arxiv.org/abs/2108.07732}
}

@article{qwen25technical,
  author  = {Yang, An and Yang, Baosong and Zhang, Beichen and Hui, Binyuan and Zheng, Bo and Yu, Bowen and Li, Chengyuan and Liu, Dayiheng and Huang, Fei and Wei, Haoran and Lin, Huan and Yang, Jian and Tu, Jianhong and Zhang, Jianwei and Yang, Jianxin and Yang, Jiaxi and Zhou, Jingren and Lin, Junyang and Dang, Kai and Lu, Keming and Bao, Keqin and Yang, Kexin and Yu, Le and Li, Mei and Xue, Mingfeng and Zhang, Pei and Zhu, Qin and Men, Rui and Lin, Runji and Li, Tianhao and Tang, Tianyi and Xia, Tingyu and Ren, Xingzhang and Ren, Xuancheng and Fan, Yang and Su, Yang and Zhang, Yichang and Wan, Yu and Liu, Yuqiong and Cui, Zeyu and Zhang, Zhenru and Qiu, Zihan},
  title   = {{Qwen2.5} Technical Report},
  journal = {arXiv preprint arXiv:2412.15115},
  year    = {2024},
  doi     = {10.48550/arXiv.2412.15115},
  url     = {https://arxiv.org/abs/2412.15115}
}

@article{llama31herd,
  author  = {Grattafiori, Aaron and others},
  title   = {The {Llama 3} Herd of Models},
  journal = {arXiv preprint arXiv:2407.21783},
  year    = {2024},
  doi     = {10.48550/arXiv.2407.21783},
  url     = {https://arxiv.org/abs/2407.21783}
}

@article{nemotronnano2,
  author  = {{NVIDIA}},
  title   = {{NVIDIA Nemotron Nano 2}: An Accurate and Efficient Hybrid {Mamba-Transformer} Reasoning Model},
  journal = {arXiv preprint arXiv:2508.14444},
  year    = {2025},
  doi     = {10.48550/arXiv.2508.14444},
  url     = {https://arxiv.org/abs/2508.14444}
}

@inproceedings{lin2024awq,
  author    = {Lin, Ji and Tang, Jiaming and Tang, Haotian and Yang, Shang and Chen, Wei-Ming and Wang, Wei-Chen and Xiao, Guangxuan and Dang, Xingyu and Gan, Chuang and Han, Song},
  title     = {{AWQ}: Activation-aware Weight Quantization for On-Device {LLM} Compression and Acceleration},
  booktitle = {Proceedings of Machine Learning and Systems},
  volume    = {6},
  pages     = {87--100},
  year      = {2024},
  url       = {https://proceedings.mlsys.org/paper_files/paper/2024/hash/42a452cbafa9dd64e9ba4aa95cc1ef21-Abstract-Conference.html}
}

@inproceedings{jain2025livecodebench,
  author    = {Jain, Naman and Han, King and Gu, Alex and Li, Wen-Ding and Yan, Fanjia and Zhang, Tianjun and Wang, Sida and Solar-Lezama, Armando and Sen, Koushik and Stoica, Ion},
  title     = {{LiveCodeBench}: Holistic and Contamination Free Evaluation of Large Language Models for Code},
  booktitle = {Proceedings of the International Conference on Learning Representations},
  year      = {2025},
  url       = {https://openreview.net/forum?id=chfJJYC3iL}
}

@inproceedings{bigcodebench,
  title     = {{BigCodeBench}: Benchmarking Code Generation with Diverse Function Calls and Complex Instructions},
  author    = {Zhuo, Terry Yue and Vu, Minh Chien and Chim, Jenny and Hu, Han and Yu, Wenhao and Widyasari, Ratnadira and Yusuf, Imam Nur Bani and Zhan, Haolan and He, Junda and Paul, Indraneil and others},
  booktitle = {Proceedings of the International Conference on Learning Representations},
  year      = {2025},
  url       = {https://openreview.net/forum?id=YrycTjllL0}
}

\clearpage
% ---------------------------------------------------------------------------
% Supplementary Material (appended so that the arXiv record is one document).
% Sections, tables, figures, equations, algorithms and propositions restart
% with the prefix S.
% ---------------------------------------------------------------------------
\setcounter{section}{0}\renewcommand{\thesection}{S\arabic{section}}
\setcounter{table}{0}\renewcommand{\thetable}{S\arabic{table}}
\setcounter{figure}{0}\renewcommand{\thefigure}{S\arabic{figure}}
\setcounter{equation}{0}\renewcommand{\theequation}{S\arabic{equation}}
\setcounter{algorithm}{0}\renewcommand{\thealgorithm}{S\arabic{algorithm}}
\setcounter{proposition}{0}\renewcommand{\theproposition}{S\arabic{proposition}}
\setcounter{remark}{0}

\section*{Supplementary Material}
\noindent\textbf{Scope.} This supplement records the exact verifier-gated estimands, finite-bank
derivations, descriptive matched-information audit, exact-fold
exchangeable-actions calibration, deterministic clustered inference, negative
self-attacks, draw-position diagnostics, cross-paper data ledger, and SHA-256
artifact bindings supporting the accompanying article.  All quantities
condition on frozen response banks and declared evaluators.  A pointwise
realized maximum is treated as a clairvoyant descriptive statistic, not as
evidence of heterogeneous expected action value.  The supplement also gives a
stopped-history value decomposition and states the common-population condition
required before multiple gates can compose.  A separate LiveCodeBench FIT/DEV/sealed-TEST
ladder is reported as a support/transport diagnosis rather than a generic
negative result about outcome-blind features.  A later preregistered
BigCodeBench replication enforces a two-sided support gate before any new
feature fit or DEV/TEST opening and records an evaluator-randomness amendment
that prevents retry-conditioned labels.
Section, table, figure, equation, algorithm, and proposition numbers below carry
the prefix S; citations refer to the reference list above.

\section{Evidence roles and frozen scope}

The evidence is deliberately stratified.  The confirmatory mechanism result
uses a hash-defined TEST split of 152 MBPP+ v0.2.0 queries, ten stochastic
draws from each of Qwen2.5-14B-Instruct-AWQ and
Qwen2.5-7B-Instruct-AWQ, and an EvalPlus 0.3.1 scorer.  Base-test passage is
policy visible; Base-and-Plus correctness is evaluator only.  Llama-3.1-8B-
Instruct and NVIDIA-Nemotron-Nano-9B-v2 are post-outcome model-family
sensitivities.  The historical 11-model archive is an audit and development
record, not an independent replication of the fresh MBPP+ result.

The prespecified observability-ladder follow-up is a separate experiment on
342 LiveCodeBench \texttt{code\_generation\_lite} v5+v6 tasks: one pipeline
pilot, 136 FIT, 68 DEV, and 137 sealed TEST queries.  It holds the Qwen14 start
model and the Qwen14-resample/Qwen7-reroute action set fixed while varying only
nested stopped-history information.  The FIT/DEV controller was frozen before
one TEST outcome opening.  Its result does not replace the MBPP+ mechanism.
The opened-population matched-information audit takes a second action on every
episode and is reported separately because it does not condition on the
Gate-1 stopping event.

The BigCodeBench v0.1.4 replication is a second, independently generated
support experiment.  A pinned authority qualifies 1,100 of 1,140 tasks and
assigns one pilot, 483 FIT, 211 DEV, and 405 TEST tasks by near-duplicate
component.  Its gate fails on FIT and therefore leaves all DEV/TEST outcomes
unopened; it cannot replace or retrospectively reinterpret the LiveCodeBench
TEST result.

The one-query pilot MBPP/100 is excluded from all scientific splits.  The
reported generation-token count includes the terminal token, whereas decoded
text omits it.  Numeric draw indices remain fixed batch positions; neither the
generation protocol nor the analysis treats them as independent and
identically distributed samples.

\begin{table*}[t]
\centering
\caption{Canonical notation and information boundary.  These symbols replace
earlier overloaded uses of $H$, $Y$, $U$, $D$, and $S$.}
\label{tab:s-notation}
\small
\begin{tabularx}{\textwidth}{@{}p{0.14\textwidth}p{0.19\textwidth}X
p{0.19\textwidth}@{}}
\toprule
Symbol & Boundary & Meaning & Replaces or disambiguates \\
\midrule
$\mathcal H_{it}$ & policy-visible & Observable history/filtration through
call $t$ & $H_i$ used informally for history \\
$\omega_i$ & policy-internal & Frozen policy randomization for query $i$ &
$U_i$ as randomization \\
$Z_{ij}$ & policy-visible & Verifier result or declared stopped-history
feature vector & Observable signal only \\
$S$ & policy-visible event & Declared stopping event, such as Base-test
acceptance & Avoids conflict with a resample action label \\
$Y^0_{ij}$ & evaluator-only & Hidden correctness of the accepted or initial
response & $H_{ij}$ used as hidden correctness \\
$Y^a_{ij}$ & evaluator-only & Hidden correctness of the completion produced by
action $a$ & Action-specific correctness \\
$U_{ij}(a)$ & evaluator-only & Utility after action $a$, usually
$\max\{Y^0_{ij},Y^a_{ij}\}$ & $Y_a$ used as utility in the theory section \\
$\Delta_{ij}$ & evaluator-only &
$U_{ij}(\mathrm{rer})-U_{ij}(\mathrm{res})$ & $D_{ij}$ as action difference \\
$D$ & design constant & Number of archived draws per query/model & Draw count
only \\
$F_{id}$ & evaluator-only subset &
$\mathbf1\{Z^{14}_{id}=1,Y^{14}_{id}=0\}$; not a runtime stop &
False-positive subset of $S$ \\
$a\in\{\mathrm{res},\mathrm{rer}\}$ & declared action & Resample or reroute &
$\mathrm{res}/\mathrm{rer}$ action labels \\
\bottomrule
\end{tabularx}
\end{table*}

\section{Verifier-gated recovery estimands}

For TEST query $i$ and draw $d\in\{0,\ldots,D-1\}$, let
$Z^{14}_{id}$ denote visible 14B Base-test passage and let $Y^m_{id}$ denote
hidden Base-and-Plus correctness for model $m$.  Define an irreversible
false-positive stop by
\begin{equation}
 F_{id}=\mathbf 1\{Z^{14}_{id}=1,\;Y^{14}_{id}=0\}.
\end{equation}
The prespecified independent-pairing recovery opportunity is
\begin{equation}\label{eq:s-primary}
 \widehat\Delta_{\mathrm{VG}}
 =\frac{1}{N}\sum_{i=1}^{N}
  \left(\frac{1}{D}\sum_d F_{id}\right)
  \left(\frac{1}{D}\sum_e Y^{7}_{ie}\right).
\end{equation}
The fixed-offset sensitivity is
\begin{equation}\label{eq:s-offset}
 \widehat\Delta_{\mathrm{off}}
 =\frac{1}{ND}\sum_{i=1}^{N}\sum_{d=0}^{D-1}
  F_{id}Y^{7}_{i,(d+1)\bmod D}.
\end{equation}
The prespecified secondary event is a 7B Base-test rejection:
\begin{equation}\label{eq:s-secondary}
\begin{split}
 \widehat\Delta_{\mathrm{sec}}
 &=\frac{1}{N}\sum_i\frac{1}{D}\sum_d
 \mathbf 1\{Z^{7}_{id}=0\}\\
 &\qquad\times\left[
   \frac{1}{D}\sum_eY^{14}_{ie}
   -\frac{1}{D-1}\sum_{e\ne d}Y^{7}_{ie}
 \right].
\end{split}
\end{equation}
It compares fixed rerouting to 14B with leave-one-out fixed resampling of 7B
after an observable rejection.  It is event-weighted, prespecified as
underpowered, and is neither a stopping-debt estimand nor a selector.  All
three equations are evaluator-only finite-bank quantities and none selects a
response at deployment time; only Eqs.~\eqref{eq:s-primary} and
\eqref{eq:s-offset} measure false-positive-stop recovery mass.

\begin{table*}[t]
\centering
\caption{Fresh verifier-gated recovery results.  Intervals resample 152 query
clusters with 10,000 deterministic bootstrap replicates.}
\label{tab:s-recovery}
\scriptsize
\setlength{\tabcolsep}{3pt}
\begin{tabularx}{\textwidth}{@{}>{\raggedright\arraybackslash}X
  >{\raggedleft\arraybackslash}p{0.13\textwidth}
  >{\centering\arraybackslash}p{0.20\textwidth}
  >{\raggedright\arraybackslash}p{0.19\textwidth}@{}}
\toprule
Sequence or sensitivity & Estimate (points) & Interval & Status \\
\midrule
Qwen14 false-positive stop $\rightarrow$ Qwen7, \eqref{eq:s-primary}
  & $+2.592$ & $[+1.618,+3.664]$ (95\%) & primary PASS \\
Qwen14 $\rightarrow$ Qwen7 fixed offset, \eqref{eq:s-offset}
  & $+2.500$ & $[+1.513,+3.618]$ (95\%) & sensitivity PASS \\
After 7B rejection: reroute 14B $-$ leave-one-out resample 7B,
Eq.~\eqref{eq:s-secondary}
  & $+2.882$ & $[+0.931,+5.201]$ (95\%) & prespecified underpowered secondary \\
Qwen14 $\rightarrow$ Llama-3.1-8B
  & $+2.974$ & $[+1.553,+4.921]$ (97.5\%) & post-outcome \\
Qwen14 $\rightarrow$ Nemotron Nano 9B
  & $+3.303$ & $[+1.711,+5.461]$ (97.5\%) & post-outcome \\
\bottomrule
\end{tabularx}
\end{table*}

The primary point is the exact reduced fraction $197/7600$.  Nonzero primary
recovery occurs in 28 of 152 query clusters; the ten largest clusters account
for 60.2\% of the total.  The concentration is visible in
Figure~\ref{fig:s-recovery} and is why uncertainty is clustered by query
rather than by completion.

\begin{figure*}[t]
\centering
\includegraphics[width=0.94\textwidth]{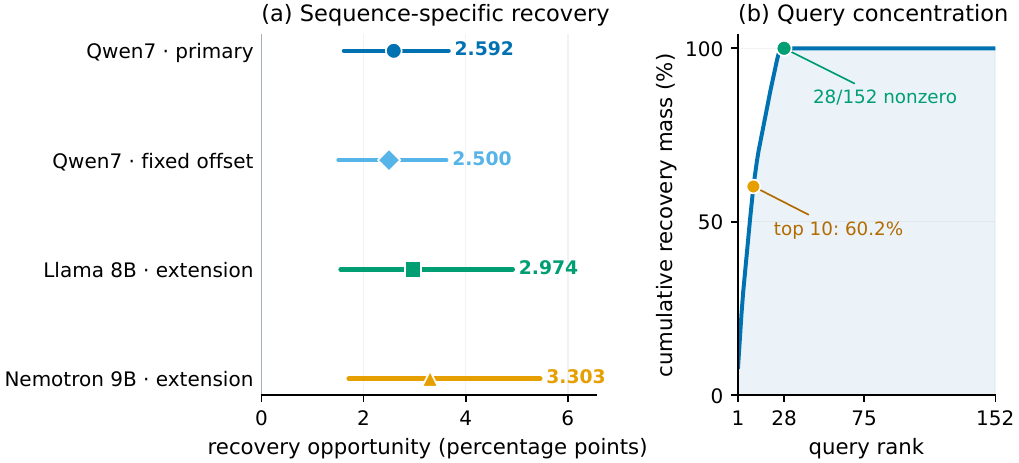}
\caption{Fresh MBPP+ recovery estimates and query-level concentration.  The
familywise-corrected extension arms support model-family robustness, not an
independent TEST replication.}
\Description{A forest plot gives five positive recovery estimates and their
uncertainty intervals.  A ranked concentration curve shows that a minority of
queries contributes most of the primary recovery mass.}
\label{fig:s-recovery}
\end{figure*}

\section{Closed-form finite-bank oracle}

For query $i$ and model $m$, let
\begin{equation}
 r_{im}=\min\{d\in\{1,\ldots,D\}:Y_{imd}=1\},
\end{equation}
with $r_{im}=+\infty$ if the fixed-order bank contains no success.  If a call
to model $m$ costs $c_m>0$, define
\begin{equation}\label{eq:s-oracle}
 C_i^*=\min_m c_m r_{im}.
\end{equation}

\begin{proposition}[Minimum-cost finite-bank opportunity]
The maximum success indicator attainable by any evaluator-informed legal
action sequence under cap $B$ is $\mathbf 1\{C_i^*\le B\}$.
\end{proposition}

\begin{proof}
Reaching the first successful draw of model $m$ requires exactly $r_{im}$ paid
calls to that model and therefore cost $c_mr_{im}$.  Interleaved calls to other
models have nonnegative cost and cannot advance model $m$'s fixed draw head.
Conversely, calling a minimizing model $r_{im}$ times realizes its first
success.  Minimizing over models proves \eqref{eq:s-oracle}.
\end{proof}

\begin{proposition}[Exchangeable realized-maximum baseline]
\label{prop:s-exchangeable-max}
If $(U_0,U_1)$ has an exchangeable joint law conditional on $S$, then
\begin{equation}\label{eq:s-exchangeable-max}
 \mathbb E[\max\{U_0,U_1\}\mid S]
 -\max_a\mathbb E[U_a\mid S]
 =\frac12\mathbb E[|U_1-U_0|\mid S].
\end{equation}
The gap is positive whenever realized outcomes disagree with positive
probability, even though the two conditional expected action values are
identical.  For independent Bernoulli$(p)$ actions it equals $p(1-p)$.
\end{proposition}

\begin{proof}
Apply $\max\{x,y\}=(x+y+|x-y|)/2$ and use exchangeability to identify the
common conditional mean.  In the Bernoulli case, the probability of a
disagreement is $2p(1-p)$.
\end{proof}

\begin{remark}[General two-action identity]
Exchangeability is not required.  Let $\Delta=U_1-U_0$.  For any two
integrable actions and any event $S$ with positive probability,
\begin{equation}\label{eq:s-general-realized-max}
\begin{split}
 &\mathbb E[\max\{U_0,U_1\}\mid S]
 -\max_{a\in\{0,1\}}\mathbb E[U_a\mid S]\\
 &\qquad=\frac12\mathbb E[|\Delta|\mid S]
 -\frac12\left|\mathbb E[\Delta\mid S]\right|.
\end{split}
\end{equation}
The exchangeable value is therefore the largest expected gap compatible with
a fixed disagreement rate.  A shortfall records a marginal action difference,
not observable conditional heterogeneity; the statistic contains no
stopped-history signal.
\end{remark}

Thus $V_{\mathrm{or}}$ is a clairvoyant realization value, not a test of
conditional-mean heterogeneity.  Empirical action and oracle maxima are also
selected statistics; an empirical difference requires an explicit reference
distribution before it can support a heterogeneity claim.

The implementation was compared exhaustively with recursive enumeration for
all $2^6$ binary outcome banks of two models and three draws.  The two-Qwen
frontier has a maximum oracle--best-fixed difference of 3.289 percentage
points at nominal cap 28.  The exploratory four-model equal-cost frontier has
a maximum difference of 8.553 points.  These gaps show opportunity, not an
observable selection rule.

\begin{figure*}[t]
\centering
\includegraphics[width=0.95\textwidth]{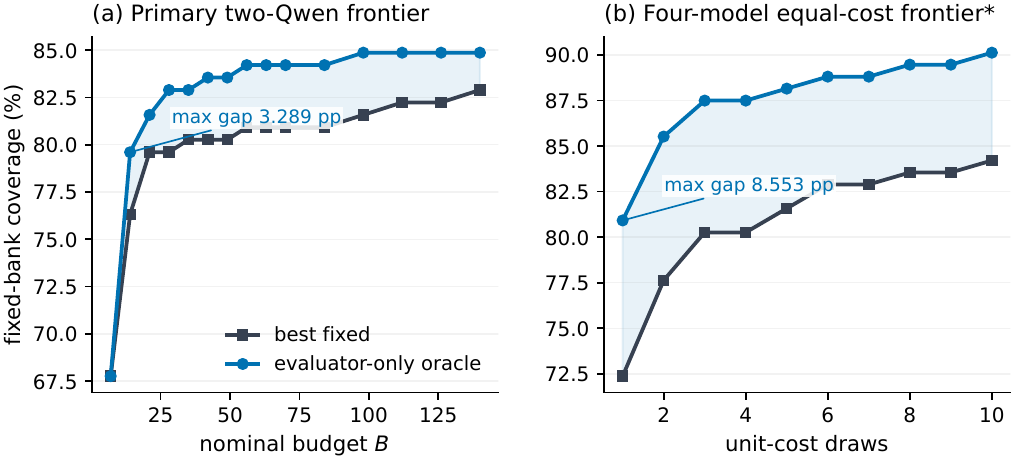}
\caption{Closed-form finite-bank opportunity frontiers.  Shading is the
evaluator-only oracle gap above the best fixed model--draw-count policy.}
\Description{Two panels compare the oracle and best fixed coverage as budget
increases for the prespecified Qwen pair and the exploratory four-model pool.}
\label{fig:s-frontier}
\end{figure*}

\section{Stopped-history observability decomposition}
\label{sec:s-stopped-history}

Let $S$ be an observable stopping event with positive probability.  Conditional
on $S$, suppose two legal second actions have evaluator utilities
$U_0,U_1\in[0,1]$, and let $Z$ be any outcome-blind signal measurable from the
stopped history.  Hidden correctness used to construct $U_0$ and $U_1$ is not
part of $Z$.  Let $a_{\mathrm{FIT}}\in\{0,1\}$ be the fixed action frozen from
FIT, write $\Delta=U_1-U_0$, and define
\begin{align}
 V_{\mathrm{frozen}}
   &=\mathbb E[U_{a_{\mathrm{FIT}}}\mid S],
   \label{eq:s-vfrozen}\\
 V_{\mathrm{fix}}
   &=\max_{a\in\{0,1\}}\mathbb E[U_a\mid S],
   \label{eq:s-vfix}\\
 V_Z
   &=\mathbb E\!\left[
       \max_{a\in\{0,1\}}\mathbb E[U_a\mid Z,S]
       \mathrel{\Big|} S\right],
   \label{eq:s-vz}\\
 V_{\mathrm{or}}
   &=\mathbb E[\max\{U_0,U_1\}\mid S].
   \label{eq:s-voracle}
\end{align}
These are, respectively, the value of the action selected and frozen on FIT,
the best fixed second action in the target population, the Bayes value
attainable from the declared stopped-history signal, and the evaluator-only
oracle.  Distinguishing the first two matters whenever the action ranking does
not transport from FIT to the target population.

\begin{proposition}[Support-aware stopped-history value decomposition]
\label{prop:s-observability-decomposition}
The four values satisfy
\begin{equation}
 V_{\mathrm{frozen}}\le V_{\mathrm{fix}}\le V_Z\le V_{\mathrm{or}},
 \label{eq:s-value-order}
\end{equation}
and the oracle--frozen gap decomposes exactly as
\begin{align}
 V_{\mathrm{or}}-V_{\mathrm{frozen}}
&=\underbrace{V_{\mathrm{fix}}-V_{\mathrm{frozen}}}_{\text{fixed-action transport loss}}\notag\\
&\quad+\underbrace{V_Z-V_{\mathrm{fix}}}_{\text{observable gain}}\notag\\
&\quad+\underbrace{V_{\mathrm{or}}-V_Z}_{\text{residual information debt}}.
 \label{eq:s-observability-decomposition}
\end{align}
If $Z_1$ is a measurable coarsening of $Z_2$, then
$V_{Z_1}\le V_{Z_2}$.  Thus a nested observability ladder is monotone at the
Bayes-value level.
\end{proposition}

\begin{proof}
Because $a_{\mathrm{FIT}}$ is one member of the legal action set,
$V_{\mathrm{frozen}}\le V_{\mathrm{fix}}$.  For the remaining inequalities,
using $\max\{x,y\}=x+(y-x)_+$ gives the identities
\begin{align}
 V_{\mathrm{fix}}
 &=\mathbb E[U_0\mid S]+\bigl(\mathbb E[\Delta\mid S]\bigr)_+,\\
 V_Z
 &=\mathbb E[U_0\mid S]\notag\\
 &\quad+\mathbb E\!\left[
      \bigl(\mathbb E[\Delta\mid Z,S]\bigr)_+\mathrel{\Big|}S\right],\\
 V_{\mathrm{or}}
 &=\mathbb E[U_0\mid S]+\mathbb E[\Delta_+\mid S].
\end{align}
Because $x\mapsto x_+$ is convex, conditional Jensen and the tower property
give
\begin{equation}
 \bigl(\mathbb E[\Delta\mid S]\bigr)_+
 \le
 \mathbb E\!\left[
   \bigl(\mathbb E[\Delta\mid Z,S]\bigr)_+\mathrel{\Big|}S\right]
 \le \mathbb E[\Delta_+\mid S].
\end{equation}
Subtracting the first identity proves \eqref{eq:s-value-order}; adding and
subtracting $V_{\mathrm{fix}}$ and $V_Z$ proves
\eqref{eq:s-observability-decomposition}.  If
$Z_1$ is a coarsening of $Z_2$, apply conditional Jensen once more to
$\mathbb E[\Delta\mid Z_1,S]
=\mathbb E[\mathbb E[\Delta\mid Z_2,S]\mid Z_1,S]$.
\end{proof}

If second-action utility preserves an already correct stop,
$U_a=\max\{Y^0,Y^a\}$, let
$V_{\mathrm{stop}}=\mathbb E[Y^0\mid S]$.  Then
$V_{\mathrm{stop}}\le V_{\mathrm{frozen}}$ and
\begin{align}
 V_{\mathrm{or}}-V_{\mathrm{stop}}
 &=(V_{\mathrm{frozen}}-V_{\mathrm{stop}})
 +(V_{\mathrm{fix}}-V_{\mathrm{frozen}})\notag\\
 &\quad+(V_Z-V_{\mathrm{fix}})+(V_{\mathrm{or}}-V_Z).
 \label{eq:s-stop-anchored-decomposition}
\end{align}
The first term is recovery under the FIT-frozen second action; the remaining
terms separate action-ranking transport, observable value, and residual
evaluator-only information.

When $V_{\mathrm{or}}>V_{\mathrm{fix}}$, the normalized recoverable share
\begin{equation}
 \rho_Z=\frac{V_Z-V_{\mathrm{fix}}}
              {V_{\mathrm{or}}-V_{\mathrm{fix}}}\in[0,1]
 \label{eq:s-rho}
\end{equation}
states how much oracle advantage above the target-population best fixed action
is recoverable from signal $Z$.  It does not include
$V_{\mathrm{fix}}-V_{\mathrm{frozen}}$, which reflects fixed-action selection
or split transport rather than stopped-history observability.
The primary $+2.592$-point mechanism estimate in the article is a related
event-weighted debt quantity for one prespecified action; it is not itself
$\rho_Z$.  Conversely, the cross-fitted ridge results estimate the values of
two particular learned policies.  They are empirical witnesses below the
Bayes envelope $V_Z$, not proof that $V_Z=V_{\mathrm{fix}}$ for every legal
outcome-blind learner.

For LiveCodeBench, the observed TEST values are 0.585714 for stop-only,
0.595238 for FIT-selected resampling, 0.611905 for the descriptive TEST-best
fixed reroute action, and 0.621429 for the realized maximum.  The associated
1.667- and 0.952-point differences are descriptive selected-sample
arithmetic, not estimators of the population terms in
Eq.~\eqref{eq:s-observability-decomposition}.  Because FIT has zero
action-difference support, the ladder does not identify $V_Z$ or $\rho_Z$.
The MBPP+ matched audit cannot instantiate this decomposition for Gate 1
because its implicit event is $S=\Omega$, rather than the observable stopping
event.  The exact stopped/non-stopped population accounting appears in
Section~\ref{sec:s-exchangeable-audit}.

The same value ordering extends to any finite action set by replacing the
two-action maximum in \eqref{eq:s-vfix}--\eqref{eq:s-voracle} with
$\max_{a\in\mathcal A}$.  If $\mathcal A_1\subseteq\mathcal A_2$, the absolute
fixed, signal-informed, and oracle values cannot decrease.  Their
\emph{differences}, however, need not increase.  For example, under two
equally likely values of an observed $Z$, let two actions have utilities
$(1,0)$ and $(0,1)$.  Then $V_Z=1$ and $V_{\mathrm{fix}}=1/2$.  Adding a third
action with constant utility $3/4$ leaves $V_Z=1$ but raises
$V_{\mathrm{fix}}$ to $3/4$, reducing the observable advantage from $1/2$ to
$1/4$.  This establishes why action abundance alone is not an observability
guarantee.

\smallskip\noindent\textbf{Model-pool closure.}
The action set is therefore a controlled part of the estimand, not a list that
can be extended without changing the question.  Replacing
$\mathcal A_1$ by $\mathcal A_2$ changes $V_{\mathrm{fix}}$, $V_Z$, and
$V_{\mathrm{or}}$ simultaneously; replacing the start model can additionally
change the conditioning event $S$.  A 32B or newly released model would thus
require new generation, cost binding, multiplicity control, and a separately
declared target.  It could reveal more evaluator-only opportunity while
leaving the recovered share $\rho_Z$ unchanged or smaller.  For the completed
observability ladder, $\mathcal A$ and the controller family were held fixed so
that the only intended intervention is the refinement
$Z_1\preceq Z_2\preceq\cdots$.  Larger or newer models remain a separate
scale- or generation-transfer study rather than an omitted tier of this one.

\subsection{Finite-sample action support before feature sufficiency}
\label{sec:s-support-gate}

The Bayes ordering above assumes that the conditional action values can in
principle be learned.  A finite FIT/DEV implementation has an earlier support
requirement.  Let $a_{\mathrm{fix}}$ be selected on FIT, let
$a_{\mathrm{alt}}$ denote the other action, and define the fitted target
\begin{equation}
 \Delta_j=U_j(a_{\mathrm{alt}})-U_j(a_{\mathrm{fix}})
 \in\{-1,0,1\}
 \label{eq:s-action-advantage-target}
\end{equation}
for FIT stopped-history episode $j$.

\begin{proposition}[Zero-support degeneracy of the frozen ladder]
\label{prop:s-zero-support}
For the ridge-plus-isotonic learner used in the LiveCodeBench ladder, suppose
$\Delta_j=0$ for every FIT episode.  For every ridge penalty in the frozen positive
grid, the calibrated action-advantage prediction is then zero for every input.
If deviation requires $\widehat\Delta(Z)>\tau$ with $\tau\ge0$, the fitted policy
never deviates on DEV or TEST, regardless of the feature sigma-field.  Hence
its held-out value equals the frozen fixed-action value, but this equality
does not imply $V_Z=V_{\mathrm{fix}}$.
\end{proposition}

\begin{proof}
Every ridge normal equation has zero right-hand side because its response
vector is zero.  The fitted intercept and coefficients are therefore zero,
both in each FIT fold and in the final all-FIT fit.  All out-of-fold
predictions are zero, and isotonic calibration against an all-zero response
assigns value zero to its only block.  The calibrated prediction is thus zero
for any feature vector.  The strict inequality $0>\tau$ is false for every
nonnegative threshold, so the action remains $a_{\mathrm{fix}}$ on every
held-out episode.  This reasoning never restricts the joint law of $(\Delta,Z)$ in
the held-out population, and therefore cannot identify its Bayes value $V_Z$.
\end{proof}

This proposition is implementation-specific and should not be confused with
an impossibility theorem for outcome-blind signals.  A future ladder should
freeze a support-positivity gate before outcome opening.  At minimum, FIT must
contain a prespecified number of episodes with $\Delta>0$; if the scientific claim
is heterogeneous action selection rather than replacement by one dominant
action, it must also contain a prespecified number with $\Delta<0$.  These are
necessary support checks, not sufficient power or calibration guarantees.
They are a post-audit design lesson and do not retrospectively alter the
original five-layer matched-information protocol.

\section{LiveCodeBench observability-ladder execution}
\label{sec:s-ladder}

Each LiveCodeBench query contributes three candidate blocks with distinct
domain-separated deterministic seeds.  A block becomes an episode only when
one of at most three sequential
Qwen14 history attempts first passes every public test.  Private/full
correctness is evaluator only.  L1 includes structural and cost fields; L2
adds token confidence; L3 adds public-verifier traces and a FIT-calibrated
score; and L4 adds agreement summaries over the observed pre-stop history.
All tiers use the same fixed action, ridge family, isotonic calibration,
nonnegative threshold grid, and cost gate.

For stopped block $b$ of query $i$, write $Y^0_{ib}$ for stop correctness and
$Y^a_{ib}$ for second-action correctness, both visible only to the evaluator.
For $a\in\mathcal A=\{\textsc{Resample},\textsc{Reroute}\}$,
\begin{align}
 U_{ib}(a)&=\max\{Y^0_{ib},Y^a_{ib}\},
 \label{eq:s-ladder-utility}\\
 \widehat V_r(\pi)&=\frac{1}{|\mathcal I_r^+|}
 \sum_{i\in\mathcal I_r^+}\frac{1}{|\mathcal B_i|}
 \sum_{b\in\mathcal B_i}U_{ib}(\pi(Z_{ib})),
 \label{eq:s-ladder-value}
\end{align}
where $\mathcal I_r^+$ contains split-$r$ queries with at least one natural
stop.  A controller makes exactly one second-action call for each such
episode.  Stop-only is a descriptive reference, not a legal controller action;
the TEST value covers 70 eligible queries rather than all 137 TEST queries.
FIT ridge, action-isotonic, and verifier-calibration fits weight stopped
episodes equally.  Fixed-action selection and DEV/TEST value, token, latency,
and inference calculations first average within query and then weight eligible
queries equally as in \eqref{eq:s-ladder-value}.

The base \texttt{VISIBLE\_ONLY} and \texttt{OPEN\_OUTCOMES} feature artifacts
leave the L3 calibrated-verifier field null.  During controller execution the
frozen FIT calibrator injects that field before action selection; no DEV/TEST
correctness is used to fit it.  After removing each OPEN row's evaluator field,
all 155 scientific stopped-history rows match the VISIBLE rows exactly.  The
artifacts differ in their \texttt{episodes}, \texttt{mode}, and
\texttt{test\_scoring\_authority} fields, but both bind the same materializer
source closure,
\texttt{8a049d9814aa516baee7b1aeebc04454}\allowbreak
\texttt{cb817418f87606b8aaeb8b34e344d3a1}.

\begin{algorithm*}[t]
\caption{Frozen LiveCodeBench observability ladder}
\label{alg:s-observability-ladder}
\begin{algorithmic}[1]
\State choose the FIT equal-query best fixed action; break an exact tie toward
\textsc{Resample}
\For{$k=1,\ldots,4$}
  \State construct cumulative $Z_k$; standardize on FIT and append explicit
  missingness indicators
  \State set $\Delta=U(\textsc{alternative})-U(\textsc{fixed})$ on FIT episodes
  \For{$\lambda\in\{10^{-4},10^{-3},\ldots,10^4\}$}
    \State obtain five query-fold out-of-fold ridge predictions, fit isotonic
    calibration, and refit ridge on all FIT episodes
    \For{$\tau\in\{0,.01,.02,\ldots,.50\}$}
      \State on DEV choose the alternative iff $\widehat\Delta_k>\tau$; require
      total-token ratio $\le1.05$ and second-action-latency ratio $\le1.05$
    \EndFor
  \EndFor
  \State freeze by highest DEV equal-query utility, fewest deviations, largest
  $\tau$, then largest $\lambda$
\EndFor
\State open TEST once; compare each frozen tier with fixed and a deterministic
deviation-count-matched random policy
\State compute 20,000 PCG64 query-cluster replicates with seed 20260830404
(random comparison seed $+1$)
\State form one-sided studentized max-$T$ simultaneous lower bounds over L1--L4
and pass a tier only when both lower bounds are $>0$ and its cost gate passes
\end{algorithmic}
\end{algorithm*}

\begin{table*}[t]
\centering
\caption{Stopped-history support across the frozen LiveCodeBench splits.  A
``rescue'' means that at least one legal second action turns an incorrect stop
into full correctness.}
\label{tab:s-ladder-support}
\small
\scriptsize
\begin{tabularx}{\textwidth}{@{}Xrrrrrrr@{}}
\toprule
Split & Queries & Stops & Incorrect & Any rescue & $\Delta>0$ & $\Delta<0$ & $\Delta=0$ \\
\midrule
FIT & 136 & 165 & 66 & 0 & 0 & 0 & 165 \\
DEV & 68 & 89 & 41 & 5 & 0 & 4 & 85 \\
sealed TEST & 137 & 155 & 49 & 5 & 3 & 2 & 150 \\
\bottomrule
\end{tabularx}
\end{table*}

TEST has 411 candidate blocks and stopped histories in 70 queries.  Of the
155 stops, 106 are fully correct.  Among the 49 incorrect stops, 44 are
rescued by neither action, two by resampling only, three by rerouting only,
and zero by both.  ``Any rescue'' is measured relative to stop-only, whereas
$\Delta$ compares reroute with the FIT-selected resample action.  Thus DEV has five
rescue episodes but no positive alternative-over-fixed target; four favor the
fixed resample action and one is tied.  Positive deviation support appears
only in TEST.  Nevertheless,
Proposition~\ref{prop:s-zero-support} forces every frozen L1--L4 controller to
make zero deviations.

\begin{figure*}[t]
\centering
\includegraphics[width=\textwidth]{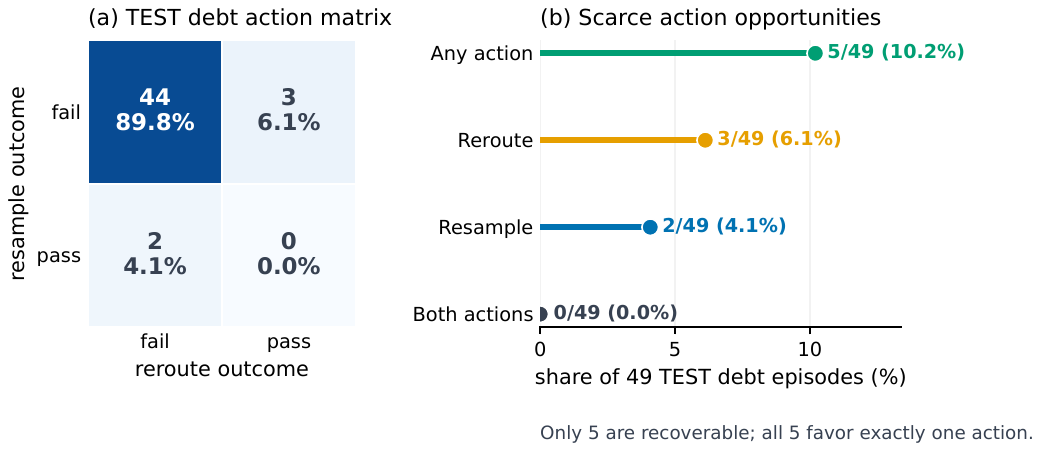}
\caption{Action outcomes within the 49 sealed TEST stopping-debt episodes.
Forty-four are rescued by neither action, three by rerouting only, two by
resampling only, and zero by both actions.  The five recoverable episodes show
finite-bank realized rescue, but their scarcity does not establish a
population rate or an observable rule for choosing the action.}
\Description{A two-by-two action-outcome matrix and a dot plot show 44
neither-action rescues, three reroute-only rescues, two resample-only rescues,
and no both-action rescues among 49 incorrect TEST stops.}
\label{fig:s-ladder-action-abundance}
\end{figure*}

Figure~\ref{fig:s-ladder-action-abundance} makes the action-specific rescue count
explicit without turning evaluator-only outcomes into controller inputs.

\begin{table*}[t]
\centering
\caption{Equal-query-weighted sealed TEST values.  The reroute row is a
descriptive sensitivity, not the frozen primary baseline.}
\label{tab:s-ladder-values}
\small
\begin{tabularx}{\textwidth}{@{}>{\raggedright\arraybackslash}Xrrr@{}}
\toprule
Rule & Value & Gain over frozen resample & Learned deviations \\
\midrule
Frozen fixed resample & 0.595238 & 0.000 points & --- \\
Fixed reroute (descriptive) & 0.611905 & $+1.667$ points & --- \\
Evaluator-only oracle & 0.621429 & $+2.619$ points & --- \\
L1, L2, L3, or L4 frozen controller & 0.595238 & 0.000 points & 0 \\
\bottomrule
\end{tabularx}
\end{table*}

\begin{figure*}[t]
\centering
\includegraphics[width=\textwidth]{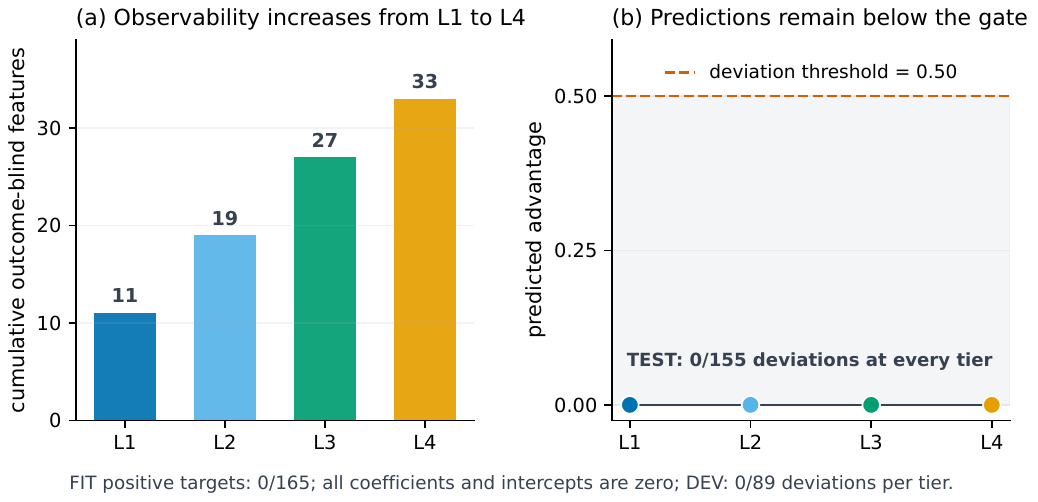}
\caption{The frozen ladder expands from 11 to 33 cumulative scientific
outcome-blind fields (22 to 66 design columns after missingness expansion),
yet every fitted action-advantage prediction is zero and
lies below the strict 0.50 deviation threshold.  This zero-deviation path is
the algebraic consequence of all 165 FIT targets being exactly zero under the frozen
ridge-plus-isotonic learner; it is not evidence that the added features are
generally uninformative.}
\Description{The left panel shows cumulative feature counts of 11, 19, 27,
and 33.  The right panel shows zero predicted advantage and zero of 155 TEST
deviations at every tier below a threshold of 0.50.}
\label{fig:s-ladder-zero-deviation}
\end{figure*}

Figure~\ref{fig:s-ladder-zero-deviation} visualizes the implementation-specific
degeneracy in Proposition~\ref{prop:s-zero-support}; it is not a comparison of
the Bayes envelopes associated with L1--L4.

The machine conclusion is
\nolinkurl{OBSERVABILITY_LADDER_NO_TESTED_SIGNAL_BEATS_FIXED}.  Its
scientific interpretation is narrower: the experiment detects finite-sample
FIT support and action-ranking transport failure and cannot compare feature
sufficiency.  The descriptive reroute-minus-resample difference is $+1.667$
points; it neither establishes reroute superiority nor licenses a post-outcome
baseline change.
The evaluator-only residual above reroute is 0.952 points.  A post hoc
calculation gives a history-plus-action generated-token ratio of 0.95073 and a
second-action-only latency ratio of 0.59001 for fixed reroute versus fixed
resample.  These descriptive ratios do not amend the utility-only FIT tie rule.
The sealed schema's historical field name
\texttt{residual\_information\_debt\_binary64\_hex} stores
$V_{\mathrm{or}}-V_{\mathrm{learned}}$; it is not an estimate of the
theoretical Bayes residual $V_{\mathrm{or}}-V_Z$ and will be renamed in a
future schema revision.

\section{BigCodeBench two-sided FIT support replication}
\label{sec:s-bcb}

\subsection{Frozen design and stopping rule}

BigCodeBench v0.1.4 is pinned by parquet revision, byte count, and SHA-256;
the official evaluator is pinned to release v0.2.4, commit
\texttt{9059fb84d1188c02edeac4995361656a2fdecbef}, and an immutable container
manifest.  Near-duplicate prompt components are assigned as units to one
pilot, FIT, DEV, or TEST, and the 483 FIT tasks are further assigned to five
folds.  For each task and each of three episodes, up to three Qwen14 history
attempts are generated.  The first attempt passing the frozen public subset
defines the stop; a stopped episode receives exactly one disjoint Qwen14
resample and one Qwen7 reroute.  The resulting banks contain 895 stopped
episodes in 351 queries and zero action draws in the other 132 FIT queries.

Let $S_{ib}$ denote the already frozen public stop event and let
$Y^{0}_{ib},Y^{\mathrm{res}}_{ib},Y^{\mathrm{rer}}_{ib}$ be one official full-suite realization
for history, resample, and reroute.  The two legal utilities are
\begin{align}
 U_{ib}(\mathrm{res})&=Y^0_{ib}\lor Y^{\mathrm{res}}_{ib},\\
 U_{ib}(\mathrm{rer})&=Y^0_{ib}\lor Y^{\mathrm{rer}}_{ib},\\
 \Delta_{ib}&=U_{ib}(\mathrm{rer})-U_{ib}(\mathrm{res})\in\{-1,0,+1\}.
 \label{eq:s-bcb-advantage}
\end{align}
Write
\begin{align}
\mathcal E&=\{(i,b):S_{ib}=1\},\qquad
\mathcal Q=\{i:\exists b\ (i,b)\in\mathcal E\},\notag\\
N_s&=\sum_{(i,b)\in\mathcal E}\mathbf 1\{\Delta_{ib}=s\},\notag\\
Q_s&=\sum_{i\in\mathcal Q}\mathbf 1\{\exists b:\Delta_{ib}=s\},\notag\\
Q_{sf}&=\sum_{i\in\mathcal Q}\mathbf 1\{g(i)=f\}\,
\mathbf 1\{\exists b:\Delta_{ib}=s\},\notag\\
&\qquad s\in\{-1,+1\},
\label{eq:s-bcb-support-counts}
\end{align}
where $g(i)\in\{0,\ldots,4\}$ is the frozen component-level fold.  For the
L3 stop-confidence calibrator, define the stop-class counts
\begin{align}
C_y^E&=\sum_{(i,b)\in\mathcal E}\mathbf 1\{Y^0_{ib}=y\},\\
C_y^Q&=\sum_{i\in\mathcal Q}
\mathbf 1\{\exists b:(i,b)\in\mathcal E,\ Y^0_{ib}=y\},
\qquad y\in\{0,1\}.
\label{eq:s-bcb-class-counts}
\end{align}
A query is sign-present if any of its stopped episodes has the sign.  With
$q=Q_s$ and $n=|\mathcal Q|$, the implementation computes the one-sided
97.5\% Clopper--Pearson lower bound
\begin{equation}
L_s=
\begin{cases}
0, & q=0,\\
\mathrm{Beta}^{-1}(0.025;q,n-q+1), & q>0,
\end{cases}
\label{eq:s-bcb-cp}
\end{equation}
by 96 deterministic bisection iterations on the exact binomial upper tail.
The two bounds have Bonferroni family confidence at least 95\%.  The complete
authorization bit is
\begin{equation}
\begin{split}
G_{\mathrm{FIT}}={}&
\mathbf 1\{|\mathcal E|\ge400,\ |\mathcal Q|\ge200\}\\
&\land\!\bigwedge_{s\in\{-1,+1\}}
\bigl(\mathbf 1\{N_s\ge25,\ Q_s\ge20\}\\
&\qquad\qquad\land\ \mathbf 1\{\min_fQ_{sf}\ge2,\ L_s>0.01\}\bigr)\\
&\land\!\bigwedge_{y\in\{0,1\}}
\mathbf 1\{C_y^E\ge25,\ C_y^Q\ge20\}.
\end{split}
\label{eq:s-bcb-gate}
\end{equation}
These are conjunctive hard gates: total sample size, prevalence, or fold
coverage cannot compensate for a failed sign count.  In particular, the
gate tests whether both selector labels are estimable before any selector is
fit; it is not a test that resampling or rerouting has zero value.

\begin{algorithm*}[t]
\caption{Outcome-blind FIT gate evaluation and authorization}
\label{alg:s-bcb-gate}
\begin{algorithmic}[1]
\Require frozen source closure, FIT generation banks, checkpoint run config,
verified authority, and preregistered minima
\State verify byte identities, read-only modes, canonical checkpoint prefix,
ordered coordinates, request continuity, and zero DEV/TEST flags
\If{any code, digest, provenance, or container invariant fails}
  \State preserve all evidence and terminate; do not reinterpret the failure
\EndIf
\For{each $(i,b)\in\mathcal E$ in frozen coordinate order}
  \State accept exactly one structured official full-suite result for each
  of history, resample, and reroute
  \State compute $U_{ib}(\mathrm{res}),U_{ib}(\mathrm{rer}),\Delta_{ib}$; record public replay and
  partition consistency only as diagnostics
  \State never retry, discard, or select an outcome from a diagnostic value
\EndFor
\State compute Eqs.~\eqref{eq:s-bcb-support-counts}--\eqref{eq:s-bcb-gate}
and freeze the gate input, receipt, provenance, and terminal hashes
\If{$G_{\mathrm{FIT}}=0$}
  \State lock L1--L4, DEV, and TEST; return the terminal support-stop verdict
\Else
  \State authorize FIT controller fitting and DEV selection only
  \State freeze the DEV-selected controller; permit one TEST opening only
  after the separate opening guard validates every binding
\EndIf
\end{algorithmic}
\end{algorithm*}

\subsection{Why partition equality requires common randomness}
\label{sec:s-bcb-randomness}

Candidate 010 originally treated the realized equality between full-suite and
separately executed public/evaluator-only outcomes as fatal.  That identity is
valid only under deterministic tests or common evaluator randomness.  Write
$Y_P(c;\omega)$ and $Y_E(c;\omega)$ for public and evaluator-only outcomes of
code $c$ under execution state $\omega$.  If the full suite is exactly their
union in the same realization, then
\begin{equation}
 Y_F(c;\omega)=Y_P(c;\omega)\wedge Y_E(c;\omega).
 \label{eq:s-bcb-common-randomness}
\end{equation}
Separate processes instead realize $\omega_P,\omega_E,\omega_F$.  Even with
the same test-source partition,
\begin{equation}
 Y_F(c;\omega_F)
 \ne Y_P(c;\omega_P)\wedge Y_E(c;\omega_E)
 \label{eq:s-bcb-independent-randomness}
\end{equation}
can occur when tests, dependencies, time, process state, or generated values
are stochastic.  Re-running a failed coordinate until
Eq.~\eqref{eq:s-bcb-common-randomness} happens would accept labels conditional
on a post-outcome event and bias the scientific sample.

The outcome-blind candidate-011 amendment therefore freezes $S_{ib}$ as the
stop event, takes exactly one $Y_F$ realization per arm as the correctness
label, and retains public replay and partition equality only as diagnostics.
No diagnostic value triggers a retry.  Code identity, task identity, source
closure, container isolation, draw digests, and checkpoint continuity remain
fail-closed.  Candidate 010 is preserved but excluded at its intact 227-job
prefix; all 895 candidate-011 jobs were rerun from zero so no
selection-conditioned checkpoint was reused.

\subsection{Terminal FIT result}

All 895 candidate-011 checkpoints are exact canonical JSON with mode 0400 and
coordinates 0--894.  Each binds the checkpoint run config, HISTORY/RESAMPLE/
REROUTE draw digests, structured worker result, outcome row, and successful
owned-container cleanup.  The final roster contains 674 correct stops and 221
incorrect stops, the latter appearing in 118 of 351 eligible queries.

\begin{table*}[t]
\centering
\caption{BigCodeBench FIT stopping debt, action abundance, and signed support.}
\label{tab:s-bcb-support}
\small
\begin{tabularx}{\textwidth}{@{}>{\raggedright\arraybackslash}Xrr@{}}
\toprule
Quantity & Episodes & Distinct queries \\
\midrule
Stopped histories & 895 & 351 \\
Incorrect stops & 221 & 118 \\
Neither action rescues & 167 & --- \\
Resample only / reroute only rescues & 22 / 23 & 19 / 19 signed support \\
Both actions rescue & 9 & --- \\
$\Delta=+1$ / $\Delta=-1$ & 23 / 22 & 19 / 19 \\
Hard minimum for each sign & 25 & 20 \\
\bottomrule
\end{tabularx}
\end{table*}

Both signs occur in all five folds: positive/negative query counts are
$(5,2),(3,4),(5,6),(3,3),(3,4)$.  Their one-sided exact lower prevalence
bounds are both 0.0329014, above 0.01.  Nevertheless $\Delta=+1$ misses the
episode/query minima by two/one and $\Delta=-1$ misses them by three/one.  The exact
verdict is
\texttt{STOP\_INSUFFICIENT\_TWO\_SIDED\_FIT\_SUPPORT}, with authority effect
\texttt{TERMINAL\_STOP\_\allowbreak DEV\_AND\_TEST\_MUST\_REMAIN\_UNOPENED}.
L1--L4 are not fitted, no DEV model selection occurs, and TEST is never opened.
The observability tiers are therefore unestimated rather than tied or null.

\subsection{Value, concentration, cost, and diagnostic robustness}

The following values are post-gate descriptions, not inputs that can revise
the support verdict.
\begin{table*}[t]
\centering
\caption{BigCodeBench FIT values and frozen action costs.  Value first averages
stopped episodes within each of 351 eligible queries.  Batch time comes from
separate action-bank runs and is descriptive.}
\label{tab:s-bcb-values-cost}
\small
\begin{tabularx}{\textwidth}{@{}>{\raggedright\arraybackslash}Xrrr@{}}
\toprule
Rule & Equal-query value & Tokens/action & Recorded seconds/action \\
\midrule
Stop only & 0.739791 & --- & --- \\
Fixed Qwen14 resample & 0.771130 & 509.705 & 9.129 \\
Fixed Qwen7 reroute & 0.776353 & 339.710 & 3.899 \\
Two-action evaluator oracle & 0.799145 & --- & --- \\
\bottomrule
\end{tabularx}
\end{table*}

Reroute exceeds resample by $11/2106=0.522$ percentage points while using
169.995 fewer generated tokens per equal-query draw.  It is also faster in the
separately recorded bank wall times.  This is a FIT-sample fixed-action
Pareto description, not authorization for static prompt routing and not a
substitute for two-sided selector support.  Episode weighting preserves the
same descriptive ordering: stop, resample, reroute, and oracle values are
0.753073, 0.787709, 0.788827, and 0.813408.

Only 36 queries carry either action-advantage sign and two carry both signs in
different episodes.  The ten largest sign-bearing query contributions contain
14 of 23 positive and 13 of 22 negative episodes.  Among eligible queries,
233 have no incorrect stop, while 50, 33, and 35 have one, two, and three
incorrect stops.  These distributions motivate query-level minima and
query-cluster inference; they do not justify counting episodes as independent.

Public-stop outcome replay is consistent in all 895 jobs.  Separately executed
suite partitions disagree in five of 2,685 completions across four jobs.  This
is a diagnostic of evaluator instability, not a retry gate or an estimated
label-error rate.  Removing those four jobs leaves 22/19 positive and 21/18
negative episodes/queries, still below the hard minima.  Thus the terminal
branch is robust to diagnostic exclusion, and both episode- and query-weighted
descriptions prefer the same fixed action.  No robustness analysis reads DEV
or TEST.

\subsection{Independent validation and reproduction bindings}

A standard-library audit independent of the protocol package rechecks the
64-file frozen source roster, 895 contiguous checkpoint names and modes,
canonical bytes, ordered coordinates, three actual bank draw digests per job,
worker-result and outcome-row digests, container audits, terminal diagnostic
counts, gate projections, and zero DEV/TEST flags.  A second validator
recomputes the exact gate receipt with the frozen candidate-011 implementation.
Both pass.  Publication-level identities are:
\begin{list}{}{\setlength{\leftmargin}{0pt}\setlength{\itemsep}{2pt}\setlength{\parsep}{0pt}\setlength{\topsep}{3pt}}
\item \textbf{Remote 64-file source roster.}
\begin{samepage}
\shabind{c7fe43b1e41b1618471d0c687ed0e487}{30469e9874c096ac44cbb4885008465f}\par
Digest of a \texttt{sha256sum}-format listing of the 64 read-only source files
retained on the execution host.  The remote tree and listing are not part of
the published package.
\end{samepage}
\item \textbf{Packaged 14-file scoring-closure roster.}
\shabind{10d1c1a4198a4ab1da6587026f029558}{1f744a8e6e524d80959f6a1e941ad8eb}
\item \textbf{Checkpoint run config}\newline\shabind{b2cc49040eb334e1b60728d5d5ddc0d}{7a96628fd6af4ab4a9d8d98eba44ec297}
\item \textbf{FIT scoring terminal}\newline\shabind{257f9d65cc9103c6bfaf30f4e058a6da}{4c65613d52c8d0870b4e3d14e2edc72c}
\item \textbf{FIT gate input}\newline\shabind{98685d71d5ca0216be225bba3c912088d}{78ee5a96d24408b810690cf93151171}
\item \textbf{FIT gate receipt}\newline\shabind{a8ead2b2d87a6c25da2b756b02a1a506}{1e72245d1a6bacfe22b80e5e9ed70f4e}
\item \textbf{FIT gate provenance}\newline\shabind{ad0a8b35bf414fcf2cf3b8619b7af356}{8ce0f9023a6c1a4ee96124ed31fd6a78}
\item \textbf{Cost projection}\newline\shabind{328aea1c2a8c585b2bc2d9516e30604d}{079a6c2c4c2411a695295ad38c979ac6}
\item \textbf{Terminal analysis}\newline\shabind{ab2c87b548ec506b58f90cb203fc1510}{bad05d889db004a3982d013d8d2030da}
\end{list}

The cost projection contains no source code or new correctness label.  It
binds every selected action draw to the terminal digest and projects only
generated tokens, bytes, and query-batch elapsed time from the frozen action
banks.  The analysis script refuses to overwrite an existing result and
recomputes every table entry above from the six canonical evidence inputs.

\section{Exact arithmetic for the descriptive cost boundary}
\label{sec:s-cost-boundary}

This section recomputes the main paper's cost sensitivity from the canonical
opened TEST feature artifact with the following SHA-256 identity:
\begin{center}
\shabind{e7ea286f4506e7efb1ba3327f5bdb716}{6901940310cc2015052c1348b93d1c5b}
\end{center}
It uses exact rational arithmetic, equal-query weighting over the 70 TEST
queries with at least one natural stop, and no randomness.  The receipt is
classified \texttt{POSTHOC\_DESCRIPTIVE\_COST\_BOUNDARY}; it is not a new
confirmatory TEST decision.

Direct recomputation gives
\begin{align}
 \bar C_H&=\frac{33259}{105},\qquad
 \bar C_A(\mathrm{res})=\frac{28337}{140},\\
 \bar C_A(\mathrm{rer})&=\frac{6189}{35},\\
 \bar L_A(\mathrm{res})&=\frac{220631209931}{70000000000}\ \mathrm{s},\\
 \bar L_A(\mathrm{rer})&=\frac{65086794021}{35000000000}\ \mathrm{s}.
\end{align}
Therefore
\begin{equation}
\begin{split}
 \bar C_T(\mathrm{res})&=\frac{218047}{420},\qquad
 \bar C_T(\mathrm{rer})=\frac{51826}{105},\\
 \frac{\bar C_T(\mathrm{rer})}{\bar C_T(\mathrm{res})}
 &=\frac{207304}{218047}=0.9507308.
\end{split}
\end{equation}
The corresponding action-only token ratio is 0.8736281.  The recorded
action-latency ratio is
\begin{equation}
 \frac{\bar L_A(\mathrm{rer})}{\bar L_A(\mathrm{res})}
 =\frac{130173588042}{220631209931}=0.5900053.
\end{equation}

Let $\lambda_T=p_T/w$, $\lambda_L=p_L/w$, and $\delta_a=q_a/w$.
Substituting the exact utility contrasts into the main paper's net-benefit
definition yields
\begin{align}
 \delta_{\mathrm{res}}+202.407143\lambda_T
 +3.151874\lambda_L &< 0.00952381,\label{eq:s-resample-stop-cost}\\
 \delta_{\mathrm{rer}}+176.828571\lambda_T
 +1.859623\lambda_L &< 0.02619048.\label{eq:s-reroute-stop-cost}
\end{align}
The first inequality characterizes RESAMPLE versus STOP; the second
characterizes REROUTE versus STOP.  Their direct comparison is
\begin{equation}
 \delta_{\mathrm{rer}}-\delta_{\mathrm{res}}
 < \frac{1}{60}+\frac{3581}{140}\lambda_T
 +\frac{90457621889}{70000000000}\lambda_L.
\end{equation}
Only a sufficiently large unmeasured rerouting-specific penalty can reverse
REROUTE's advantage on the three recorded sample coordinates.  Conversely,
because each L1--L4 policy is action-identical to fixed RESAMPLE,
$NB_{L_k}-NB_{\mathrm{res}}=-q_k\leq0$ for any nonnegative controller overhead.

These boundaries include generated output tokens only.  They exclude input
and prefill tokens, history latency, queueing, cold starts, model switching and
residency, verifier and controller execution, scorer costs, dollars, and
energy.  Oracle cost is intentionally undefined because utility ties have no
frozen physical-action or cost rule.  The receipt binds every numerator,
denominator, decimal rendering, source byte count, and source digest in
\nolinkurl{paper/figure_data/cost_break_even_receipt.json}.

\begin{figure*}[t]
\centering
\includegraphics[width=\textwidth]{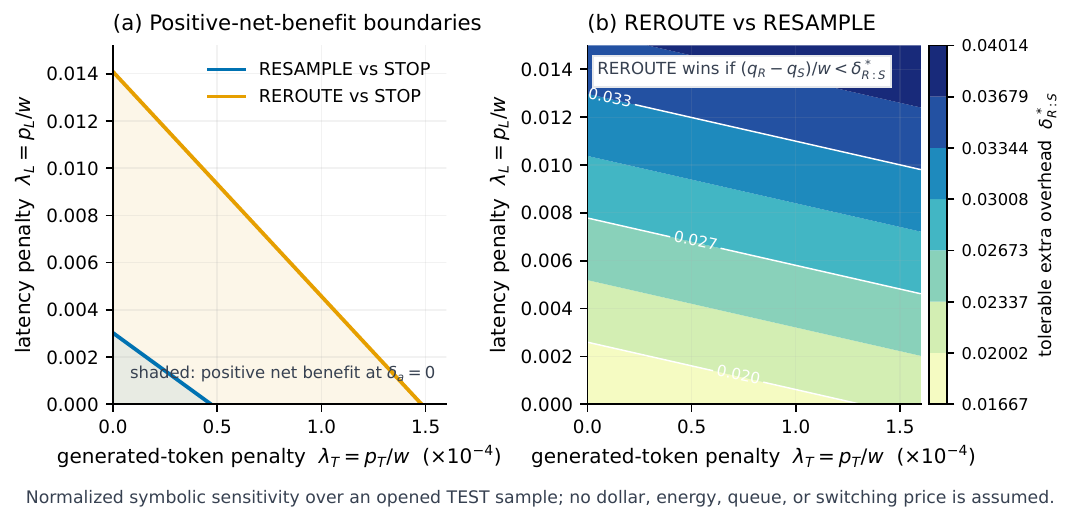}
\caption{Normalized post-opening cost sensitivity.  Panel (a) shows the
action-versus-STOP positive-net-benefit regions when unmeasured overhead is
zero.  Panel (b) shows the maximum additional rerouting-specific overhead
that preserves REROUTE's advantage over RESAMPLE.  The axes are symbolic
correctness-equivalent penalties, not observed prices.  The analysis is
conditional on the TEST queries with a natural stop and is descriptive, not a
preregistered production Pareto claim.}
\Description{Two panels show linear normalized cost boundaries.  Rerouting
has a larger positive-net-benefit region than resampling against stop-only,
and a contour surface gives the unmeasured overhead difference that would
reverse rerouting's recorded advantage over resampling.}
\label{fig:s-cost-break-even}
\end{figure*}

\section{Descriptive all-episode matched-information audit}
\label{sec:s-exchangeable-audit}

For a fixed start-model draw $(i,d)$ and legal second action $a$, write
$d^+=(d+1)\bmod D$ and
\begin{equation}
 U_{ida}=\max\{Y^s_{id},Y^a_{id^+}\}.
\end{equation}
Let $a^{\mathrm{fix}}_{-i}$ be the best fixed second action selected without
the held-out query and let $a^*_{id}=\arg\max_a U_{ida}$.  The original audit
reported
\begin{align}
 \widehat\Delta_{\mathrm{rm}}
 &=\mathbb E[U_{ida^*_{id}}-U_{ida^{\mathrm{fix}}_{-i}}],\\
 \widehat\Delta_{\mathrm{policy}}
 &=\mathbb E[U_{id\widehat\pi_{-i}(X_{id})}
              -U_{ida^{\mathrm{fix}}_{-i}}].
\end{align}
For arbitrary two-action utilities, the realized-maximum statistic satisfies
Eq.~\eqref{eq:s-general-realized-max}; it contains no stopped-history signal
and does not test conditional heterogeneity.  The exchangeable reference
calibrates the fold-specific fixed-action procedure under action-label
symmetry.

All six probes select 1,520 episodes.  The observable MBPP+ stop occurs on
1,240 episodes and the other 280 are verifier rejections.  The 198
false-positive stops are an evaluator-only subset of the observable stops.
The audit's implicit event is therefore $S=\Omega$, not the MBPP+ stopping
event.

\begin{algorithm*}[t]
\caption{Exact-fold exchangeable-actions reference}
\label{alg:s-exchangeable-null}
\begin{algorithmic}[1]
\State \textbf{Input:} two-action utility tensor $U$, recorded outer query folds
$\mathcal F$, action-cost tie break
\For{$b=1,\ldots,10^6$}
  \State independently swap the two action labels in every episode with
  probability $1/2$
  \For{each recorded outer fold}
    \State reselect the fixed action on the permuted training queries
    \State score realized maximum minus that action on the held-out queries
  \EndFor
  \State aggregate the held-out gap over all 1,520 episodes
\EndFor
\State \Return the null mean, percentile interval, and observed percentile
\end{algorithmic}
\end{algorithm*}

The exact construction contains 96 discordant episodes.  It reproduces the
reported observed statistic as $41/1{,}520=2.697368$ points.  Its analytic
exchangeable mean is $96/(2\times1{,}520)=3.157895$ points.  With seed
20260902, one million permutations give mean 3.157851 points, standard
deviation 0.448374, Monte Carlo standard error 0.000448, and 95\% interval
$[2.434211,3.947368]$.  The observed statistic is 0.460482 points below the
null mean; 18.283\% of null replicates are no larger.

The 96 discordances decompose as follows.  Correct stops contribute
$0=0+0$; false-positive stops contribute $21=14+7$; and verifier rejections
contribute $75=41+34$, where each sum is reroute-only plus resample-only.
Thus 75/96 discordances occur outside the observable stopped population.

On the descriptive restriction $Z^{14}=1$, rerouting, resampling, and the
pointwise maximum have 1,080, 1,073, and 1,087 successes.  The
realized-maximum gap is $7/1{,}240=0.564516$ points and the analytic
exchangeable mean is $21/2{,}480=0.846774$ points.

\begin{figure*}[t]
\centering
\includegraphics[width=\textwidth]{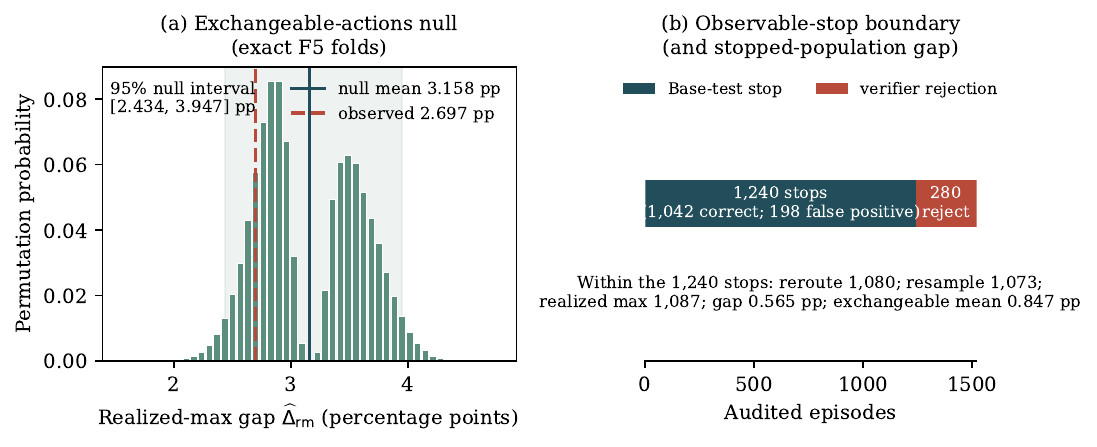}
\caption{Exchangeable calibration and the observable-stop boundary.  The
reference distribution calibrates marginal label symmetry rather than
conditional heterogeneity.  The second panel separates observable stops from
verifier rejections and reports the stopped-population diagnostic.}
\Description{A permutation distribution contains the observed 2.697-point
realized-maximum statistic; the second panel displays the audit population
boundary.}
\label{fig:s-exchangeable-null}
\end{figure*}

The structure-only controller gap remains $0.000$ points
$[-0.987,+1.053]$, and the structure-plus-token-hash gap remains $-0.197$
points $[-1.184,+0.789]$.  These valid descriptive comparisons show that
neither tested outcome-blind learner improves on the cross-fitted fixed action.
They do not establish equality of the fixed actions, absence of conditional
heterogeneity for all legal signals, or an observability bottleneck.
The frozen machine-readable verdict
\texttt{OBSERVABILITY\_BOTTLENECK} is preserved as execution provenance; the
article explicitly retracts its causal interpretation.

The post-outcome four-model equal-synthetic-cost sensitivity reported a
$+4.539$-point realized-maximum gap and negative controller gaps.  It was
constructed after outcome opening, has no exchangeable calibration in this
study, and is not used as evidence for heterogeneity, observability, or
efficiency.

\section{Deterministic cross-fitting and clustered inference}

Query IDs are ordered by a domain-separated SHA-256 hash and assigned
round-robin to five outer folds.  Inner selection uses four query folds and a
frozen ridge grid.  All ten draw positions from the same query remain in the
same outer fold, inner fold, and bootstrap cluster.  Action ties resolve first
by lower nominal cost and then by frozen model index.

For bootstrap replicate $b$ and sample position $j$, the resampled query index
is the first eight bytes of
\begin{equation}
 \mathrm{SHA256}(\mathrm{domain}\Vert\mathrm{U64BE}(b)
 \Vert\mathrm{U64BE}(j))\bmod N.
\end{equation}
Replicates are merged in increasing $b$.  The 24-worker report is result-
identical to an independent one-worker recomputation after normalizing only
the recorded worker count.

\begin{figure*}[t]
\centering
\includegraphics[width=\textwidth]{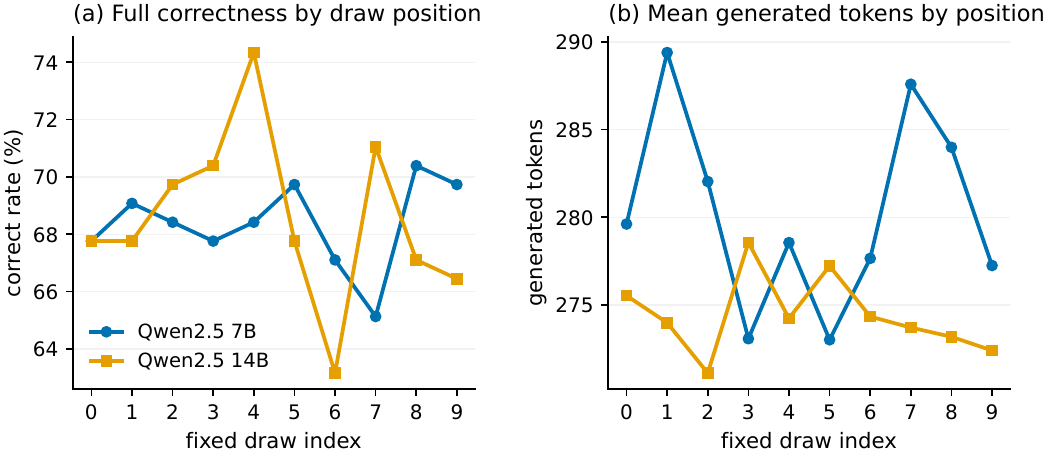}
\caption{Correctness and generated-token diagnostics across the ten fixed draw
positions.  Holm correction across the four tests rejects none; non-rejection
does not establish exchangeability.}
\Description{Two line panels show modest draw-position variation for
correctness and token counts in the two Qwen banks.}
\label{fig:s-draw}
\end{figure*}

\section{Fail-closed self-attacks}

Table~\ref{tab:s-attacks} lists the adversarial checks that were run before
the audit conclusion was accepted; every check passed.

\begin{table*}[t]
\centering
\caption{Final adversarial checks run before accepting the audit conclusion.}
\label{tab:s-attacks}
\small
\begin{tabularx}{\textwidth}{@{}p{0.28\textwidth}X p{0.10\textwidth}@{}}
\toprule
Invariant & Attack or independent check & Result \\
\midrule
Information separation & visible materializer accepts raw generation only;
selector accepts predictions and costs only & PASS \\
Canonical identity & strict schemas and byte-for-byte canonical JSON reproduction
for all run envelopes, query files, and score banks & PASS \\
Completion foreign keys & query/binding, raw-file digest, draw, seed, byte
length, and completion digest for 6,080 cells & PASS \\
Corruption response & changed completion and query-file digests rejected before
analysis & PASS \\
Draw ledger & repeated draw, future jump, unavailable head, and over-budget
action rejected & PASS \\
Closed form & \eqref{eq:s-oracle} equals exhaustive legal interleavings on
all tiny banks & PASS \\
Parallel determinism & one-worker and 24-worker 10,000-replicate reports match
after worker-field normalization & PASS \\
Targeted suite & known-answer STOP/simple/dynamic cases plus TEST adapter
attacks & 20/20 \\
\bottomrule
\end{tabularx}
\end{table*}

\section{Cross-paper data and contribution ledger}

The companion manuscript is publicly available as ``How Much of the Routing
Gap Is Real? Decomposing the Router-to-Oracle Gap into Reproducible Specialist
Advantage and Single-Draw Label Noise'' (arXiv:2607.03436) is a distinct companion manuscript.  It studies coupling, scorer
choice, and single-commit routing-oracle ceilings.  This
article instead studies a sequential verifier-gated failure event, action-
ranking support, and the conditions under which an outcome-blind controller
could be evaluated.  The relationship and shared historical inputs are disclosed.

Version 2 of this arXiv record, ``Resample or Reroute? Budget-Aware Test-Time
Model Selection for Large Language Models'' (arXiv:2607.08665v2, revised
10 July 2026), is the earlier replay-based version of this research line.  The
present version is a substantive methodological rebuild of it: it does not
import that version's learned-router,
retrospective cost-oracle, or universal Pareto claims, and it replaces them
with a hidden-truth boundary, verifier-gated estimands, fresh MBPP+ evidence,
the support-aware matched-information audit, and a separated LiveCodeBench
ladder.  The old archive remains visible only as audit/development provenance.

\begin{table*}[t]
\centering
\caption{Cross-paper separation and data-reuse ledger.}
\label{tab:s-ledger}
\small
\begin{tabularx}{\textwidth}{@{}p{0.24\textwidth}p{0.22\textwidth}X@{}}
\toprule
Asset or claim & Origin & Disposition in this article \\
\midrule
GSM8K/MATH/\allowbreak GPQA response texts & shared historical generation archive
& audit input only; IDs, scorer channels, and hashes disclosed \\
Routing-oracle decomposition & arXiv:2607.03436 companion
& background citation; not re-proved or headlined \\
Earlier RoR replay and broad policy claims & Version 2 of this record (arXiv:\allowbreak 2607.08665v2)
& prior-version provenance; unsupported claims are not carried forward \\
Fresh Qwen MBPP+ banks & this study
& confirmatory verifier-gated TEST evidence \\
Fresh Llama/Nemotron banks & post-outcome extension
& model-family robustness only; not independent replication \\
LiveCodeBench v5+v6 ladder & this study
& separated FIT/DEV/sealed-TEST support and observability diagnosis \\
BigCodeBench v0.1.4 FIT gate & this study
& preregistered support replication; DEV and TEST remain unopened \\
Verifier-gated estimands & this study
& headline sequence-specific mechanism contribution \\
Matched-information audit & this study
& descriptive all-episode realized-maximum audit with exchangeable calibration \\
Figures and principal theorems & separate manuscripts
& no figure, headline result, or principal theorem reused \\
\bottomrule
\end{tabularx}
\end{table*}

\section{Frozen artifact bindings}

The publication-level artifacts are bound by these SHA-256 identities:
\begin{list}{}{\setlength{\leftmargin}{0pt}\setlength{\itemsep}{2pt}\setlength{\parsep}{0pt}\setlength{\topsep}{3pt}}
\item \textbf{Fresh primary result}\newline\shabind{a79b39ac4b40df6bd4c6469f23f820ac}{a00a4a90b05cabc3404b8eab2b015968}
\item \textbf{Draw-index diagnostic}\newline\shabind{4c3141bc0af60581235a422ee0560f92}{88e75f8dbca7cf2942cfdcd241577d88}
\item \textbf{Model-family extension}\newline\shabind{8103ae371f624c6812445bd7072002297}{fd3c5375c09451d5864a9dcf95ac14d}
\item \textbf{Matched-information audit}\newline\shabind{c566ddd4069c2a7e364fab73b59545a8}{2528f840c4c91480350923f1092c6858}
\item \textbf{Exchangeable-null receipt}\newline\shabind{3e617fb1fb3f019b40058d055ac626a4}{c302da209d522a20e9d2532f7618fe49}
\item \textbf{Fresh final authority}\newline\shabind{4c939c360e24fb1ad5eba469e6807d46}{38fef7975e43f54756a552a0d5145ac4}
\item \textbf{Qwen7 score bank}\newline\shabind{26aa52bd41f29f77669981a0ae3bcc4e}{74b8b5a3dcc2091c36e7485863cae078}
\item \textbf{Qwen14 score bank}\newline\shabind{8038cb37326a72d1d3a20e5703f299cb}{995ea17b868a4de10ccfa95f99d1f7a7}
\item \textbf{Llama score bank}\newline\shabind{859bd0bfeae0d32a94f9ef0c0463696}{352c5faeb8e6b7af7bc0ccd359ddbafdb}
\item \textbf{Nemotron score bank}\newline\shabind{35e8f699a71529f30c6f85ee1d80e130}{1b2101a1ee0bb3458f06345320dafcba}
\item \textbf{LiveCodeBench ladder controller}\newline\shabind{41bb026d42204b2435941126991ccda4}{57ef477569d97d92cc5ce05237873a53}
\item \textbf{LiveCodeBench TEST-opening claim}\newline\shabind{3c64e09d72aa0e3f85691294d6f2030}{a6376f78895cc18bf2ace76fcbb6944c3}
\item \textbf{LiveCodeBench TEST score config}\newline\shabind{eb181e81f8c0d79d714ffb77742831cd}{6a998cc48ae509ae4c2fe6200b907151}
\item \textbf{LiveCodeBench TEST score terminal}\newline\shabind{4b17d45c6fa2dfdd19e2c6b88e72674}{33e0011acaffbaf932415a2b8a1556e95}
\item \textbf{LiveCodeBench visible-only features}\newline\shabind{8aac81dc5218bf0930f78abaf0d95f83}{8f4a5b0512b732678995cb5dbcf5ce3f}
\item \textbf{LiveCodeBench opened features}\newline\shabind{e7ea286f4506e7efb1ba3327f5bdb716}{6901940310cc2015052c1348b93d1c5b}
\item \textbf{LiveCodeBench sealed result}\newline\shabind{2711bda6f4ae6abf63cffbacf7cec5bc}{1ea3166bf4a09de467e7664250ee3e7d}
\end{list}

The matched-information report binds 6,080 completion foreign keys across
four raw-generation and score banks.  The manuscript and supplement reference
12 active figure PDFs: five are recorded in
\nolinkurl{figs/kbs_figure_manifest.json}, while the remaining seven are bound
by the observability, cost-boundary, BigCodeBench, and exchangeable-action
receipts or manifests and checked by
\nolinkurl{paper/validate_kbs_package.py}.  No single 12-figure manifest is
claimed.

\section{Historical replay audit boundary}
\label{sec:s-historical-replay}

The legacy 11-model replay over GSM8K, MATH-500, GPQA-Diamond, and HumanEval+
~\cite{gsm8k,mathdataset,lightman2024verify,gpqa,humaneval,evalplus}
is retained solely to document why broad allocator, static-routing, and
truth-guided-verifier claims were removed.  Its fixed-bank permutations are
not independent datasets, parameter-count costs are not prices, and the
historically inspected split is not confirmatory evidence.  Observable
verifiers alter trajectories, agreement can terminate on wrong answers, and
same-cap comparisons can be dominated elsewhere on the measured grid.  These
results are subtractive provenance and do not pass or fail a headline gate.

\begin{figure*}[t]
\centering
\IfFileExists{figures/fig_real_verifier_kbs.pdf}{%
  \includegraphics[width=\textwidth]{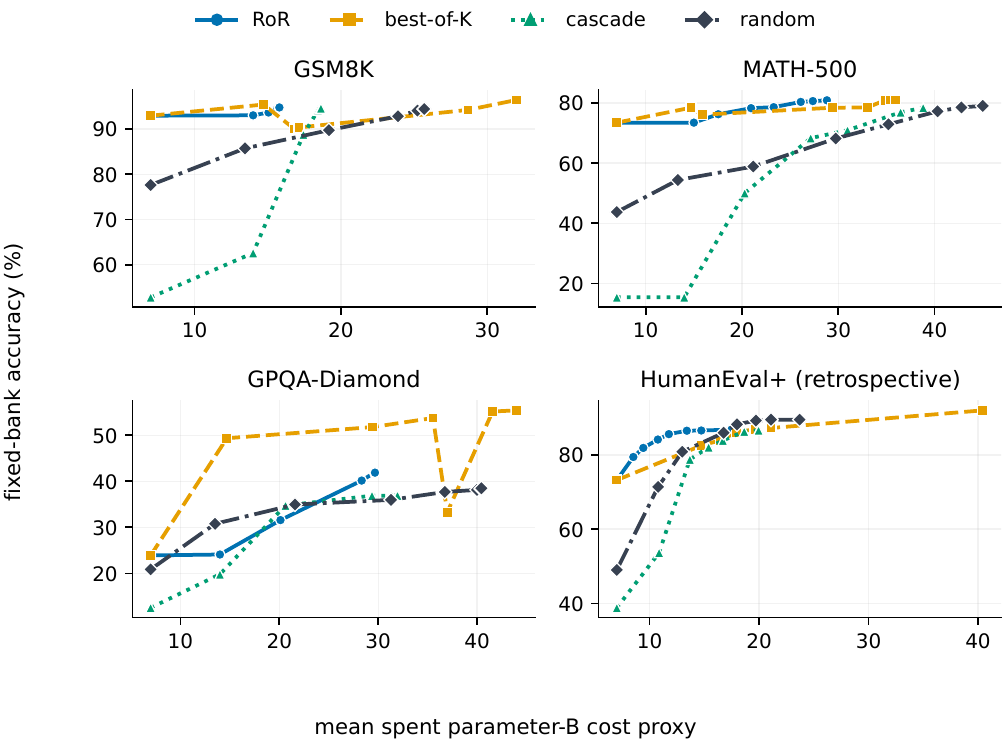}%
}{\fbox{\parbox[c][3.4cm][c]{0.92\textwidth}{\centering
Observable-verifier figure pending generated rebuild results.}}}
\caption{Policy performance under observable verifier adapters.  Unlike the
legacy $q$ interpolation, these trajectories never expose held-out hidden
correctness to allocation.  For legibility, lines connect evaluated points in
increasing realized mean cost; they are not, by themselves, Pareto frontiers.
Each panel has independently scaled cost and accuracy axes, so slopes and
distances must not be compared across benchmarks.}
\Description{Four independently scaled panels show cost versus fixed-bank
accuracy for RoR, finite-bank best-of-K, cascade, and random allocation.
Higher-cap RoR points are dominated on GSM8K and GPQA-Diamond; MATH-500 has a
small adapter-limited frontier region; HumanEval+ shows a cost--accuracy
trade-off but retains incomplete generation provenance.}
\label{fig:s-real}
\end{figure*}

\section{Interpretation checklist}

For every future table, figure, or claim derived from this work, verify:
\begin{enumerate}
\item Which frozen response archive, split, scorer channel, and model order
produced evaluator truth?
\item What information was policy visible at each decision, and what remained
evaluator only?
\item Did all compared policies use the same stopping event, target
population, action set, utility, and observable signal?
\item Were every query's draw positions kept together in training,
cross-validation, and uncertainty estimation?
\item Is the cost axis a nominal cap, spent-call proxy, generated tokens,
measured latency, or monetary price?
\item Was every pointwise realized maximum calibrated against an appropriate
exchangeable reference and kept distinct from expected action heterogeneity?
\item Is the result confirmatory, cross-fitted on an opened TEST population,
or post-outcome sensitivity analysis?
\item Does any new conclusion exceed the frozen artifact, model, verifier, and
draw-order scope?
\item Before interpreting a null fitted controller as a feature result, did
FIT contain prespecified positive and negative action-advantage support?
\end{enumerate}

\end{document}